\documentclass[10pt, letterpaper, twocolumn]{IEEEtran}      

\IEEEoverridecommandlockouts                                 
\usepackage{color}
\usepackage{latexsym} 
\usepackage[normalem]{ulem}   
\usepackage{cite}
\usepackage{graphicx}
\usepackage[english]{babel}    
\usepackage{graphicx}

\usepackage{pdfpages}
\usepackage{soul}
\usepackage[small]{caption}

\usepackage[hyphens]{url}

\usepackage{siunitx}
\usepackage{booktabs}
\usepackage{breakurl}
\usepackage[breaklinks]{hyperref}

\usepackage{amsmath} 
\usepackage{amssymb}  
\usepackage{amsthm}
\usepackage{bbm}
\usepackage{multirow}
\makeatletter

\usepackage[ruled,vlined]{algorithm2e}

\theoremstyle{plain}
\newtheorem{theorem}{Theorem}
\newtheorem{lemma}[theorem]{Lemma}

\theoremstyle{definition}
\newtheorem{definition}{Definition}
\newtheorem{remark}{Remark}
\newtheorem{problem}{Problem}
\newtheorem{example}{Example}
\newtheorem{assumption}{Assumption}
\begin{document}

\title{Probabilistic Motion Planning under Temporal Tasks \\and Soft Constraints}

\author{
Meng~Guo,~\IEEEmembership{Member,~IEEE,}~and~Michael~M.~Zavlanos,~\IEEEmembership{Member,~IEEE} %
\thanks{The authors are with the Department of Mechanical Engineering and Materials Science, Duke University, Durham, NC 27708 USA. Emails: {\tt\small meng.guo, michael.zavlanos@duke.edu}. This work is supported in part by NSF under grant IIS \#1302283.
}
}

\maketitle 

\markboth{IEEE~Transactions~on~Automatic~Control,~June~2017.}{Meng~Guo~and~Michael~M.~Zavlanos} 

\thispagestyle{empty} \pagestyle{empty}
\begin{abstract}
This paper studies motion planning of a mobile robot under uncertainty. The control objective is to synthesize a {finite-memory} control policy, such that a high-level task specified as a Linear Temporal Logic (LTL) formula is satisfied with a desired high probability. Uncertainty is considered in the workspace properties, robot actions, and task outcomes, giving rise to a Markov Decision Process (MDP) that models the proposed system. Different from most existing methods, we consider cost optimization both in the prefix and suffix of the system trajectory. We also analyze the potential trade-off between reducing the mean total cost and maximizing the probability that the task is satisfied. The proposed solution is based on formulating two coupled Linear Programs, for the prefix and suffix, respectively, and combining them into a multi-objective optimization problem, which provides provable guarantees on the probabilistic satisfiability and the total cost optimality. We show that our method outperforms relevant approaches that employ Round-Robin policies in the trajectory suffix. Furthermore, we propose a new control synthesis algorithm to minimize the frequency of reaching a bad state when the probability of satisfying the tasks is zero, in which case most existing methods return no solution. We validate the above schemes via both numerical simulations and experimental studies. 
\end{abstract}
\section{Introduction}\label{sec:intro}

\IEEEPARstart{I}{n} this paper we study the problem of robot motion planning under uncertainty and temporal task specifications. We consider uncertainty in the workspace properties, robot motion and actions, and outcome of task executions, which gives rise to a Markov Decision Process (MDP) to model the proposed system. MDPs have been used extensively to model motion and sensing uncertainty in robotics~\cite{thrun2005probabilistic, puterman2014markov} and then solve decision making problems that optimize a given control objective. 
The most common objective is to reach a goal state from an initial state while minimizing the cost. 
The resulting solution is a policy that maps states to actions~\cite{puterman2014markov}. 
On the other hand, Linear Temporal Logic (LTL) provides a formal language to describe complex high-level tasks beyond the classic start-to-goal navigation. A LTL task formula is usually specified with respect to an abstraction of the robot motion within the allowed workspace~\cite{fainekos2009temporal}, modeled by a deterministic finite transition system (FTS). 
Then a high-level discrete plan is found using off-the-shelf model-checking algorithms~\cite{baier2008principles}, which is then executed through low-level continuous controllers~\cite{fainekos2009temporal, belta2007symbolic}. This framework is extended to allow for both robot motion and actions in the task specification~\cite{guo2016hybrid} and partially-known or dynamic workspaces in~\cite{guo2015multi,wolff2012robust}.

Recently, there have been many efforts to address the problem of synthesizing a control policy for a MDP that satisfies high-level temporal tasks specified in various formal languages.
Different classes of Probabilistic Computation Tree Logic (PCTL) formulas have been studied in~\cite{lahijanian2015formal} for abstraction and verification over Interval-valued Markov Chains.
The work in~\cite{cizelj2014control} proposes a control policy for a mobile robot that maximizes the probability of satisfying a bounded linear temporal logic (BLTL) formula. 
Syntactical co-safe LTL formulas (sc-LTL) are considered in~\cite{ulusoy2014incremental} for a deterministic robot that co-exists with other robots whose behavior is modeled as a MDP.
A FTS with time-varying rewards is controlled to satisfy a LTL formula and maximize the accumulated reward in~\cite{ding2014ltl}.
A robust control policy for MDPs with uncertain transition probabilities is proposed in~\cite{wolff2012robust}.  
A verification toolbox is provided in~\cite{kwiatkowska2011prism} for probabilistic discrete-time or continuous-time {Markov Chain (MC)}, under a wide variety of quantitative properties expressed in PCTL, LTL, CTL, and so on. 

In this work, we study motion planning of a mobile robot under uncertainty in both robot motion and workspace properties. The goal is to synthesize a finite-memory control policy that generates robot trajectories that satisfy a high-level LTL task formula with desired high probability. At the same time, we optimize the total cost \emph{both} in the prefix and suffix parts of the system trajectories. Our proposed approach is based on solving two coupled Linear Programs, one for the prefix and one for the suffix, over the occupancy measures of the product automaton introduced in~\cite{altman1996constrained}. Moreover, we explore cases where the probability of satisfying the LTL tasks is zero, so that an Accepting End Component (AEC) does not exist in the MDP, where most relevant work returns no solutions. To address such situations, we treat satisfaction of the tasks as soft constraints and propose a relaxed suffix plan that minimizes the frequency with which the system enters bad states that violate the task specifications. We show that our approach outperforms the widely-used Round-Robin policy, via both numerical simulations and experimental studies. We also compare our proposed method with the widely-used probabilistic model-checking tool PRISM~\cite{kwiatkowska2011prism}.

Our work is related to literature on (i) policy synthesis for MDPs under multiple objectives; (ii) cost optimization within AECs in MDPs; and (iii) infeasible temporal tasks. We discuss below this literature and highlight our contributions.

Since we consider both temporal tasks and total-cost criteria over MDPs, this work is closely related to policy synthesis of MDPs under multiple objectives.
The work in~\cite{altman1996constrained} proposes a framework with provable correctness to synthesize a control policy for MDPs under multiple constrained total-cost criteria. 
{A survey on multi-objective  decision-making for MDPs can be found in~\cite{roijers2013survey}}. 
On the other hand, verification of MDPs under \emph{multiple} high-level tasks is addressed in~\cite{etessami2007multi}, where the probability of satisfying each subtask  is lower-bounded by a given value.
Moreover, a quantitative multi-objective verification scheme is proposed in~\cite{forejt2011quantitative,forejt2012pareto} for numerical queries over probabilistic reward predicates.

On the other hand, the seminal works~\cite{randour2015variations, randour2015percentile} consider MDPs with multi-dimensional weights under multi-percentile queries that may be conflicting.
{However, most of the above work does not address cost optimization over the suffix of the system trajectory within the AECs, neither does it address the case where no AECs can be found in the product automaton, which are the main contributions here.}

The satisfaction of a LTL formula is associated with reaching the corresponding AECs. In particular, in \cite[Chapter 10]{baier2008principles}, a value iteration method is used to solve the maximal reachability problem towards the AECs to obtain a policy for the plan prefix.
For planning within the AECs,~\cite{ding2011mdp,forejt2011quantitative, baier2008principles} adopt the Round-Robin policy,  which guarantees only correctness but not optimality.
Optimal  policies for the plan suffix that keeps the system within the AECs have been proposed in~\cite{ding2014optimal, smith2011optimal,chatterjee2011energy,fu2015pareto}. Specifically,  in~\cite{ding2014optimal} the expected cost of satisfying instances of a desired property is minimized,  while in~\cite{smith2011optimal} the minimal bottleneck cost is considered. Both approaches in~\cite{ding2014optimal,smith2011optimal} require particular types of LTL formulas (such as ``always eventually''). {The work in~\cite{chatterjee2011energy, bruyere2016meet} considers MDPs with $\omega$-regular specifications and quantitative resource constraints within the AECs}. The work in~\cite{fu2015pareto} investigates the Pareto cost of a human-in-the-loop MDP measured by a given discounted cost function.
{Compared to this literature, the multi-objective optimization problem that we formulate to solve the control synthesis problem allows us to explicitly characterize the trade-off between prefix and suffix optimality. We then extend this methodology to the case where no AECs can be found.}

Most aforementioned work~\cite{randour2015variations,randour2015percentile,Dimitrova2016robust, baier2008principles, forejt2011quantitative, ding2011mdp, ding2014optimal} relies on the assumption that the product automaton contains at least one AEC. However, in many situations this assumption does not hold so that the probability of satisfying the task under any policy is \emph{zero}. In this case, it is still important to identify those policies that minimize the frequency with which the system will reach the bad states that violate the task specifications.
{Consequently, it is desirable to synthesize a policy with certain risk guarantees even when soft LTL tasks are considered that are only partially-feasible. To the best of our knowledge, there is no work on control synthesis for infeasible soft LTL task formulas defined on MDPs, especially when an AEC can not  be found in the resulting product automaton.} For deterministic transition systems, a framework
for robot motion planning in partially-known workspaces is
proposed in~\cite{guo2015multi} that can handle soft LTL task formulas whose
satisfiability is improved over time; a least-violating control
strategy is synthesized in~\cite{tumova2013least} for a set of LTL safety rules. In the case of MDPs, a relevant formulation is considered in~\cite{Ehlers2016risk} where a MDP is controlled to satisfy an $\omega$-regular formula. {A policy is proposed to ensure that the MDP enters a failure state relatively late in the prefix. However, a multi-objective criterion of the control policy, especially in the plan suffix, is not considered there.}
Also, recent work in~\cite{lahijanian2016specification} proposes an approach to increase the satisfaction probability by modifying the task formula which, however, only considers co-safe LTL formulas without cost optimization constraints.

In summary, the main contribution of this work is three-fold: (i) a framework that optimizes the total cost both in the plan prefix and suffix, while ensuring that the tasks are satisfied with a desired high probability; (ii) a new algorithm to synthesize the control policies that have a high probability of satisfying the task over long time intervals, for cases where an AEC does not exist; { and (iii) a new method that allows the system to recover from bad states and continue the task.}

The rest of the paper is organized as follows. Section~\ref{sec:prelims} introduces  necessary preliminaries. In Section~\ref{sec:formulation}, we formalizes the considered problem. Section~\ref{sec:solution} presents our solution in details, which includes four major parts. Section~\ref{sec:example} demonstrates the feasibility of the results by numerical simulations. 
Section~\ref{sec:experiment} contains the experimental results. We conclude and discuss about future directions in Section~\ref{sec:conc}.

\section{Preliminaries}\label{sec:prelims}

\subsection{Transient MDP}\label{sec:transient-mdp}
A Markov Decision Process (MDP) is defined as a 6-tuple $\mathcal{M} \triangleq (X, \, U,\,D,\, p_D,\, c_D,\, x_0)$, where~$X$ is the finite state space; $U$ is the finite control action space (with a slight abuse of notation, $U(x)$ also denotes the set of control actions \emph{allowed} at state~$x\in X$); $D = \{(x,u)\,|\, x\in X,\, u\in U(x)\}$ is the set of possible state-action pairs; $p_D: X\times U \times X \rightarrow {[0, 1]}$ is the transition probability function so that~$p_D(x,\,u,\,\check{x})$ is the transition probability from state~$x$ to state~$\check{x}$ via control action~$u$ and ~$\sum_{\check{x}\in X}p_D(x,u,\check{x}) = 1$, $\forall (x,\,u) \in D$; 
{$c_D: D \rightarrow \mathbb{R}^{>0}$} that~$c_D(x,\,u)$ is the cost of performing action~$u\in U(x)$ at state~$x\in X$; and $x_0\in X$ is the initial state.
Denote by~$Post(x,\, u)\triangleq \{\check{x}\in X\,|\, p_D(x,u,\check{x})>0\}$,~$\forall (x,\,u) \in D$.

The above MDP evolves by taking an action~$u\in U(x)$ associated with every state~$x\in X$. Denote by~$R_T= x_0u_0x_1u_1\cdots x_Tu_T$ the past run that is a sequence of previous states and actions up to time~$T\geq 0$.
{As defined in~\cite{puterman2014markov}, a control policy $\boldsymbol{\mu}=\mu_0\mu_1\cdots$ is a sequence of decision rules $\mu_t$ at time $t\geq 0$. 
A control policy is stationary if~${\mu}_t=\mu$, $\forall t\geq 0$, where $\mu$ can be randomized so that~${\mu}:X \times U\rightarrow [0,1]$ or deterministic so that~$\mu: X \rightarrow U$, $\forall t\geq 0$. On the other hand, a policy is history dependent or finite-memory if ${\mu}_t: R_t \times U \rightarrow [0,1]$, where $R_t$ is the past history until time $t\geq 0$.}

\subsection{End Components}\label{sec:EC}

A \emph{sub-MDP} of~$\mathcal{M}$ is a pair~$(S,\, A)$ where~$S\subseteq X$ and~$A: S\rightarrow 2^U$ such that (i)~$S\neq \emptyset$, $\emptyset \neq A(s) \subseteq U(s)$, $\forall s\in S$; (ii)~$Post(s,\, u)\subseteq S$, $\forall s\in S$ and~$\forall u \in A(s)$. 
An \emph{End Component} (EC) of~$\mathcal{M}$ is a sub-MDP~$(S,\, A)$ such that the digraph~$G_{(S,A)}$ induced by~$(S,A)$ is strongly connected. 
An end component~$(S,\,A)$ is called \emph{maximal} if there is no other end component~$(S',\, A')$ such that~$(S,\,A)\neq (S',\,A')$,~$S\subseteq S'$ and~$A(s)\subseteq A'(s)$, $\forall s \in S$. 
The set of Maximal End Components (MECs) of a MDP is finite and can be uniquely determined.
{The analysis of MECs would include each EC as a special case.} 
We refer the readers to Definitions~10.116, 10.117 and 10.124 of~\cite{baier2008principles} for details. 
Moreover, an \emph{Accepting} MEC (AMEC) is an end component that satisfies certain accepting conditions such as the Streett and Robin conditions, which will be defined in the sequel.
On the other hand, a \emph{Strongly Connected Component} (SCC) of the digraph~$G_{\mathcal{M}}$ induced by~$\mathcal{M}$ is a set of states~$S\subseteq X$ so that there exists a path in each direction between any pair of states in~$S$. {Similarly, an \emph{Accepting} SCC (ASCC) is a SCC that satisfies certain accepting conditions. Note that the main difference between a MEC~$(S,A)$ and a SCC~$S$ is that the SCC does not restrict the set of actions~$U(s)$ that can be taken at each state~$s\in S$. In other words, there might be paths that start from any state within the SCC and end at states outside the SCC.}

\subsection{LTL and DRA}\label{sec:ltl}
The  ingredients of a Linear Temporal Logic (LTL) formula are a set of atomic propositions $AP$ and several Boolean and temporal operators. Atomic propositions are Boolean variables that can be either true or false.  A LTL formula is specified according to the  syntax~\cite{baier2008principles}:
$\varphi \triangleq \top \;|\; p  \;|\; \varphi_1 \wedge \varphi_2  \;|\; \neg \varphi  \;|\; \bigcirc \varphi  \;|\;  \varphi_1 \,\textsf{U}\, \varphi_2,$
where $\top\triangleq \texttt{True}$, $p \in AP$, $\bigcirc$ (\emph{next}), $\textsf{U}$ (\emph{until}) and $\bot\triangleq \neg \top$. For brevity, we omit the derivations of other operators like $\square$ (\emph{always}), $\Diamond$ (\emph{eventually}), $\Rightarrow$ (\emph{implication}).
The semantics of LTL is defined over the set of infinite words over~$2^{AP}$. Intuitively, $p \in AP$ is satisfied on a word $w = w(1)w(2)\ldots$ if it holds at $w(1)$, i.e., if~$p \in w(1)$. Formula $\bigcirc \, \varphi$ holds true if $\varphi$ is satisfied on the word suffix that begins in the next position $w(2)$, whereas $\varphi_1 \, \textsf{U}\, \varphi_2$ states that $\varphi_1$ has to be true until $\varphi_2$ becomes true. Finally, $\Diamond \, \varphi$ and $\square \, \varphi$ are true if $\varphi$ holds on $w$ eventually and always, respectively. We refer the readers to Chapter~5 of~\cite{baier2008principles} for the full definition.

The set of words that satisfy a LTL formula~$\varphi$ over $AP$ can be captured through a Deterministic Rabin Automaton~(DRA)~$\mathcal{A}_{\varphi}$~\cite{baier2008principles}, defined as~$\mathcal{A}_{\varphi}=(Q, \,2^{AP},\, \delta,\, q_0,\,\text{Acc}_{\mathcal{A}})$,
where $Q$ is a  set of states; {$2^{AP}$ is the alphabet}; $\delta\subseteq Q\times 2^{AP} \times {Q}$ is a transition relation; $q_0 \in Q$ is the initial state; and~$\text{Acc}_{\mathcal{A}} \subseteq 2^Q \times 2^Q$ is a set of accepting pairs, i.e., $\text{Acc}_{\mathcal{A}} = \{(H^1_{\mathcal{A}}, I^1_{\mathcal{A}}), (H^2_{\mathcal{A}}, I^2_{\mathcal{A}}), \cdots, (H^N_{\mathcal{A}}, I^N_{\mathcal{A}})\}$ where~$H^i_{\mathcal{A}},\, I^i_{\mathcal{A}}\subseteq Q$, $\forall i =1,2,\cdots, N$.
An infinite run~$q_0q_1q_2\cdots$ of~$\mathcal{A}$ is \emph{accepting} if there exists \emph{at least one} pair~$(H^i_{\mathcal{A}},\, I^i_{\mathcal{A}})\in \text{Acc}_{\mathcal{A}}$ such that~$\exists n\geq 0$, it holds~$\forall m\geq n,\, q_m\notin H^i_{\mathcal{A}}$ and~$\overset{\infty}{\exists} n\geq 0$,~$q_n\in I^i_{\mathcal{A}}$, where~$\overset{\infty}{\exists}$ stands for~``existing infinitely many''.
Namely, this run should intersect with~$H^i_{\mathcal{A}}$ \emph{finitely} many times while with~$I^i_{\mathcal{A}}$ \emph{infinitely} many times. 
There are translation tools~\cite{klein2007ltl2dstar} to obtain~$\mathcal{A}_{\varphi}$ given~$\varphi$, which requires the process of translating firstly the LTL formula to the associated Nondeterministic B\"uchi Automaton~(NBA), and then to the DRA {with complexity~$2^{2^{\mathcal{O}(n\log n)}}$}, where~$n$ is the length of~$\varphi$.
Our implementation of the Python interface for~\cite{klein2007ltl2dstar} can be found in~\cite{git_mdp_tg}.
Note that~\cite{klein2007ltl2dstar} allows for different levels of automata simplifications to be made regarding the size of~$\mathcal{A}_{\varphi}$, {and a simplified automation may result in loss of optimality.}

\section{Problem Formulation}\label{sec:formulation}

\subsection{Mathematical Model}\label{sec:model}
In order to model uncertainty in both the robot motion and the workspace properties, we extend the definition of a MDP from Section~\ref{sec:transient-mdp} to include probabilistic labels, as the~\emph{probabilistically-labeled} MDP:
\begin{equation}\label{eq:mdp}
\mathcal{M} = (X, \, U,\, D,\, p_{D}, \, (x_0,\,l_0),\, AP, \,L,\, p_{L},\,c_D),
\end{equation}
where~$AP$ is a set of atomic propositions that capture the properties of interest in the workspace;~{$L:X \rightarrow 2^{2^{AP}}$ contains the set of property subsets that can be true at each state}; and~$p_{L}:X\times 2^{AP}\rightarrow {[0,\,1]}$ specifies the associated probability. Particularly,~$p_{L}(x,\,l)$ denotes the probability that state~$x\in X$ satisfies the set of propositions~$l \subset AP$.
Note that~$\sum_{l \in L(x)} p_L(x,\,l)=1$, $\forall x \in X$.
Moreover,~$(x_0,\,l_0)$ contains the initial state~$x_0\in X$ and the initial label~$l_0\in L(x_0)$, while the rest of the notations in~\eqref{eq:mdp} are the same as defined in Section~\ref{sec:transient-mdp}. 
The probabilistic labeling function provides a way to consider time-varying and dynamic workspace properties. 
Moreover, there is a LTL task formula~$\varphi$ specified over the same set of atomic propositions~$AP$, as the desired behavior of~$\mathcal{M}$. 
{We assume that the MDP~$\mathcal{M}$ in~\eqref{eq:mdp} is \emph{fully-observable} due to the following assumption.}

\begin{assumption}\label{assumption:observable}
At any stage~$t\geq 0$, the current robot state $x_t\in X$ and its label $l_t\in L(x_t)$ are fully-observable. \hfill $\blacksquare$
\end{assumption}

\begin{figure}[t]
     \centering
     \includegraphics[width=0.5\textwidth]{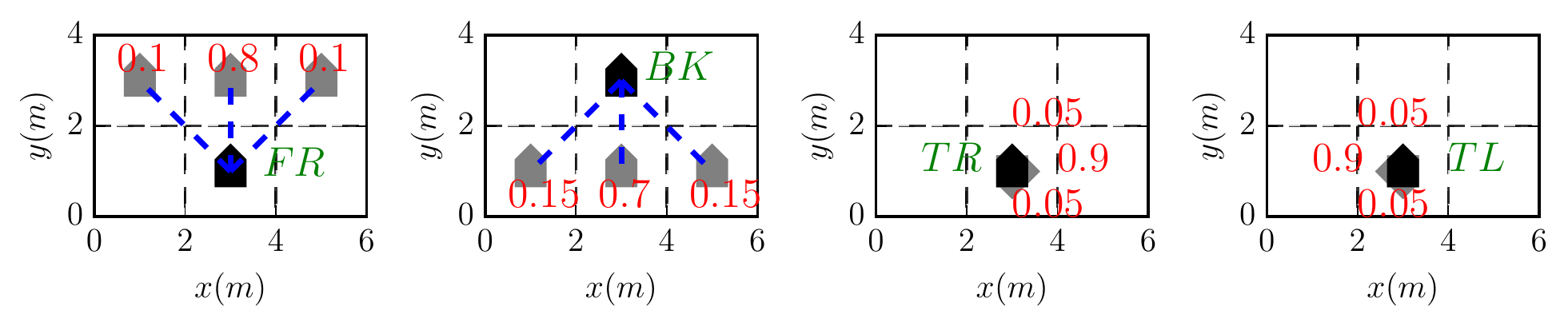}
     \caption{Uncertainty of each action primitive, {see Section~\ref{sec:example} for details}. Possible post states  are in grey from the starting state in black, where the associated possibilities are marked in red.}
     \label{fig:action}
   \end{figure}

While the robot is moving within the workspace, it is capable of sensing an actual property and determine the label of the state it is located at. 
At stage~$T\geq 0$, the robot's past path is given by~$X_T= x_0 x_1 \cdots x_T \in X^{(T+1)}$, the past sequence of observed labels is given by~$L_T=l_0 l_{1}\cdots l_T\in (2^{AP})^{(T+1)}$ and the past sequence of control actions is~$U_T=u_0 u_1 \cdots u_T \in U^{(T+1)}$. It holds that~$p_D(x_t, u_t, x_{t+1})>0$ and~$p_L(x_t, l_t)>0$, $\forall t\geq 0$. These three sequences can be composed into the complete past run~$R_T=x_0l_0u_0 x_1l_1u_1\cdots x_Tl_Tu_T$. Denote by~$\boldsymbol{X}_T$, $\boldsymbol{L}_T$ and~$\boldsymbol{R}_T$ the set of all possible past sequences of states, labels, and runs up to stage~$T$. We set $T=\infty$ for infinite sequences.

\begin{definition}\label{def:mean-total-cost}
The mean total cost~\cite{chatterjee2011two, puterman2014markov} of an infinite robot run $R_{\infty}$ of~$\mathcal{M}$ is defined as
\begin{equation}\label{eq:total-cost}
\textbf{Cost}(R_{\infty}) = \lim \inf_{n\rightarrow \infty} \frac{1}{n}\sum_{t=0}^{n} c_D(x_t,u_t),
\end{equation}
where $R_\infty=x_0l_0u_0 x_1l_1u_1\cdots\in \boldsymbol{R}_{\infty}$.  \hfill $\blacksquare$
\end{definition}
As discussed in~\cite{chatterjee2011two, chatterjee2011energy, randour2015percentile, puterman2014markov}, the above mean total cost is called the \emph{mean-payoff} function (or limit-average), where the~``$\lim$'' operator is needed as the limit-average might not exist for some runs, see~\cite{chatterjee2011two, chatterjee2011energy,brazdil2015multigain}.

{Our goal is to find a {finite-memory} policy for~$\mathcal{M}$, denoted by~$\boldsymbol{\mu}=\mu_0\mu_1\cdots$. The control policy at stage~$t\geq 0$ is given by~$\mu_t:\mathbf{R}_t \times U\rightarrow [0, 1]$, where~$\mathbf{R}_t$ is the past run~$R_t$, $\forall t\geq 0$.}
Denote by~$\overline{\boldsymbol{\mu}}$ the set of all such policies.
Given a control policy~$\boldsymbol{\mu}\in \overline{\boldsymbol{\mu}}$, {the probability measure~$\text{Pr}^{\mathcal{M}}_{\boldsymbol{\mu}}(\cdot)$ on the smallest~$\sigma$-algebra, over all possible infinite sequences~$\boldsymbol{R}_{\infty}$ that contain~$R_{T}$, is the unique measure~\cite{baier2008principles} by}
\begin{equation}\label{eq:prob-measure}
\begin{split}
{Pr}_{\mathcal{M}}^{\boldsymbol{\mu}}(\mathbf{R}_{\infty})= \prod_{t=0}^{T} &\, p_D(x_t,\,u_t,\,x_{t+1})\\
&\quad \cdot p_L(x_t,\,l_t) \cdot {\mu_t(\boldsymbol{R}_t,u_t)},
\end{split}
\end{equation}
{where~$\mu(\boldsymbol{R}_t,u_t)$ is defined as the probability of choosing action~$u_t$ given the past run $\boldsymbol{R}_t$.} Then we define the probability of~$\mathcal{M}$ satisfying~$\varphi$ under policy~$\boldsymbol{\mu}$ by:
$${Pr}_{{\mathcal{M}}}^{\boldsymbol{\mu}}(\varphi)={Pr}_{\mathcal{M}}^{\boldsymbol{\mu}}\{ \boldsymbol{R}_{\infty} \,|\, \boldsymbol{L}_{\infty}\models \varphi \},$$ 
where the satisfaction relation~``$\models$'' is introduced in Section~\ref{sec:ltl}, given an infinite word and a LTL formula. 
Accordingly, the \emph{risk} is defined as the probability that the task formula~$\varphi$ is {not} satisfied by~$\mathcal{M}$ under the policy~$\boldsymbol{\mu}$, namely,~$\textbf{Risk}_{{\mathcal{M}}}^{\boldsymbol{\mu}}(\varphi) = 1-{Pr}_{{\mathcal{M}}}^{\boldsymbol{\mu}}(\varphi)$. 

\begin{problem}\label{prob:main}
Given the labeled MDP~$\mathcal{M}$ defined in~\eqref{eq:mdp} and the task specification~$\varphi$, our goal is to sovle:
\begin{equation}\label{eq:objective}
\begin{split} 
&{\min_{\boldsymbol{\mu}\in \overline{\boldsymbol{\mu}}} \; \mathbb{E}^{\boldsymbol{\mu}}_{\mathcal{M}} \{\textbf{Cost}(R_{\infty})\}}\\
&\; {s.t.} \quad \textbf{Risk}_{{\mathcal{M}}}^{\boldsymbol{\mu}}(\varphi) \leq \gamma,\\
\end{split}
\end{equation}
where ${\gamma\geq 0}$ is a pre-defined parameter as the allowed risk; the optimal policy minimizes the mean total cost and ensures that the risk of violating~$\varphi$ remains bounded by~$\gamma$. \hfill $\blacksquare$
\end{problem}

Note that the traditional definition of un-discounted expected total cost over an infinite run from~\cite{altman1996constrained, puterman2014markov} is not used here, as it is infinite except for the special case of transient MDPs defined in Section~\ref{sec:transient-mdp}. However, in this work, the model~$\mathcal{M}$ is not restricted to be transient. 
Moreover, the discounted total cost in~\cite{puterman2014markov} is not used here either due to two reasons: first, it is not obvious how to choose the discount factor for various control tasks~$\varphi$~\cite{fu2015pareto}; and second, we are more interested in  optimizing the repetitive long-term behavior of the system, rather than the short-term one~\cite{randour2015percentile}.
In-depth discussions on the optimization of infinite-horizon un-discounted or discounted total-cost criteria over MDPs with or without constraints can be found in~\cite{puterman2014markov}.

\begin{remark}\label{rmk:randomized}
Different from the maximal reachability problem addressed in~\cite{baier2008principles,ding2011mdp}, a deterministic policy would not suffice here. 
Instead, randomization is required due to the mean total-cost criterion and the risk constraint, similar to~\cite{altman1996constrained}. \hfill $\blacksquare$
\end{remark}

\section{Solution}\label{sec:solution}

This section contains the three major parts of the proposed solution: (i) the construction of the product automaton and its AMECs; (ii) the algorithms to synthesize the optimal plan prefix and suffix, for both cases where the AMECs exist or not; (iii) the complete policy, and the online execution algorithm.

\subsection{Product Automaton and AMECs}\label{sec:product}
To begin with, we construct the DRA~$\mathcal{A}_{\varphi}$ associated with the~LTL task formula~$\varphi$ via the translation tools~\cite{klein2007ltl2dstar, git_mdp_tg}. Let it be~$\mathcal{A}_{\varphi}=(Q, \,2^{AP},\, \delta,\, q_{0},\,\text{Acc}_{\mathcal{A}})$, where the notations are defined in Section~\ref{sec:ltl}. 
Then we construct a product automaton between the robot model~$\mathcal{M}$ and the DRA~$\mathcal{A}_{\varphi}$.

\begin{definition}\label{def:product}
Denote by~$\mathcal{P}$ the product $\mathcal{M}\times \mathcal{A}_{\varphi}$ as a 7-tuple:
\begin{equation}\label{eq:prod}
\mathcal{P}=(S,\,U,\,E,\,p_E,\,c_E,\, s_0,\,\text{Acc}_{\mathcal{P}}),
\end{equation}
 where: the state~$S\subseteq X\times 2^{AP} \times Q$ is so that~$\langle x,\,l,\,q \rangle \in S$, $\forall x \in X$,~$\forall l\in L(x)$ and~$\forall q\in Q$; the action set~$U$ is the same as in~\eqref{eq:mdp} and~$U(s)=U(x)$,~$\forall s=\langle x,l,q\rangle \in S$; $E=\{(s,u)\,|\, s\in S,\, u\in U(s)\}$; the transition probability~$p_E:S\times U \times S\rightarrow {[0,\,1]}$ is so that
\begin{equation}\label{eq:new_pro}
p_E\big{(}\langle x,l,q\rangle,\,u,\, \langle \check{x},\check{l},\check{q}\rangle \big{)} =  p_D(x,\,u,\,\check{x})\cdot p_L(\check{x},\,\check{l})
\end{equation}
where (i) $\langle x,l,q\rangle,\, \langle \check{x},\check{l},\check{q}\rangle \in S$; (ii)~$(x,u)\in D$; and (iii)~$\check{q}\in \delta(q,\,l)$;  the cost function~$c_E:E\rightarrow \mathbb{R}^{>0}$ is so that~$c_E\big{(}\langle x,l,q\rangle,\,u \big{)}=c_D(x,u)$, $\forall \big{(}\langle x,l,q\rangle,\,u \big{)}\in E$. Namely, the label~$l$ should fulfill the transition condition from~$q$ to~$\check{q}$ in~$\mathcal{A}_{\varphi}$; the single initial state is~$s_0=\langle x_0,l_0,q_0 \rangle \in S$; the accepting pairs are defined as~$\text{Acc}_{\mathcal{P}}=\{(H^i_{\mathcal{P}},\, I^i_{\mathcal{P}}), i=1,2,\cdots,N\}$, where~$H^i_{\mathcal{P}}=\{\langle x,l,q\rangle \in S\,|\, q\in H^i_{\mathcal{A}}\}$ and~$I^i_{\mathcal{P}}=\{\langle x,l,q\rangle \in S\,|\, q\in I^i_{\mathcal{A}}\}$,~$\forall i=1,2,\cdots,N$. \hfill $\blacksquare$

\end{definition}

The product~$\mathcal{P}$ computes the intersection between all traces of~$\mathcal{M}$ and all words that are accepted by~$\mathcal{A}_{\varphi}$, to find all admissible robot behaviors that satisfy the task~$\varphi$. It combines the uncertainty in robot motion and the workspace model by including both~$x$ and~$l$ in the states.
The Rabin accepting condition of~$\mathcal{P}$ is defined as follows: An infinite path~$R_{\mathcal{P}}=s_0s_1\cdots$ of~$\mathcal{P}$ is accepting if for at least one pair~$(H^i_{\mathcal{P}}, I^i_{\mathcal{P}})\in \text{Acc}_{\mathcal{P}}$ it holds that~$R_{\mathcal{P}}$ intersects with~$H^i_{\mathcal{P}}$ finitely often while with~$ I^i_{\mathcal{P}}$ infinitely often. To transform this   condition into equivalent graph properties,  we need to compute the AMECs of~$\mathcal{P}$ associated with its accepting pairs~$\text{Acc}_{\mathcal{P}}$. Detailed definition of MECs is given in Section~\ref{sec:EC}.

In order to find the complete set of AMECs of~$\mathcal{P}$, for each pair~$(H^i_{\mathcal{P}}, I^i_{\mathcal{P}})\in \text{Acc}_{\mathcal{P}}$, perform the following steps:

(i) Build the MDP~$\mathcal{Z}_i^{\neg H}\triangleq (S', \, U',\, E',\, p'_E)$, where~$S'= S_i^{\neg H} \cup \{\nu\}$ is the set of states with~$S_i^{\neg H}=S\backslash H^i_{\mathcal{P}}$ and~$\nu$ a trap state; 
~$U'=U\, \cup\, \{\tau_0\}$ is the set of actions where~$\tau_0$ is a pseudo action; 
~$E'\subset S'\times U$ is the set of transitions with the associated probability~$p'_E$ which are defined by three cases: 
(a) for the transitions within~$S_i^{\neg H}$ it holds that~$(s,\,u)\in E'$ and~$p'_E(s,\,u,\, \check{s})=p_E(s,\,u,\, \check{s})$, $\forall (s,\,u)\in E$ where~$s,\,\check{s}\in  S_i^{\neg H}$;
(b) for the transitions from~$S_i^{\neg H}$ to outside~$S_i^{\neg H}$ it holds that~$(s,\,u)\in E'$ {and~$p'_E(s,\,u,\, \nu)=\sum_{\check{s} \notin S_i^{\neg H}}p_E(s,\,u,\, \check{s})$, $\forall (s,\,u)\in E$ where~$s\in  S_i^{\neg H}$};
 and~(c) the trap state is included in a self-loop such that~$(\nu,\,\tau_0)\in E'$ and~$p'_E(\nu,\,\tau_0,\, \nu)=1$.
Simply speaking, all transitions from inside~$S_i^{\neg H}$ to outside~$S_i^{\neg H}$ are transformed to transitions to the trap state~$\nu$.

(ii) Determine \emph{all} MECs of~$\mathcal{Z}_i^{\neg H}$ above via Algorithm~47 in~\cite{baier2008principles}, which is based on splitting the strongly connected components (SCCs) of~$\mathcal{Z}_i^{\neg H}$ until the conditions of being an end component are fulfilled. 
Our implementation for this algorithm can be found in~\cite{git_mdp_tg}. 
Denote by~$\Xi^i=\{(S'_1,\, U'_1), (S'_2,\, U'_2), \cdots (S'_{C_i},\, U'_{C_i})\}$ the set of MECs, where~$S'_{c}\subset S'$ and~$U'_c:S'_c \rightarrow 2^{U'}$,~$\forall c=1,2,\cdots,C_i$. {Note that $S'_{c}\cap S'_{c'}=\emptyset$, $\forall (S'_c,U'_c),(S'_{c'},U'_{c'})\in \Xi^i$.}

(iii) Find~$(S'_c,\,U'_c)\in \Xi^i$ that is \emph{accepting}, i.e., it satisfies~$\nu \notin S'_c$ and~$S'_c \cap I^i_{\mathcal{P}}\neq \emptyset$. Save the AMECs in~$\Xi^i_{acc}$. Since~$\Xi^i_{acc}$ is computed for each~$(H^i_{\mathcal{P}}, I^i_{\mathcal{P}})\in \text{Acc}_{\mathcal{P}}$, we denote by~$\Xi_{acc}=\{\Xi^i_{acc},\, i=1,\cdots,N\}$ the complete set of AMECs of~$\mathcal{P}$.

\begin{remark}\label{rem:self-loop}
A single state with a self-transition can be a MEC with a proper action set. Therefore, there exists at most~$|S'|$ MECs within~$\mathcal{Z}_i^{\neg H}$, $\forall i=1,\cdots,N$. Thus Step (ii) above has complexity~$\mathcal{O}(|S'|^2)$, as shown in Lemma 10.126 of~\cite{baier2008principles}, {while Steps (i) and (iii) have complexity linear with~$|S'|$.}
\hfill $\blacksquare$
\end{remark}

\subsection{Plan Prefix and Suffix Synthesis}\label{sec:initial}
Given the complete set of AMECs~$\Xi_{acc}$ of~$\mathcal{P}$, in this section we show how to synthesize the control policy to drive the system towards~$\Xi_{acc}$ and furthermore remain inside~$\Xi_{acc}$ while satisfying the accepting condition. 
As mentioned in Section~\ref{sec:intro}, most related work~\cite{ding2011mdp, baier2008principles, etessami2007multi, forejt2011quantitative} focuses on maximizing the probability of reaching the union of AMECs, i.e.,~$\cup_{(S'_c,\,U'_c)\in \Xi_{acc}} S'_c$, where dynamic programming techniques, such as value or policy iteration, can be applied to obtain the optimal policy.
Furthermore, once the system enters any AMEC, e.g.,~$(S'_c,\, U'_c)\in \Xi_{acc}$, it has probability~$1$ of staying within~$S'_c$ by following~$U'_c$~(see Lemma~10.119 of~\cite{baier2008principles}). The Round-Robin  policy is adopted in~\cite{ding2011mdp,forejt2011quantitative, baier2008principles} that ensures all states in~$S'_c$ (including its nonempty intersection with~$I_{\mathcal{P}}^i$) are visited infinitely often. As a result, the task~$\varphi$ is satisfied by~$\mathcal{P}$ under this policy with the maximal probability.

The above solutions may suffice for verification problems that do not optimize cost or for tasks with trivial accepting conditions. However,  for the purposes of plan synthesis and for general tasks, it is of practical interest to simultaneously satisfy the probability of reaching \emph{all} the AMECs {as well as} optimize the mean cost of staying within \emph{any} AMEC and fulfilling the accepting condition. Moreover, when no AECs can be found, instead of simply reporting failure, it is important to obtain a relaxed policy that guarantees high probability of satisfying the task over long time intervals  thus minimizing the frequency of encountering bad events. 
In what follows we present a policy synthesis algorithm that consists of {four} parts:
\begin{itemize}
\item {the \emph{plan prefix} that drives the system from the initial state to all AMECs, while minimizing the expected cost and respecting the risk constraint; see Section~\ref{sec:plan-prefix}};

\item {the \emph{plan suffix} that keeps the system within the AMEC it has reached, while satisfying the accepting condition and optimizing the expected suffix cost; see Section~\ref{sec:plan-suffix};}

\item the \emph{relaxed} prefix and suffix plans for the case where no AECs of~$\mathcal{P}$ can be found; see Section~\ref{sec:non-AEC}; and

\item {the complete finite-memory policy for the original MDP $\mathcal{M}$; see Section~\ref{sec:two-cases}.}
\end{itemize}

\begin{figure}[t]
     \centering
     \includegraphics[width=0.5\textwidth]{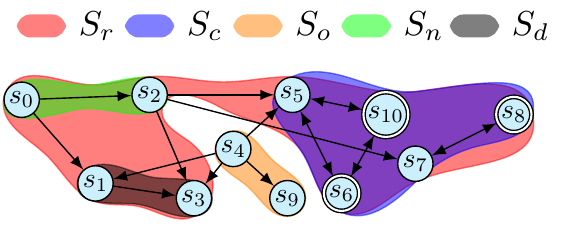}
     \caption{{Illustration of the partition of~$S$ in Definition~\ref{def:partition}, where $S_r$, $S_c$, $S_o$, $S_n$ and $S_d$ are highlighted by red, blue, orange, green and black areas, respectively. Details can be found in Example~\ref{example:partition}.}}
     \label{fig:partition}
   \end{figure}

Before stating the solution, we introduce a partition of~$S$ given the initial state~$s_0$ and the set of AMECs~$\Xi_{acc}$. Let~$S_r\subseteq S$  be the set of states within~$S$ that can be reached from~$s_0$, which can be derived via a simple graph search in~$\mathcal{P}$. 
\begin{definition}\label{def:partition}
Given~$s_0$ and~$\Xi_{acc}$, $S$ is partitioned as~$S=S_o\cup S_c \cup S_d \cup S_n$, where
$S_o\triangleq S\backslash S_r$ is the set of states that can \emph{not} be reached from~$s_0$;
{$S_c$ is the union of all goal states in~$\Xi_{acc}$, i.e., $S_c\triangleq \cup_{(S'_c,\,U'_c)\in \Xi_{acc}} S_c'$};
$S_d\subseteq S_r$ can be reached from~$s_0$ but can not reach any state in~$S_c$;
and $S_n\triangleq S_r\backslash (S_c\cup S_d)$. \hfill $\blacksquare$
\end{definition}

The set $S_d$ can be derived through a simple graph search, e.g., by reversing the directed graph associated with~$\mathcal{P}$, finding all reachable nodes of any state within each~$(S'_c,\,U'_c)\in \Xi_{acc}$ (as any AMEC is strongly connected) and finally computing its cross intersection with $S_r$; see~\cite{git_mdp_tg} for implementation details. 
Roughly speaking,~$S_n$ is the set of states related to the plan prefix,~$S_c$ is the set of goal states related to the plan suffix, and~$S_d$ is set of bad states to be avoided during the prefix. 
Since~$S_o$ contains the states that can not be reached from~$s_0$, it is neglected hereafter for the purpose of plan synthesis.

\begin{example}\label{example:partition}
This example illustrates the partition in Definition~\ref{def:partition}. 
Consider the toy product automaton~$\mathcal{P}$ in Figure~\ref{fig:partition}. For state~$s_0$, the set of reachable states is~$S_r=\{s_0,s_1,s_2,s_3,s_5,s_6,s_7,s_8,s_{10}\}$, the set of unreachable states is $S_o=\{s_4, s_9\}$, the states within an AMEC are~$S'_{c_1}=\{s_5,s_6,s_{10}\}$ and another AMEC $S'_{c_2}=\{s_7,s_8\}$, thus $S_c=S'_{c_1}\cup S'_{c_2}=\{s_5, s_6, s_7,s_8,s_{10}\}$,  the states that can be reached from $s_0$ but can not reach $S_c$ are $S_d=\{s_1, s_3\}$, and the states that~$s_0$ can reach outside~$S_c\cup S_d$ are $S_n=\{s_0,s_2\}$. \hfill $\blacksquare$
\end{example}

\subsubsection{Plan Prefix}\label{sec:plan-prefix}
Similar to~\cite{forejt2011quantitative,forejt2012pareto}, we first construct a modified sub-MDP~$\mathcal{Z}_{\texttt{pre}}$ of~$\mathcal{P}$ as~$\mathcal{Z}_{\texttt{pre}}\triangleq (S_{\texttt{p}},\, U_{\texttt{p}},\, E_{\texttt{p}},\, s_0,\,p_{\texttt{p}},\,c_{\texttt{p}})$, {where the set of states is given by $S_{\texttt{p}}= S_n \cup S_c$ with $S_n,S_c$ being defined in Definition~\ref{def:partition}.} 
The set of actions is given by~$U_{\texttt{p}}=U\cup\{\tau_0\}$ where~$\tau_0$ is a self-loop action. The set of transitions~$E_{\texttt{p}}$ is the subset of~$E$ associated with~$S_{\texttt{p}}$.
Moreover, the transition probability~$p_{\texttt{p}}$ is defined by (i)~$p_{\texttt{p}}(s,u,\check{s})=p_E(s,u,\check{s})$, $\forall s, \check{s}\in S_{\texttt{p}}$ where~$s\notin S_c$ and~$\forall u\in U(s)$; and (ii)~$p_{\texttt{p}}(s,\tau_0,s)=1$, $\forall s\in S_c$. Finally, the cost function~$c_{\texttt{p}}$ is defined by (i)~$c_{\texttt{p}}(s,u)=c_E(s,u)$, $\forall s\in S_n$ and~$\forall u\in U(s)$; and~(ii) $c_{\texttt{p}}(s,\tau_0)=0$, $\forall s\in S_c$.

Then, we find a policy for~$\mathcal{Z}_{\texttt{pre}}$ such that, starting from~$s_0$, it can reach the set of goal states~$S_c$ with a probability larger than~$1-\gamma$, while at the same time minimizing the expected total cost. Formally, consider the problem below:

\begin{problem}\label{problem:lp1}
Given the sub-MDP~$\mathcal{Z}_{\texttt{pre}}$, compute an {optimal stationary prefix} policy~$\boldsymbol{\pi}^\star_{\texttt{pre}}\in \overline{\boldsymbol{\pi}}$ that solves the problem
\begin{equation}\label{eq:problem_lp1}
\begin{split}
&\underset{\boldsymbol{\pi}\in \overline{\boldsymbol{\pi}}}{\boldsymbol{\min}} \; \; \bigg[\textbf{C}_{\texttt{pre}}(S_c) \triangleq \; \mathbb{E}^{\boldsymbol{\pi}}_{\mathcal{Z}_{\texttt{pre}}}\Big\{\sum_{t=0}^{\infty} c_{\texttt{p}}(s_t,u_t)\Big\} \bigg] \\
&\textrm{s.t.}\; \;{Pr}_{s_0}^{\boldsymbol{\pi}}(\Diamond S_c) \geq 1-\gamma,
\end{split}
\end{equation}
where~$s_0u_0s_1u_1\cdots$ is a run of~$\mathcal{Z}_{\texttt{pre}}$, $\overline{\boldsymbol{\pi}}$ is the set of all stationary policies,~the objective function is the expected total cost,~${Pr}_{s_0}^{\boldsymbol{\pi}}(\Diamond S_c)$ is the probability of reaching~$S_c$ from the initial state~$s_0$, under the policy~$\boldsymbol{\pi}$; and~$\gamma>0$ is from~\eqref{eq:objective}.
\hfill $\blacksquare$
\end{problem}

Note that the  objective function in~\eqref{eq:problem_lp1} is well-defined and finite due to the fact that~$\mathcal{Z}_{\texttt{pre}}$ is transient with respect to $S_n$, and is equal to the expected total cost of reaching~$S_c$ since the cost of staying within~$S_c$ is zero. {We omit the proof that $\mathcal{Z}_{\texttt{pre}}$ is transient here and refer the interested readers to~\cite{puterman2014markov, altman1996constrained}.}
Our proposed solution to Problem~\ref{problem:lp1} is based on transforming it into a constrained optimization problems for MDPs, which can be then solved using linear programming. The approach is inspired by~\cite{etessami2007multi,forejt2011quantitative,altman1996constrained}.
Particularly, denote by~$y_{s,u}$ the \emph{expected number of times} over the infinite horizon that the system is at state~$s$ and action~$u$ is taken, $\forall s\in S_n$ and~$\forall u\in U(s)$, which are often referred to as occupancy measures~\cite{altman1996constrained} as it holds $y_{s,u} = \sum_{t=0}^{\infty} {Pr}_{s_0}^{\boldsymbol{\pi}}[s_t=s,\,u_t=u]$,
where the probability is conditioned on a policy~$\boldsymbol{\pi}$ and the initial state~$s_0$. Note that an occupancy measure is a sum of probabilities, but not a probability itself. Consider the linear program:
\begin{subequations}\label{eq:LP1}
\begin{align}
&\underset{\{y_{s,u}\}}{\boldsymbol{\min}} \bigg{[} \textbf{C}_{\texttt{pre}}(S_c) \triangleq \sum_{(s,u)} \sum_{\check{s}\in S_{\texttt{p}}} y_{s,u} \, p_{\texttt{p}}(s,u,\check{s})\, c_{\texttt{p}}(s,u)\bigg{]} \label{eq:LP1_objective}\\
&\textrm{s.t.} \quad  \sum_{(s,u)} \sum_{\check{s}\in S_c} y_{s,u} \, p_{\texttt{p}}(s,u,\check{s}) \geq 1-\gamma; \label{eq:LP1_constraint1}\\
&\hspace{-0.2in}\sum_{u\in U(\check{s})} y_{\check{s},u} =\sum_{(s,u)} y_{s,u}\, p_{\texttt{p}}(s,u,\check{s})+ \mathbbm{1}(\check{s}=s_0), \; \forall \check{s}\in S_n; \label{eq:LP1_constraint2}\\
&\qquad  \; y_{s,u} \geq 0,\; \forall s \in S_n, \, \forall u \in U(s), \label{eq:LP1_constraint3}
\end{align}
\end{subequations}
where {$\sum_{(s,u)} \triangleq \sum_{s\in S_n} \sum_{u\in U(s)}$},  the indicator function~$\mathbbm{1}(\check{s}=s_0)=1$ if~$\check{s}=s_0$ and $\mathbbm{1}(\check{s}=s_0)=0$, otherwise. Denote by~$\text{C}_{\texttt{pre}}(S_c)$ the objective function associated with $S_c$. 
Let the solution of~\eqref{eq:LP1} be~$y_{\texttt{pre}}^\star=\{y_{s,u}^\star,\, s\in S_n, \, u\in U(s)\}$.
{Then the  optimal \emph{stationary} policy for the plan prefix, denoted by~$\boldsymbol{\pi}^\star_{\texttt{pre}}$, can be derived as follows}: the probability of choosing action~$u$ at state~$s$ equals to~$\boldsymbol{\pi}_{\texttt{pre}}^\star(s,\,u)=y^\star_{s,u}/(\sum_{u\in U(s)} y^\star_{s,u})$ if~$\sum_{u\in U(s)} y^\star_{s,u}\neq 0$; otherwise, the action at~$s$ can be chosen randomly,~$\forall s\in S_c$.

\begin{lemma}\label{lem:risk}
Given an optimal solution~$y_{\texttt{pre}}^\star$ of~\eqref{eq:LP1}, the associated policy~$\boldsymbol{\pi}^\star_{\texttt{pre}}$ ensures that~${Pr}_{s_0}^{\boldsymbol{\pi}^\star}(\Diamond S_c) \geq 1-\gamma$.
\end{lemma}
\begin{proof}
First,~$y_{s,u}$ is finite and well-defined since~$\mathcal{Z}_{\texttt{pre}}$ is transient with respect to $S_n$,. 
The second part of the proof is similar to Lemma~3.3 of~\cite{etessami2007multi}. The summation~$\sum_{(s,u)} \sum_{\check{s}\in S_c} y_{s,u} \, p_{\texttt{p}}(s,u,\check{s})$ is the expected number of times that~$\mathcal{Z}_{\texttt{pre}}$ transitions from any state in~$S_n$ into~$S_c$ for \emph{the first time}, under policy~$\boldsymbol{\pi}^\star_{\texttt{pre}}$ from the initial state~$s_0$. Since the system remains within~$S_c$ once it enters~$S_c$, the summation equals the probability of eventually reaching the set~$S_c$, which is lower-bounded by~$1-\gamma$. This completes the proof.
\end{proof}
\begin{figure}[t]
\begin{minipage}[t]{0.495\linewidth}
\centering
   \includegraphics[width =1.02\textwidth]{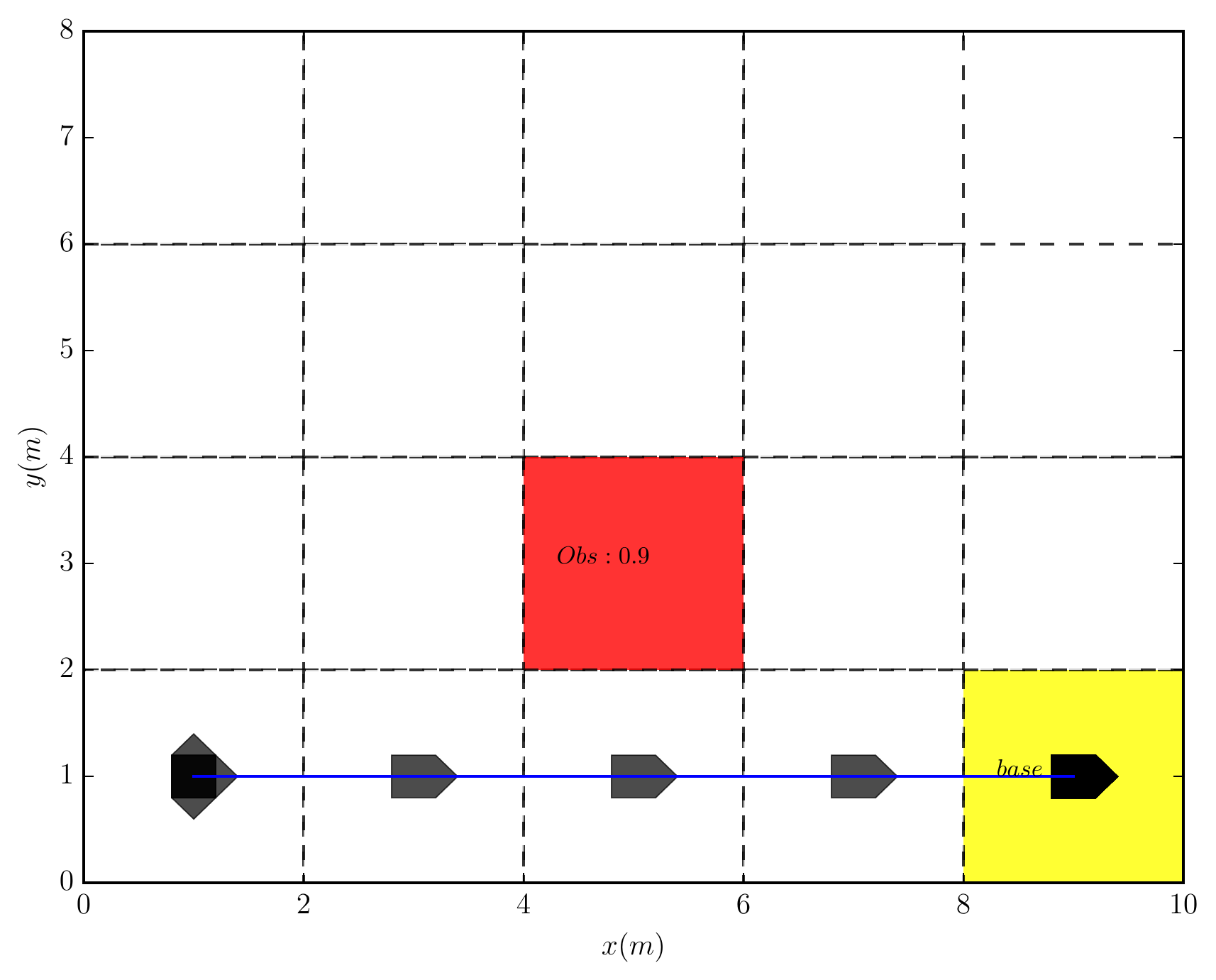}
  \end{minipage}
\begin{minipage}[t]{0.495\linewidth}
\centering
    \includegraphics[width =1.02\textwidth]{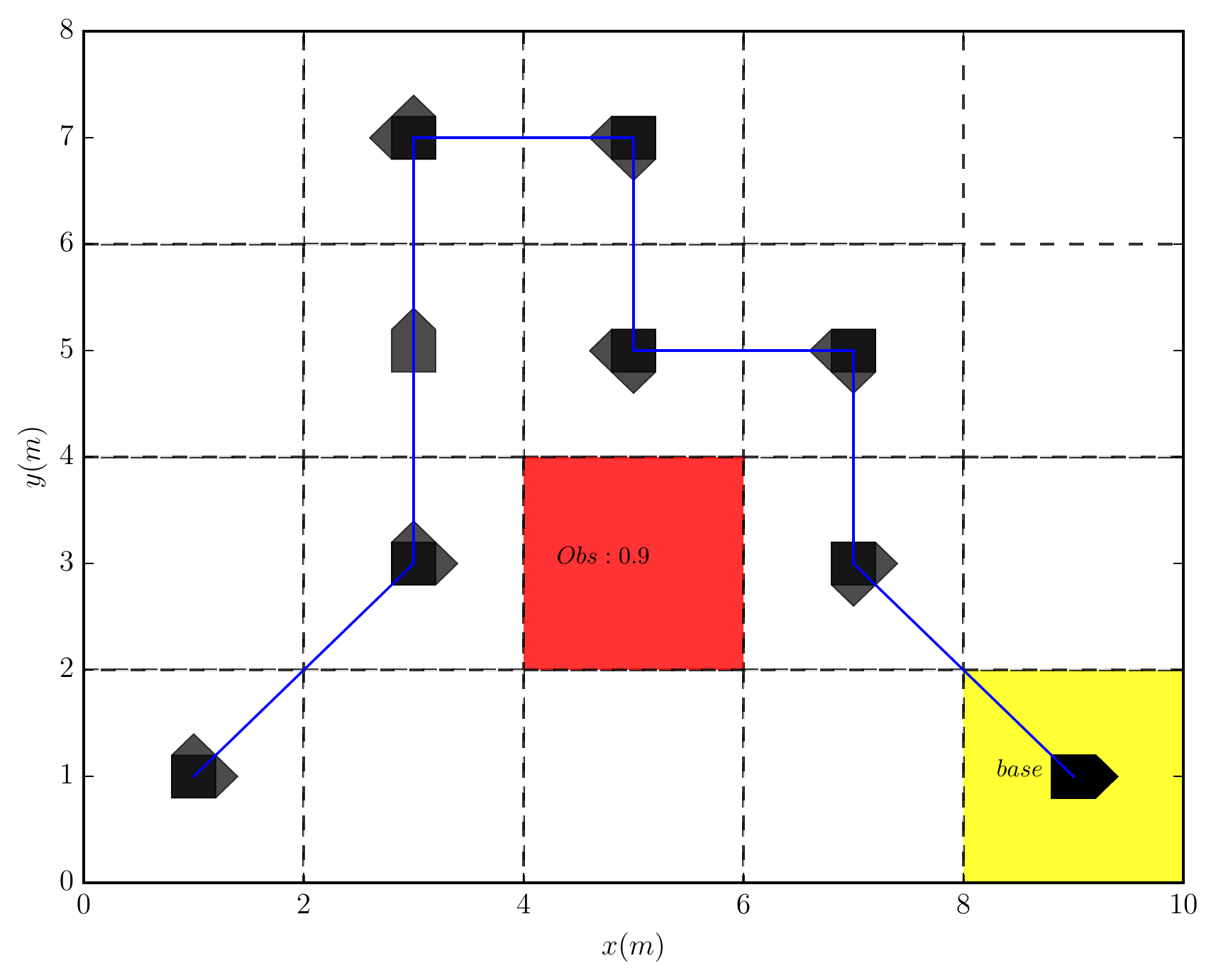}
  \end{minipage}
  \caption{Trajectories when setting~$\gamma=0.4$ (left) and~$\gamma=0$ (right). 
The task is to reach the yellow base  while avoiding the red cell. }
\label{fig:reach}
\end{figure}

\begin{example}\label{example:trade-off}
This example illustrates the important role of~$\gamma$ in the trade-off between reducing the expected total cost and minimizing the risk in Problem~\ref{problem:lp1}.
Consider the unicycle robot with action primitives illustrated in Figure~\ref{fig:action} and defined in Section~\ref{sec:example}.
The robot moves within partitioned cells as shown in Figure~\ref{fig:reach}, where the red cell has probability~$0.9$ to be occupied by an obstacle. 
Consider the task:~$\varphi=(\Diamond \square \texttt{b}) \wedge (\square \neg \texttt{obs})$, i.e., to reach the yellow base without crossing any obstacle.  
In what follows,  we solve~\eqref{eq:LP1} under risk factors~$\gamma = 0$ and~$\gamma = 0.4$ to derive two different optimal policies.
Figure~\ref{fig:reach} shows a shorter trajectory with lower expected total cost of about $12.6$ when a larger risk is allowed, compared with the right trajectory that avoids completely colliding with the obstacle, but with a much higher total cost of about $33.7$.
\hfill $\blacksquare$
\end{example}

\subsubsection{Plan Suffix with AMECs}\label{sec:plan-suffix}
In this section, we present an algorithm to synthesize the \emph{plan suffix} that minimizes the mean total cost within the AMECs, while ensuring that the system trajectory satisfies the accepting condition of~$\mathcal{P}$. {Note that the plan prefix~$\boldsymbol{\pi}^\star_{\texttt{pre}}$ from the previous section guarantees that the system enters $S_c$ from~$s_0$  with probability higher than~$1-\gamma$. Recall also that $S_c=\cup_{(S_c',U_c')\in \Xi_{acc}} S_c'$. Thus it is \emph{possible} that  the system enters \emph{any} set~$S_c'$ within $\Xi_{acc}$.
For this reason, we propose to treat each AMEC~$(S'_c,\,U'_c)\in \Xi_{acc}$ \emph{separately}, as each~$S_c'$ is associated with different~$U'_c$ and thus a different accepting condition for~$S'_c\cap I_{\mathcal{P}}^i$. }
Specifically, consider any AMEC~$(S_c',U_c')\in \Xi_{acc}$ and let~$I_c'\triangleq S'_c\cap I_{\mathcal{P}}'$, which is nonempty by the definition of an AMEC. 

Once the system enters any AMEC, 
most related work~\cite{ding2011mdp,forejt2011quantitative, baier2008principles} adopts the Round-Robin policy defined below:

\begin{definition}\label{def:round-robin}
For each state~$s_t\in S'_c$, create \emph{any} ordered sequence of actions from~$U_c'(s_t)$, denoted by~$\overline{U}(s_t)$ and its infinite repetition by~$\overline{U}^{\omega}(s_t)$.
Then at any stage~$t>0$, whenever the system reaches~$s_t\in S_c'$, the \emph{Round-Robin policy} instructs the system to take the \emph{next} action in $\overline{U}^{\omega}(s_t)$, starting from the first action in $\overline{U}^{\omega}(s_t)$.\hfill $\blacksquare$
\end{definition}

Namely, once the system enters~$S'_c$, the Round-Robin policy iterates over the allowed actions for each state, which in-turn ensures that all states in~$S'_c$ (which include~$I'_{c}$) are visited infinitely often. 
Detailed can be found in Lemma 10.119 in~\cite{baier2008principles}.

\begin{definition}\label{def:accepting-cycle}
An \emph{accepting cyclic path} of~$\mathcal{P}$, associated with~$S_c'$ and $I_c'$, is a finite path that starts from any state~$s_f\in I_c'$ and {ends in} any state~$s_g \in I_c'$, while remaining within~$S_c'$. \hfill $\blacksquare$
\end{definition}

Note that an accepting cyclic path does not necessarily start and end at the same state in~$I_c'$. Furthermore, we can define the mean cyclic cost of $\mathcal{P}$ under a stationary policy.

\begin{definition}\label{def:mean-cycle} 
The total cost of a cyclic path $P_a=s_0u_0s_1u_1\cdots s_{N_a}u_{N_a}$ is defined as
\begin{equation}\label{eq:cycle-total-cost}
\overline{\textbf{C}}_{\texttt{suf}}(P_a) \triangleq \sum_{t=0}^{N_a} c_D(s_t,u_t)
\end{equation}
where~$N_a\geq 1$ is the length of the path and $s_0,s_{N_a}\in I_c'$. 
Then its mean total cost is defined as $\textbf{C}_{\texttt{suf}}(P_a) \triangleq \frac{1}{N_a} \overline{\textbf{C}}_{\texttt{suf}}(P_a)$.
\hfill $\blacksquare$ 
\end{definition}

\begin{problem}\label{prob:suffix}
Find a {\emph{stationary} suffix} policy~$\boldsymbol{\pi}^\star_{\texttt{suf}}$ for~$\mathcal{P}$ within~$S_c'$  that minimizes the \emph{mean cyclic cost} 
\begin{equation}\label{eq:mean-cycle}
\textbf{C}_{\texttt{suf}}(S_c',U_c') = \mathbb{E}^{\boldsymbol{\pi}}_{P_a\in \mathbf{P}_a} \{\textbf{C}_{\texttt{suf}}(P_a)\}, 
\end{equation}
where $\mathbf{P}_a$ is the set of all accepting cyclic paths associated with the AMEC~$(S'_c,U'_c)$. \hfill $\blacksquare$
\end{problem}

{Inspired by~\cite{chatterjee2011two, chatterjee2011energy, randour2015percentile}, we formulate a Linear Program to solve the mean-payoff optimization problem.} First, we construct a modified sub-MDP~$\mathcal{Z}_{\texttt{suf}}$ of~$\mathcal{P}$ over~$S_c'$ by splitting~$I_c'$ into two virtual copies:~$I_{\texttt{in}}$ which only has incoming transitions into~$I_c'$ and~$I_{\texttt{out}}$ that has only outgoing transitions from~$I_c'$.
Formally, we define~$\mathcal{Z}_{\texttt{suf}}\triangleq (S_{\texttt{e}},\, U_{\texttt{e}},\, E_{\texttt{e}},\, y_0,\,p_{\texttt{e}},\,c_{\texttt{e}})$, 
where the set of states is~$S_{\texttt{e}}=(S_c'\backslash I_c')\cup I_{\texttt{in}} \cup I_{\texttt{out}}$ with~$I_{\texttt{in}}=\{s_{f}^{\texttt{in}},\, \forall s_{f}\in I_c'\}$ and~$I_{\texttt{out}} =\{s_{f}^{\texttt{out}},\, \forall s_{f}\in I_c'\}$ the virtual copies of~$I_c'$.
The set of control actions is~$U_{\texttt{e}}=U\cup\{\tau_0\}$, where~$\tau_0$ is a self-loop action. The set of state-action pairs~$E_{\texttt{e}}\subset S_{\texttt{e}}\times U_{\texttt{e}}$ is defined by (i)~$(s,u)\in E_{\texttt{e}}$,~$\forall s\in S_c'\backslash I_c' $ and~$u\in U_c'(s)$; (ii)~$(s,\tau_0)\in E_{\texttt{e}}$,~$\forall s\in I_{\texttt{in}}$; and (iii)~$(s_f^{\texttt{out}},u)\in E_{\texttt{e}}$,~$\forall s_f \in I_c'$ and~$u\in U_c'(s_f)$. Moreover, $y_0$ is the initial distribution of all states in~$S_c'$ that can be reached by taking a transition from states in~$S_n'$, defined by
$$y_0(s)=\sum_{\check{s}\in S_n'} \sum_{u\in U_{\texttt{p}}(\check{s})}p_{\texttt{p}}(\check{s},u,s) y_{\texttt{pre}}(\check{s},u), \forall s\in (S_c'\backslash I_c')\cup I_{\texttt{out}},$$
where~$\{y_{\texttt{pre}}(s,u)\}$  are the variables of~\eqref{eq:LP1}. 
Furthermore, the transition probability~$p_{\texttt{e}}$ is defined in five cases below:
(a) for transitions within~$S_c'\backslash I_c'$, it holds that~$p_{\texttt{e}}(s,u,\check{s})=p_E(s,u,\check{s})$, $\forall s,\check{s} \in  S_c'\backslash I_c'$, $\forall u\in U_{\texttt{e}}(s)$; 
(b) for transitions originated from~$I_{\texttt{out}}$, it holds that~$p_{\texttt{e}}(s_{f}^{\texttt{out}},u,\check{s})=p_E(s_f,u,\check{s})$, $\forall s_{f}^{\texttt{out}}\in I_{\texttt{out}}$, $\forall u\in U_{\texttt{e}}(s_{f}^{\texttt{out}})$ and $\forall \check{s} \in S_c'\backslash I_c'$; 
(c) for transitions into~$I_{\texttt{in}}$, it holds that~$p_{\texttt{e}}(s,u,s_{f}^{\texttt{in}})=p_E(s,u,s_f)$, $\forall s \in S_c'\backslash I_c'$,~$\forall u\in U_{\texttt{e}}(s)$ and~$\forall s_f^{\texttt{in}}\in I_{\texttt{in}}$; 
(d) for transitions from~$I_{\texttt{out}}$ to~$I_{\texttt{in}}$, it holds that~$p_{\texttt{e}}(s_{f}^{\texttt{out}},u,s_{f}^{\texttt{in}})=p_E(s_f,u,s_f)$, $\forall s_f^{\texttt{out}} \in I_{\texttt{out}}$ and~$\forall u\in U_{\texttt{e}}(s_f^{\texttt{out}})$;
and (e) for transitions within~$I_{\texttt{in}}$,~$p_{\texttt{e}}(s_{f}^{\texttt{in}},\tau_0,s_{f}^{\texttt{in}})=1$, $\forall s_{f}^{\texttt{in}}\in I_{\texttt{in}}$.
Lastly, the cost function satisfies~$c_{\texttt{e}}(s,u)= c_E(s,u)$, $\forall s\in (S_{\texttt{e}}\backslash I_{\texttt{in}})$,~$\forall u\in U_{\texttt{e}}(s)$, and~$c_{\texttt{e}}(s_{f}^{\texttt{in}},\tau_0)=0$, $\forall s_{f}^{\texttt{in}} \in I_{\texttt{in}}$. 

\begin{remark}\label{remark:initial-dist}
The initial distribution~$y_0$ of~$\mathcal{Z}_{\texttt{suf}}$ indicates how likely it is that the system controlled by the plan prefix~$\boldsymbol{\pi}^\star_{\texttt{pre}}$ will enter the AMEC~$(S'_c,\,U_c')$ via each state inside~$S_c'$. \hfill $\blacksquare$
\end{remark}

Let also $S_{\texttt{e}}'\triangleq S_{\texttt{e}}\backslash I_{\texttt{in}}$ and denote by~$z_{s,u}$ the \emph{long-run frequency}  with which the system is at state~$s$ and  the action~$u$ is applied,~$\forall s\in S_{\texttt{e}}'$ and~$\forall u\in U_{\texttt{e}}(s)$. Then, we can formulate the following linear program to solve Problem~\ref{prob:suffix}:
\begin{subequations}\label{eq:LP2}
\begin{align}
&\hspace{-0.15in} \underset{\{z_{s,u}\}}{\boldsymbol{\min}} \bigg[ \textbf{C}_{\texttt{suf}}(S_c',U_c') \triangleq \sum_{(s,u)}\sum_{\check{s}\in S_{\texttt{e}}}  z_{s,u}\, p_{\texttt{e}}(s,u,\check{s})\, c_{\texttt{e}}(s,u) \bigg] \label{eq:LP2_objective}\\
& \textrm{s.t.} \; \sum_{(s,u)} \; \sum_{\check{s}\in I_{\texttt{in}}} z_{s,u}\,p_{\texttt{e}}({s},u,\check{s}) = \sum_{s\in S_{\texttt{e}}'} y_0(s); \label{eq:LP2_constraint2}\\
&\sum_{u\in U_{\texttt{e}}(s)} z_{s,u}  = \sum_{(\check{s},u)} z_{\check{s},u}\,p_{\texttt{e}}(\check{s},u,s)+y_0(s), \; \forall s\in S_{\texttt{e}}'; \label{eq:LP2_constraint1}\\
&\; z_{s,u} \geq 0, \; \forall s\in S_{\texttt{e}}',\; \forall u \in U_{\texttt{e}}(s); \label{eq:LP2_constraint3}
\end{align}
\end{subequations}
where {$\sum_{(s,u)}\triangleq \sum_{s\in S_{\texttt{e}}'} \sum_{u\in U_{\texttt{e}}(s)}$}, the first constraint ensures that~$I_{\texttt{in}}$ is eventually reached, while the second constraint balances the incoming and outgoing flow at each state. 
Let its solution be~$z_{\texttt{suf}}^\star=\{z_{s,u}^\star,\, \forall s\in S_{\texttt{e}}', \, \forall u\in U_{\texttt{e}}(s)\}$.
{Then, the optimal \emph{stationary} policy for the plan suffix, denoted by~$\boldsymbol{\pi}^\star_{\texttt{suf}}$, can be derived as follows}:
the probability of choosing action~$u$ at state~$s$ equals to~$\boldsymbol{\pi}_{\texttt{suf}}^\star(s,u)=z^\star_{s,u}/(\sum_{u\in U_{\texttt{e}}(s)} z^\star_{s,u})$ if~$ \sum_{u\in U_{\texttt{e}}(s)} z^\star_{s,u} \neq 0$;
otherwise the action at~$s$ is chosen randomly, $\forall s\in S_{\texttt{e}}'$. Note that~$\boldsymbol{\pi}_{\texttt{suf}}^\star(s_f,u)=\boldsymbol{\pi}_{\texttt{suf}}^\star(s_f^{\texttt{out}},u)$, $\forall s_f\in I_c'$ and~$\forall u\in U_c'(s_f)$. {Namely, once the system reaches any state~$s_g\in I_{c}'$, the control policy at~$s_g$ will be the control policy for~$s_g^{\texttt{out}}\in I_{\texttt{out}}$, according to the solution of~\eqref{eq:LP2}.}

\begin{remark}\label{remark:compare-with-fu}
The initial distribution is derived from~\eqref{eq:LP1}, instead of being arbitrarily set as in~\cite{fu2015pareto}; Moreover,~\eqref{eq:LP2_constraint2} ensures that only~$I_c'$ is intersected infinitely often, instead of enforcing that \emph{all} states in the set~$S_c'$ are visited infinitely often as in~\cite{fu2015pareto}.\hfill $\blacksquare$  
\end{remark}
\begin{lemma}\label{lem:suffix}
If~\eqref{eq:LP2} has a solution, then the plan suffix~$\boldsymbol{\pi}^\star_{\texttt{suf}}$ solves Problem~\ref{prob:suffix} for the chosen AMEC~$(S_c',U_c')\in \Xi_{acc}$. 
\end{lemma}
\begin{proof}
First, by Definition~\ref{def:accepting-cycle}, the objective in~\eqref{eq:LP2} equals the mean cyclic cost of all accepting cyclic paths for~$I_c'$.
Moreover, by the definition of an AMEC, any  path~remains within~$S_{\texttt{e}}'$ by choosing only actions  within~$U_c'(s)$ at each state~$s\in S_{\texttt{e}}'$. 
\end{proof}

\begin{figure}[t]
\begin{minipage}[t]{0.495\linewidth}
\centering
   \includegraphics[width =0.97\textwidth]{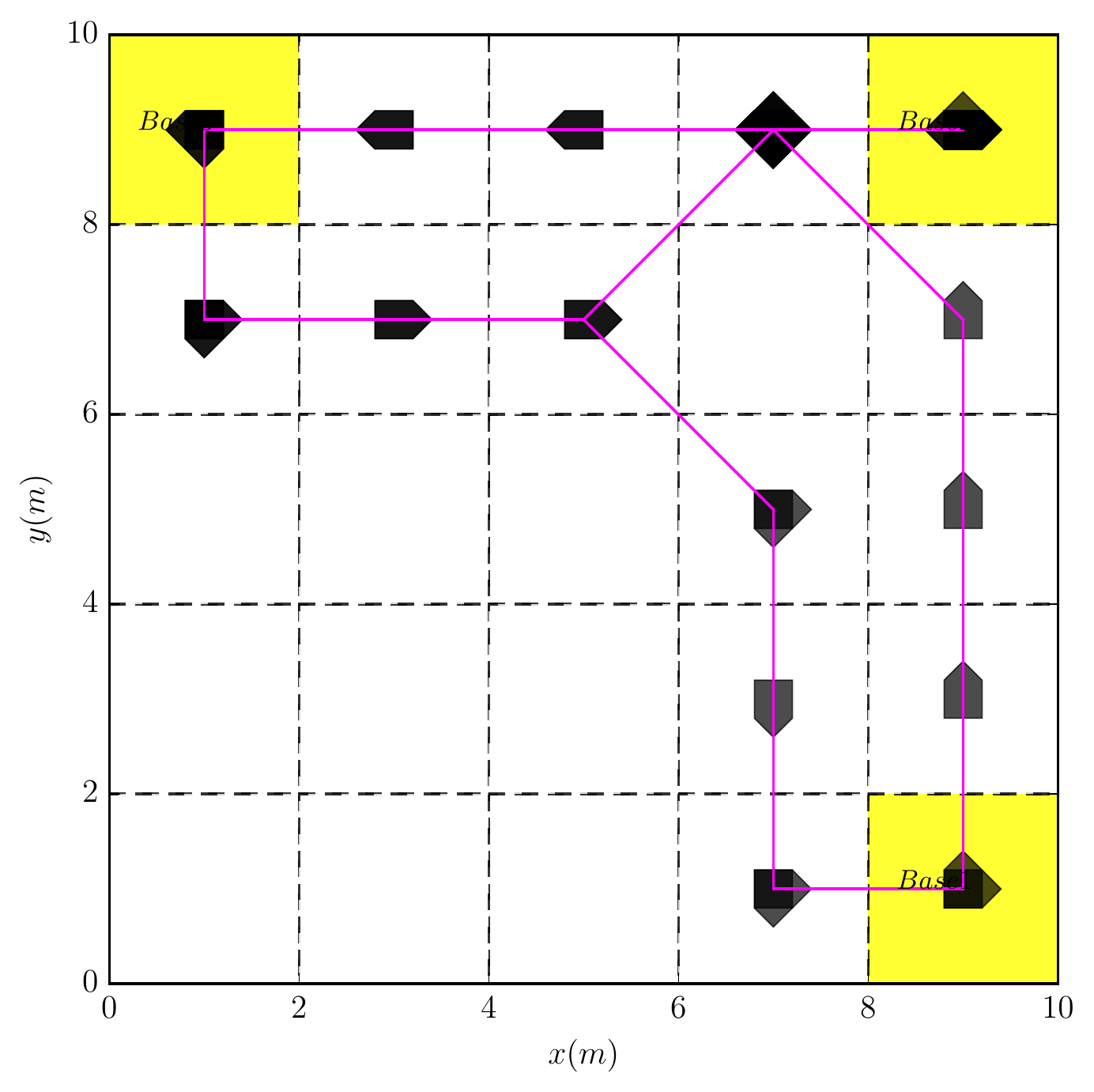}
  \end{minipage}
\begin{minipage}[t]{0.495\linewidth}
\centering
    \includegraphics[width =0.97\textwidth]{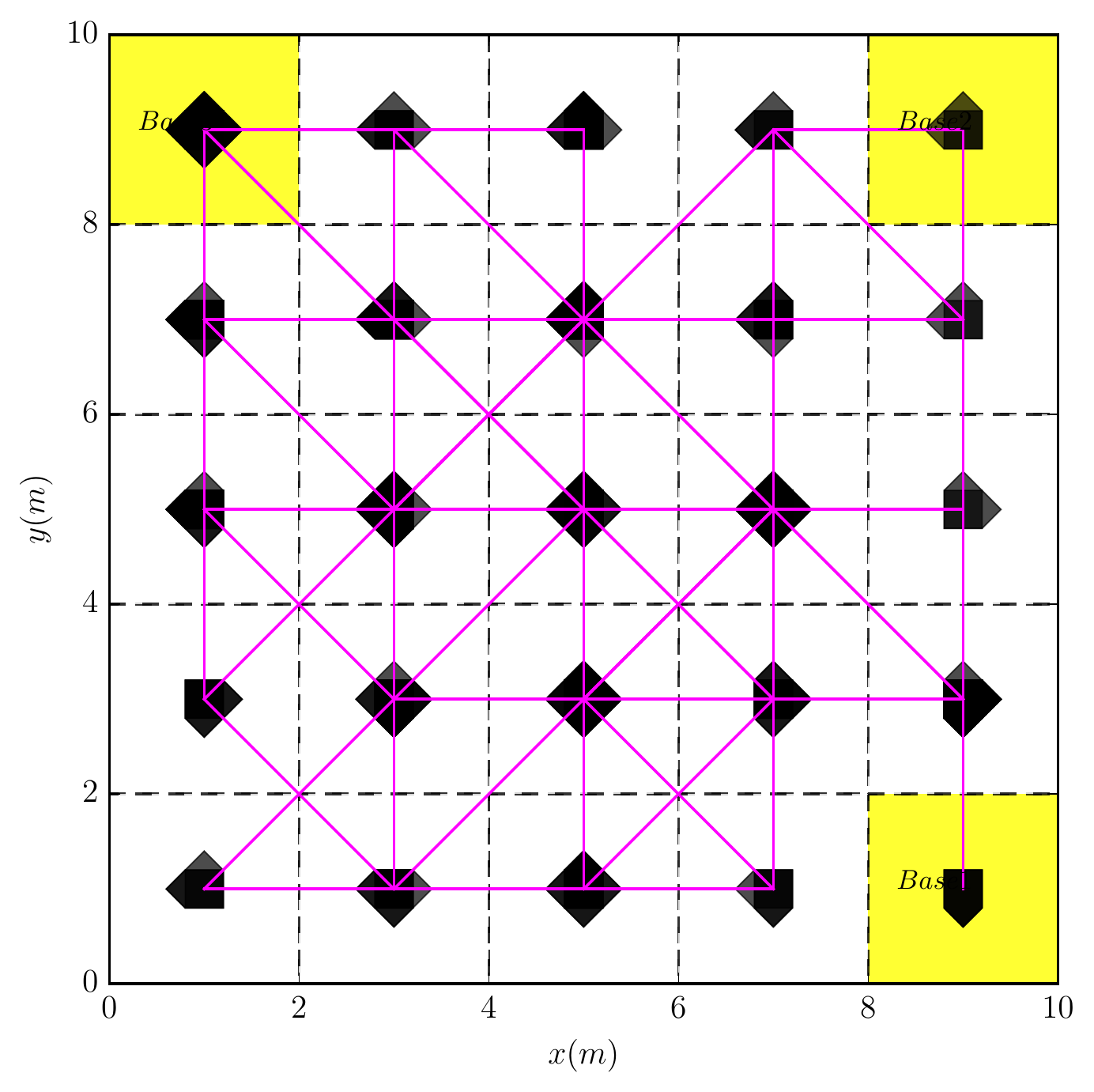}
  \end{minipage}
  \caption{Simulated trajectory  under~$\boldsymbol{\pi}_{\texttt{suf}}^\star$ (left) and under the Round-Robin policy (right), see Example~\ref{example:importance-suffix}.}
\label{fig:suffix}
\end{figure}

\begin{lemma}\label{lem:rabin-suffix}
Let $\boldsymbol{\tau}_{\mathcal{P}}$ be the set of all accepting runs of~$\mathcal{P}$ that enter~$S_c'$ after a finite number of steps. If~$\tau_{\mathcal{P}}\in \boldsymbol{\tau}_{\mathcal{P}}$ is generated under~$\boldsymbol{\pi}^\star_{\texttt{suf}}$, then~$\tau_{\mathcal{P}}$ satisfies the accepting condition of~$\mathcal{P}$. 
{Moreover, the mean total cost in~\eqref{eq:total-cost} equals the mean cyclic cost in~\eqref{eq:mean-cycle}, i.e.,}
$\mathbb{E}_{\tau_{\mathcal{P}}\in \boldsymbol{\tau}_{\mathcal{P}}} \{\textbf{\textup{Cost}}(\tau_{\mathcal{P}})\} = \textbf{\textup{C}}_{\texttt{suf}}(S'_c,\, U'_c)$.
\end{lemma}
\begin{proof}
By~\eqref{eq:LP2}, any system trajectory of~$\mathcal{P}$ under~$\boldsymbol{\pi}^\star_{\texttt{suf}}$ contains infinite occurrences of accepting cyclic paths. Since any accepting cyclic path starts from and ends in~$I_c'$ (which is finite), $\tau_{\mathcal{P}}$ intersects with~$I_c'$ infinitely often. Moreover, since any accepting cyclic path remains within~$S_c'$, $\tau_{\mathcal{P}}$ remains within~$S_c'$ for all time after entering~$S_c'$. 
In other words,~$\tau_{\mathcal{P}}$ intersects with $H^i_{\mathcal{P}}$ a finite number of times before entering~$S_c'$ and then intersects~$I^i_{\mathcal{P}}$ infinitely often after entering~$S_c'$, which satisfies the Rabin accepting condition of~$\mathcal{P}$. 
{To show the second part, notice that the product $\mathcal{P}$ under~$\boldsymbol{\pi}_{\texttt{suf}}^\star$ evolves as a Markov chain and the set of all accepting cyclic paths within~$S_c'$ has a stationary distribution. By viewing any accepting run~$\tau_{\mathcal{P}}$ as the \emph{concatenation} of an infinite number of cyclic paths, the mean total cost of $\tau_{\mathcal{P}}$ defined in~\eqref{eq:objective} over an infinite time horizon equals the mean cyclic cost in~\eqref{eq:mean-cycle} of all cyclic paths contained in $\tau_{\mathcal{P}}$. This result is important in showing the equivalence between Problems~\ref{prob:main} and~\ref{prob:suffix} later in Theorem~\ref{thm:whole}.}
\end{proof}

\begin{example}\label{example:importance-suffix}
This example illustrates the difference between the plan suffix obtained by~\eqref{eq:LP2} and the Round-Robin policy. 
Consider the same robot model from Example~\ref{example:trade-off} and the partitioned workspace in Figure~\ref{fig:suffix}.
The task is to surveil three base stations in the corners, i.e.~$\varphi=(\square \Diamond \texttt{b1})\wedge (\square \Diamond \texttt{b2}) \wedge (\square \Diamond \texttt{b3})$. 
The plan prefix is derived by solving~\eqref{eq:LP1} but two different plan suffixes are used: one using~\eqref{eq:LP2} and the Round-Robin policy. 
Figure~\ref{fig:suffix} shows the simulated trajectory under these two policies. It can be seen that the trajectory under the optimal plan suffix approximates the shortest route to cross all base stations, while the trajectory under the Round-Robin policy exhibits a rather random behavior. \hfill $\blacksquare$  
\end{example}

\subsubsection{Plan Synthesis when AECs do Not Exist}\label{sec:non-AEC}
The synthesis algorithms proposed in Sections~\ref{sec:plan-prefix} and~\ref{sec:plan-suffix} rely on the assumption that the set of AMECs~$\Xi_{acc}$ of~$\mathcal{P}$ is nonempty which, however, might not hold in many scenarios. 
In this case, most existing techniques proposed in~\cite{baier2008principles, forejt2011quantitative, ding2011mdp, ding2014optimal} can not be applied.
In this section, we first provide a simple example where no AECs exist, and then propose an approach to synthesize a \emph{relaxed} plan prefix and suffix.

\begin{example}\label{example:non-AEC}
This example provides a robot model~$\mathcal{M}$ and its task~$\varphi$ for which no AECs exist in the product automaton~$\mathcal{P}$.  
Consider the MDP~$\mathcal{M}$ in Figure~\ref{fig:toy-example} that transitions between two states~($S_1$, $S_2$) with probability~$1$ using the action~$f$. Note that $S_1$ has only probability~$0.01$ of being occupied by an obstacle and~$S_2$ is the base station. 
The task is to surveil the base station while avoiding obstacles, i.e.,~$\varphi=(\square \Diamond \texttt{b})\wedge (\square \neg \texttt{obs})$. The associated DRA is shown in Figure~\ref{fig:toy-example}. The resulting~$\mathcal{P}$ is shown in Figure~\ref{fig:toy-product}, where the set of states~$H_i^{\mathcal{P}}$ to avoid in the suffix is in red and the set of states~$I_i^{\mathcal{P}}$ to intersect infinitely often in green. The reason that no AECs exist in~$\mathcal{P}$ is because by definition an AEC~$(S',\, \{f\})$ should include \emph{all} successor states that are reachable by the single action~$f$. Then, starting from any green state in~$I_i^{\mathcal{P}}$,  the set of reachable states eventually intersect with the red states in~${H}_i^{\mathcal{P}}$. \hfill $\blacksquare$
\end{example}

When no AECs exist in~$\mathcal{P}$, the probability of satisfying the task under \emph{any} policy is \emph{zero}. However, it is still important to identify those policies that ensure high probability of avoiding bad states over long time intervals.
Consequently, we propose to use an accepting SCC (ASCC)  of~$\mathcal{P}$ as the \emph{relaxed} AMEC, due to the following lemma.

\begin{figure}[t]
\begin{minipage}[t]{0.4\linewidth}
\centering
   \includegraphics[width =1.01\textwidth]{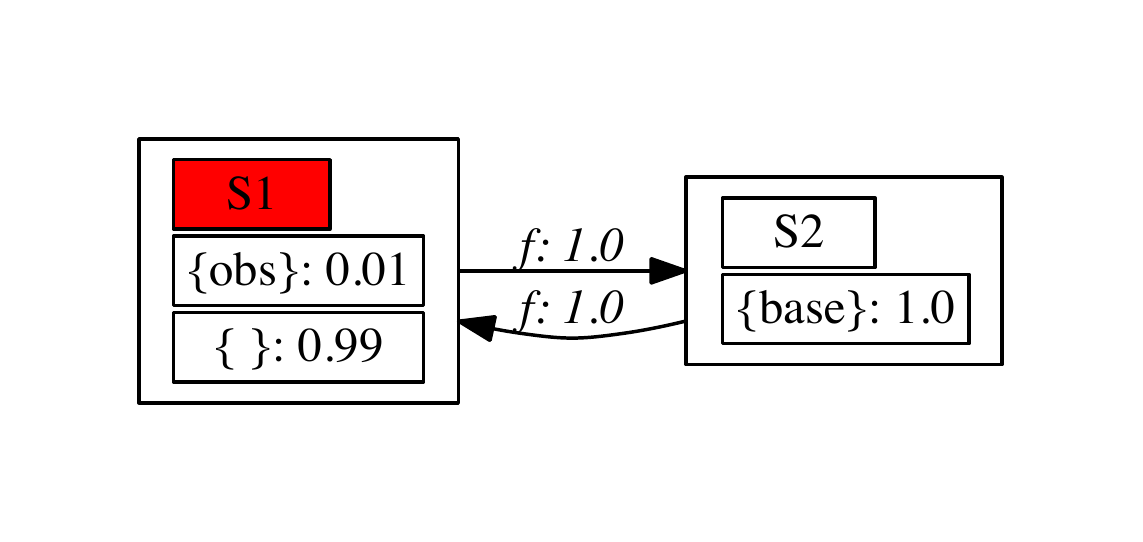}
  \end{minipage}
\begin{minipage}[t]{0.6\linewidth}
\centering
    \includegraphics[width =1.0\textwidth]{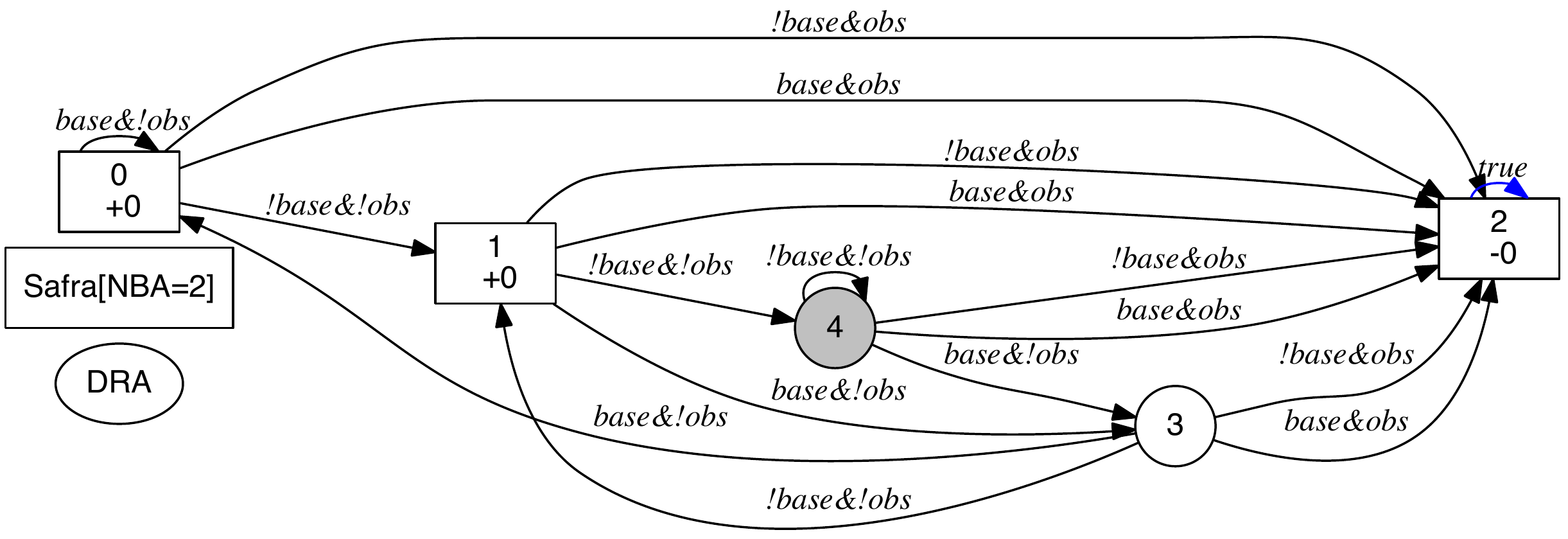}
  \end{minipage}
  \caption{
The MDP~$\mathcal{M}$ (left) and DRA~$\mathcal{A}_{\varphi}$ (right, derived via~\cite{klein2007ltl2dstar, git_mdp_tg}) described in Example~\ref{example:non-AEC}, with one accepting pair~$(\{2\},\{0,1\})$.}
\label{fig:toy-example}
\end{figure}

\begin{figure}[t]
     \centering
     \includegraphics[width=0.45\textwidth]{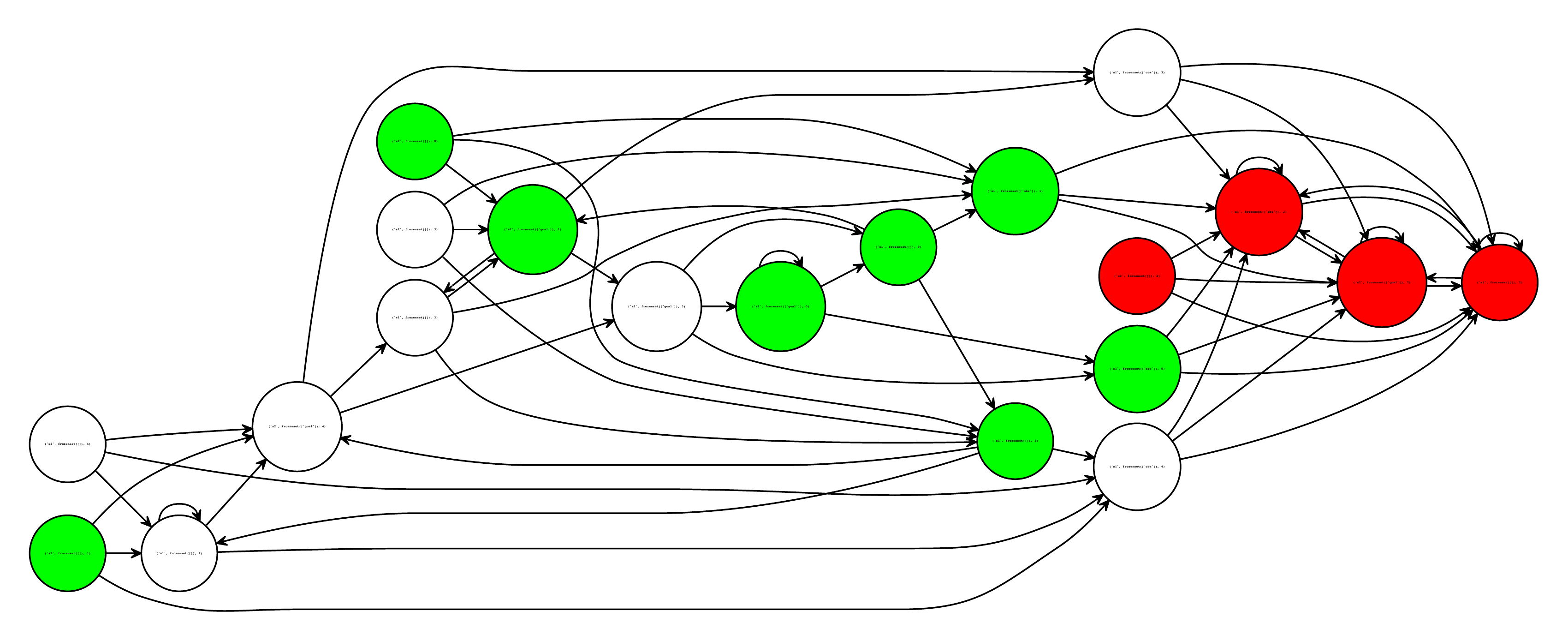}
     \caption{The product~$\mathcal{P}$ of~$\mathcal{M}$ and~$\mathcal{A}_{\varphi}$ in Figure~\ref{fig:toy-example}. The state and edge names are omitted as the structure is of importance here. At least one green state should be visited infinitely often while avoiding all red states. Note all transitions are driven by the action~$f$.}
     \label{fig:toy-product}
\end{figure}

\begin{lemma}\label{lem:exist}
Assume there exists one infinite path of~$\mathcal{P}$ that is accepting. Then, there exists at least one SCC of~$\mathcal{P}$ that intersects with~$I^i_{\mathcal{P}}$ but not with~$H^i_{\mathcal{P}}$, for at least one pair~$(H^i_{\mathcal{P}},\, I^i_{\mathcal{P}})\in \textup{Acc}_{\mathcal{P}}$. 
\end{lemma}
\begin{proof}
As mentioned before, an infinite path of~$\mathcal{P}$, denoted by~$R_{\mathcal{P}}$, is accepting if for at least one pair~$(H^i_{\mathcal{P}}, I^i_{\mathcal{P}})\in \text{Acc}_{\mathcal{P}}$ it holds that~$R_{\mathcal{P}}$ intersects with all states in~$H^i_{\mathcal{P}}$ finitely often while with~$I^i_{\mathcal{P}}$ infinitely often.
Since both~$H^i_{\mathcal{P}}$ and~$I^i_{\mathcal{P}}$ are finite, there exists a cyclic path~$s_k \cdots s_f \cdots s_k$ of~$\mathcal{P}$ that contains at least one~$s_f \in I^i_{\mathcal{P}}$ and does not contain any state within~$H^i_{\mathcal{P}}$.
By definition, this cyclic path is a SCC of~$\mathcal{P}$ that intersects with~$I^i_{\mathcal{P}}$ but not with~$H^i_{\mathcal{P}}$.
This completes the proof.
\end{proof}

Denote the set of SCCs in~$\mathcal{P}$ as~$\Omega \triangleq \{S_1', S_2',\cdots,S_C'\}$, where~$S_c'\subseteq S$. This set can derived using Tarjan’s algorithm~\cite{baier2008principles, git_mdp_tg}.
Moreover, denote by~$\Omega^i_{acc}=\{S_c'\in \Omega\,|\, S_c'\cap I^i_{\mathcal{P}}\neq \emptyset,\, S_c'\cap H^i_{\mathcal{P}} = \emptyset\}$  the set of SCCs that satisfy the accepting conditions associated with~$(H^i_{\mathcal{P}},\, I^i_{\mathcal{P}})\in \text{Acc}_{\mathcal{P}}$. 
Lemma~\ref{lem:exist} ensures that $\Omega^i_{acc}\neq \emptyset$ for at least one pair~$(H^i_{\mathcal{P}},\, I^i_{\mathcal{P}})\in \text{Acc}_{\mathcal{P}}$. 
Therefore, the union~$\Omega_{acc} \triangleq \cup_{i=1,\cdots,N}\, \Omega^i_{acc}$ is not empty.

Now the union~$S_c\triangleq \cup_{S'_c\in \Omega_{acc}}S'_c$ serves as the set of states the system should enter, starting from the initial state, and then remain inside any of the ASCC to satisfy the accepting condition.
Again the first step is to formulate a Linear Program that minimizes the expected total cost of reaching~$S_c$ from~$s_0$, while ensuring the risk is upper-bounded by the chosen~$\gamma_{\texttt{prex}}>0$.
It can be done analogously as in~\eqref{eq:LP1} but over~$S_n\triangleq S\backslash S_c$ (which is omitted here).
Denote the objective function by $\textup{\textbf{C}}_{\texttt{prex}}(S_c)$ and its set of variables by~$\{y_{\texttt{prex}}(s,u)\}$ and the associated relaxed plan prefix as~$\boldsymbol{\pi}_{\texttt{prex}}$. Same as in Section~\ref{sec:plan-suffix}, it is possible that the system under the policy~$\boldsymbol{\pi}_{\texttt{prex}}$ can enter any ASCC in $\Omega_{acc}$. Assume that the system enters~$S_c'\in \Omega_{acc}$.
Different from an AMEC~$(S_c',U_c')\in \Xi_{acc}$, the action set at each state of~$S_c'\in \Omega_{acc}$ is \emph{not} constrained. Thus, there is no guarantee that the system will stay within~$S_c'$ after entering it.

Therefore, the second step is to synthesize the relaxed plan suffix that keeps the system inside~$S_c'$ to satisfy the accepting condition with the maximal probability. 
Define the set~$I_c'=S_c'\cap I^i_{\mathcal{P}}$, which is not empty 
for an ASCC~$S_c'$. Then, an accepting cyclic path of~$\mathcal{P}$ associated with~$I_c'$, and the cyclic cost associated with~$S_c'$ and~$I_c'$ can be defined similarly as in Definition~\ref{def:accepting-cycle}.
Formally, we consider the following problem:

\begin{problem}\label{prob:suffix_rex}
Find a control policy for~$\mathcal{P}$ that minimizes the mean cyclic cost associated with the ASCC~$S_c'$:
$\mathbb{E}^{\boldsymbol{\pi}}_{P_a\in \mathbf{P}_a} \{\textbf{C}_{\texttt{sufx}}(P_a)\}$, 
where~$\mathbf{P}_a$ is the set of all accepting cyclic paths associated with $S'_c$ and $\textbf{C}_{\texttt{sufx}}$ is defined as in Definition~\ref{def:mean-cycle}; while \emph{at the same time} maximizing the probability that the cyclic paths stay within~$S_c'$.
\hfill $\blacksquare$
\end{problem}

{In Problem~\ref{prob:suffix_rex}, the first objective of minimizing the mean cyclic cost corresponds to minimizing the mean total cost in~\eqref{eq:objective} in Problem~\ref{prob:main}. The objective of maximizing the probability of the system staying within the ASCC $S'_c$ corresponds to minimizing the frequency with which the system will reach the bad states that violate the task specifications. It constitutes a relaxation of the risk constraint \eqref{eq:objective} in Problem~\ref{prob:main}.}
To solve Problem~\ref{prob:suffix_rex}, first we construct a modified MDP~$\mathcal{Z}_{\texttt{sufx}}$ over~$S_c'$, which is similar to~$\mathcal{Z}_{\texttt{suf}}$ in Section~\ref{sec:plan-suffix}. 
The set~$I_c'$ is split into two virtual copies:~$I_{\texttt{in}}$ which only has incoming transitions and~$I_{\texttt{out}}$ that has only outgoing transitions.
Formally, we define~$\mathcal{Z}_{\texttt{sufx}}=(S_{\texttt{r}},\, U_{\texttt{r}},\, E_{\texttt{r}},\, y_0,\,p_{\texttt{r}},\,c_{\texttt{r}})$, 
where the set of states is~$S_{\texttt{r}}=(S_c'\backslash I_c')\cup I_{\texttt{in}} \cup I_{\texttt{out}} \cup \{s_{bad}\}$, with~$I_{\texttt{in}}=\{s_{f}^{\texttt{in}},\, \forall s_{f}\in I_c'\}$ and~$I_{\texttt{out}} =\{s_{f}^{\texttt{out}},\, \forall s_{f}\in I_c'\}$ the two virtual copies of~$I_c'$, and~$s_{bad}$ is a virtual bad state.
The set of control actions is given by~$U_{\texttt{r}}=U\cup\{\tau_0\}$, where~$\tau_0$ is a self-loop action. 
The set of transition is~$E_{\texttt{r}}\subset S_{\texttt{r}}\times U_{\texttt{r}}$ which satisfies that (i)~$(s,u)\in E_{\texttt{r}}$,~$\forall s\in S_c'$ and~$u\in U(s)$; (ii)~$(s,\tau_0)\in E_{\texttt{r}}$,~$\forall s\in I_{\texttt{in}}$; and (iii)~$(s_{bad},\tau_0)\in E_{\texttt{r}}$.
Moreover, $y_0$ is the initial distribution of states in~$S_c'$ based on the transition from states in~$S_n'$:
$$y_0(s)=\sum_{(\check{s},u)}p_{\texttt{p}}(\check{s},u,s)\, y_{\texttt{prex}}(\check{s},u),\; \forall s\in (S_c'\backslash I_c')\cup I_{\texttt{out}},$$
where~$\sum_{(\check{s},u)}\triangleq \sum_{\check{s}\in S_n'} \sum_{u\in U_{\texttt{p}}(\check{s})}$ and~$\{y_{\texttt{prex}}(s,u)\}$ are the variables solutions from the synthesis of the relaxed plan prefix, and~$y_0(s_{bad}) =0$.
Furthermore, the transition probability~$p_{\texttt{r}}$ is defined in seven cases below:
(a) for  transitions within~$S_c'\backslash I_c'$, it holds that~$p_{\texttt{r}}(s,u,\check{s})=p_E(s,u,\check{s})$, $\forall s,\check{s} \in  S_c'\backslash I_c'$, $\forall u\in U_{\texttt{r}}(s)$;
(b) for transitions originated from~$I_{\texttt{out}}$, it holds that~$p_{\texttt{r}}(s_{f}^{\texttt{out}},u,\check{s})=p_E(s_f,u,\check{s})$, $\forall s_{f}^{\texttt{out}}\in I_{\texttt{out}}$, $\forall u\in U_{\texttt{r}}(s_{f}^{\texttt{out}})$ and $\forall \check{s} \in S_c'\backslash I_c'$;
(c) for transitions into~$I_{\texttt{in}}$, it holds that~$p_{\texttt{r}}(s,u,s_{f}^{\texttt{in}})=p_E(s,u,s_f)$, $\forall s \in S_c'\backslash I_c'$,~$\forall u\in U_{\texttt{r}}(s)$ and~$\forall s_f^{\texttt{in}}\in I_{\texttt{in}}$;
(d) for transitions from~$I_{\texttt{out}}$ to~$I_{\texttt{in}}$, it holds that~$p_{\texttt{r}}(s_{f}^{\texttt{out}},u,s_{f}^{\texttt{in}})=p_E(s_f,u,s_f)$, $\forall s_f^{\texttt{out}} \in I_{\texttt{out}}$ and~$\forall u\in U_{\texttt{r}}(s_f^{\texttt{out}})$;
(e) for transitions into the bad state~$s_{bad}$, it holds that~$p_{\texttt{r}}(s,u,s_{bad})=p_E(s,u,\check{s})$, $\forall s \in S_c'\backslash I_{\texttt{in}}$, $\forall \check{s}\in S\backslash S_c'$ and~$u\in U_{\texttt{r}}(s)$;
(f) each state within~$I_{\texttt{in}}$ is included in a self-loop such that~$p_{\texttt{r}}(s_{f}^{\texttt{in}},\tau_0,s_{f}^{\texttt{in}})=1$, $\forall s_{f}^{\texttt{in}}\in I_{\texttt{in}}$;
(g) the bad state is included in a self-loop such that~$p_{\texttt{r}}(s_{bad},\tau_0,s_{bad})=1$.
Finally, the cost function~$c_{\texttt{r}}$ is defined in two cases: (i)~$c_{\texttt{r}}(s,u)= c_E(s,u)$, $\forall s\in S_{\texttt{r}}\backslash I_{\texttt{in}}$,~$\forall u\in U_{\texttt{r}}(s)$; and (ii)~$c_{\texttt{r}}(s_{f}^{\texttt{in}},\tau_0)=0$, $\forall s_{f}^{\texttt{in}} \in I_{\texttt{in}}$ and~$c_{\texttt{r}}(s_{bad},\tau_0)=0$.

\begin{remark}\label{remark:z-relax}
Note that~$E_{\texttt{r}}$ contains all actions for each state in~$S_c'$, compared with~$E_{\texttt{e}}$ as allowed by the AMEC.
\hfill $\blacksquare$
\end{remark}

Let~$S_{\texttt{r}}'\triangleq S_{\texttt{r}}\backslash (I_{\texttt{in}}\cup \{s_{bad}\})$ and $S_{\texttt{r}}''\triangleq S_{\texttt{r}}\backslash \{s_{bad}\}$.
We can also show that~$\mathcal{Z}_{\texttt{sufx}}$ above is~$S_{\texttt{r}}'-$transient.
Then, to solve Problem~\ref{prob:suffix_rex}, we rely on a technique proposed in~\cite{trevizan2016heuristic} to deal with dead ends in Stochastic Shortest Path (SSP) problems. 
First we introduce a large positive penalty for reaching the dead state, denoted by~$d>0$. Then, we modify~\eqref{eq:LP2} as follows:
denote by~$z_{s,u}$ the long-run frequency  with which the system is at state~$s$ and the action~$u$ is taken,~$\forall s\in S_{\texttt{r}}'$ and~$\forall u\in U_{\texttt{r}}(s)$. We want to minimize the mean total cost of reaching~$I_{\texttt{in}}$ from~$I_{\texttt{out}}$, while minimizing the probability of leaving~$S_{\texttt{s}}''$. In particular, we consider the following optimization: 
\begin{subequations}\label{eq:LP3}
\begin{align}
&\underset{\{z_{s,u}\}}{\boldsymbol{\min}} \bigg{[} \text{C}_{\texttt{sufx}}(S_c',d) \triangleq \sum_{(\check{s},u)}\Big{(}\sum_{s\in S_{\texttt{r}}''}  \eta(\check{s},u,s)\, c_{\texttt{r}}(\check{s},u)\nonumber\\
&\quad \qquad \qquad \qquad \qquad \qquad \qquad + \eta(\check{s},u,s_{bad})\, d \Big{)}\bigg{]} \label{eq:LP3_objective}\\
& \textrm{s.t.} \;  \sum_{u\in U_{\texttt{r}}(s)} z_{s,u}  = \sum_{(\check{s},u)} \eta(\check{s},u,s)+y_0(s), \; \forall s\in S_{\texttt{r}}';  \label{eq:LP3_constraint1}\\
&\sum_{(\check{s},u)} \; \bigg{(}\sum_{s\in I_{\texttt{in}}} \eta(\check{s},u,s)+ \eta(\check{s},u,s_{bad})  \bigg{)}= \sum_{s\in S_{\texttt{r}}'} y_0(s);  \label{eq:LP3_constraint2}\\
&\; \qquad z_{s,u} \geq 0, \; \forall s\in S_{\texttt{r}}',\; \forall u \in U_{\texttt{r}}(s);\label{eq:LP3_constraint3}
\end{align}
\end{subequations}
where the notation {$\sum_{(\check{s},u)}\triangleq \sum_{\check{s}\in S_{\texttt{r}}'} \sum_{u\in U_{\texttt{r}}(s)} $}, the variables satisfy that $\eta(\check{s},u,s)\triangleq z_{\check{s},u}\, p_{\texttt{r}}(\check{s},u,s)$,  $\eta(\check{s},u,s_{bad})\triangleq z_{\check{s},u}\, p_{\texttt{r}}(\check{s},u,s_{bad})$,~$\text{C}_{\texttt{sufx}}(S_c',d)$ denotes the objective function as the summation of the mean cost of reaching~$I_{\texttt{in}}$ and the expected penalty of reaching~$s_{bad}$. 
The first constraint balances the incoming and outgoing flow at each state, while the second constraint ensures that~$I_{\texttt{in}}\cup \{s_{bad}\}$ are eventually reached. 
Let the optimal solution of~\eqref{eq:LP3} be~$z^\star_{\texttt{sufx}}=\{z_{s,u}^\star,\, s\in S_{\texttt{r}}',  u\in U_{\texttt{r}}(s)\}$.
{Then, the optimal  \emph{stationary} policy for the relaxed plan suffix, denoted by~$\boldsymbol{\pi}^\star_{\texttt{sufx}}$, can be derived as follows}:
for states in~$S_{\texttt{r}}'$, the optimal policy is given by~$\pi_{\texttt{sufx}}^\star(s,u)=z^\star_{s,u}/(\sum_{u\in U_{\texttt{r}}(s)} z^\star_{s,u})$ if~$ \sum_{u\in U_{\texttt{r}}(s)} z^\star_{s,u} \neq 0$; otherwise the action at~$s$ is chosen randomly, $\forall s\in S_{\texttt{r}}'$. Note that~$\pi_{\texttt{sufx}}^\star(s_f,u)=\pi_{\texttt{sufx}}^\star(s_f^{\texttt{out}},u)$, $\forall s_f \in I_c'$ and~$\forall u\in U(s_f)$.

\begin{lemma}\label{lem:reach-prob}
Under the relaxed plan suffix~$\boldsymbol{\pi}^\star_{\textup{\texttt{sufx}}}$, the probability of~$\mathcal{Z}_{\texttt{sufx}}$ reaching~$I_{\texttt{in}}$ from~$I_{\texttt{out}}$ while staying within~$S_{\texttt{r}}''$ over an infinite horizon, is lower bounded by~$1-\gamma_{\texttt{sufx}}(d)$, where~${\gamma_{\texttt{sufx}}(d)}\triangleq \sum_{\check{s}\in S_{\texttt{r}}'}\sum_{u\in U_{\texttt{r}}(\check{s})}z^\star_{\texttt{sufx}}(\check{s},u) \,p_{\texttt{r}}(\check{s},u,s_{bad})$.
\end{lemma}
\begin{proof}
The proof is a simple inference from~\eqref{eq:LP3_constraint2}. 
\end{proof}

\begin{remark}\label{remark:different_suffix}
A lower bound can be enforced  on~$\gamma_{\texttt{sufx}}$ as in~\eqref{eq:LP1}.
However, this bound is hard to estimate and a large bound can yield the problem infeasible.  In contrast,~\eqref{eq:LP3} always has a solution and~$\gamma_{\texttt{sufx}}(d)$ is tunable by varying~$d$.   \hfill $\blacksquare$
\end{remark}

\subsection{The Complete Policy}\label{sec:complete-policy}
In this section, we present how to combine the \emph{stationary} plan prefix and plan suffix of~$\mathcal{P}$ into the complete \emph{finite-memory} policy of the original MDP~$\mathcal{M}$. Furthermore, we show how to execute this finite-memory policy online.

\subsubsection{Combining the Plan Prefix and Suffix}\label{sec:two-cases}

When AMECs of~$\mathcal{P}$ exist, we can combine the plan prefix synthesis  and the plan suffix synthesis for each AMEC into one Linear Program:
\begin{align}\label{eq:comb-amec}
&\hspace{-0.1in}\min_{\{y_{s,u},\boldsymbol{z}_{s,u}\}}\; \beta\cdot\textup{\textbf{C}}_{\texttt{pre}}(S_c)+(1-\beta)\sum_{(S_c',U_c')\in \Xi_{acc}}\textup{\textbf{C}}_{\texttt{suf}}(S_c',U_c'), \\
& \textrm{s.t.} \qquad\qquad \textrm{Constraints \eqref{eq:LP1_constraint1}--\eqref{eq:LP1_constraint3} and \eqref{eq:LP2_constraint1}--\eqref{eq:LP2_constraint3}}\nonumber
\end{align}
{where $\textup{\textbf{C}}_{\texttt{pre}}(S_c)$ and $\textup{\textbf{C}}_{\texttt{suf}}(S_c',U_c')$ are defined in \eqref{eq:LP1_objective} and \eqref{eq:LP2_objective}, respectively, the variables $\{y_{s,u}\}$ satisfy the constraints \eqref{eq:LP1_constraint1}--\eqref{eq:LP1_constraint3} and \eqref{eq:LP2_constraint1}, and the variables $\boldsymbol{z}_{s,u}\triangleq \{z_{s,u}(S_c'),\forall (S_c',U_c')\in \Xi_{acc}\}$, where $z_{s,u}(S_c')$, satisfy the constraints \eqref{eq:LP2_constraint1}--\eqref{eq:LP2_constraint3} for the AMEC~$(S_c',U_c')\in \Xi_{acc}$. The parameter $0\leq \beta\leq 1$ captures the importance of minimizing the expected total cost to reach~$S_c$ versus stay in $S_c$. Note that the initial conditions $y_0$ in~\eqref{eq:LP2_constraint1} for each state in the suffix are expressed over the variables $\{y_{s,u}\}$. In other words, the initial conditions of each AMEC are now optimized to solve the combined objective function~\eqref{eq:comb-amec}.
It can be solved via any Linear Programming solver, e.g.,~``Gurobi''~\cite{gurobi} and~``CPLEX''. Once the optimal solution $\{y^\star_{s,u}\}$ and $\boldsymbol{z}^\star_{s,u}$ is obtained, the optimal plan prefix~$\boldsymbol{\pi}^\star_{\texttt{pre}}$ can be constructed as described in Section~\ref{sec:plan-prefix} and the plan suffix~$\boldsymbol{\pi}^\star_{\texttt{suf}}$ as in Section~\ref{sec:plan-suffix}.}

On the other hand, when no AECs of~$\mathcal{P}$ exist, as discussed in Section~\ref{sec:non-AEC}, we can combine the relaxed plan prefix and suffix synthesis for each ASCC into one  Linear Program:
\begin{align}\label{eq:comb-ascc}
&\min_{\{y_{s,u},\boldsymbol{z}_{s,u}\}}\; \beta\cdot \textup{\textbf{C}}_{\texttt{prex}}(S_c)+(1-\beta)\sum_{S_c'\in \Omega_{acc}} \textup{\textbf{C}}_{\texttt{sufx}}(S_c',d), \\
& \textrm{s.t.} \qquad\qquad \textrm{Constraints \eqref{eq:LP1_constraint1}--\eqref{eq:LP1_constraint3} and \eqref{eq:LP3_constraint1}--\eqref{eq:LP3_constraint3}}\nonumber
\end{align}
where $\textup{\textbf{C}}_{\texttt{prex}}(S_c)$ and $\textup{\textbf{C}}_{\texttt{sufx}}(S_c',d)$ are defined in \eqref{eq:LP1_objective} and \eqref{eq:LP3_objective}, respectively, the variables $\{y_{s,u}\}$ satisfy the constraints \eqref{eq:LP1_constraint1}--\eqref{eq:LP1_constraint3},  and the variables $\boldsymbol{z}_{s,u}\triangleq \{z_{s,u}(S_c'),\forall S_c'\in \Omega_{acc}\}$, where $z_{s,u}(S_c')$, satisfy the constraints \eqref{eq:LP3_constraint1}--\eqref{eq:LP3_constraint3} for the ASCC~$S_c'\in \Omega_{acc}$. The parameter $0\leq \beta\leq 1$ captures the importance of minimizing the expected total cost to reach~$S_c$ versus stay in $S_c$.
Similar to the previous case, the initial conditions $y_0$ in~\eqref{eq:LP3_constraint1} for each state in the ASCCs are expressed over the variables $\{y_{s,u}\}$. Thus the initial conditions are now optimized to solve the combined objective function~\eqref{eq:comb-ascc}.
Again, it can be solved via any Linear Programming solver. Once the optimal $\{y^\star_{s,u}\}$ and $\boldsymbol{z}^\star_{s,u}$ is obtained, the optimal relaxed plan prefix~$\boldsymbol{\pi}^\star_{\texttt{prex}}$ and relaxed plan suffix~$\boldsymbol{\pi}^\star_{\texttt{sufx}}$ can be constructed as described in Section~\ref{sec:non-AEC}. 

Note that the size of both Linear Programs in~\eqref{eq:comb-amec} and~\eqref{eq:comb-ascc} is linear with respect to the number of transitions in~$\mathcal{P}$ and can be solved in polynomial time~\cite{dantzig2016linear}. Note also that the multi-objective costs introduced in~\eqref{eq:comb-amec} and \eqref{eq:comb-ascc} provide a balance between optimizing the plan prefix and suffix. Compared to only optimizing the plan suffix, i.e., for $\beta=0$ as required to solve Problems~\ref{prob:suffix} and \ref{prob:suffix_rex},  increasing slightly the value of $\beta$ can lead to a significant decrease in the total cost of the plan prefix, without sacrificing much the optimality in the plan suffix. 

Observe that the optimal policy derived above only includes the states within~$S_n\cup S_c$. Thus no policy is specified for the \emph{bad states} in~$S_d$. Once the system reaches any bad state, it has violated the formula~$\varphi$ and can not satisfy it anymore. 
Thus, it is common practice to stop the system once that happens~\cite{ding2011mdp, baier2008principles}. 
We propose here a new method that allows the system to \emph{recover} from the bad state in~$S_d$ and continue performing the task, which could be useful for partially-feasible tasks with soft constraints, as discussed in~\cite{guo2015multi}.

\begin{definition}\label{def:projection}
The \emph{projected distance} of a bad state~$s_d=\langle x,l,q\rangle \in S_d$  onto~$S_c \cup S_n$ via~$u\in U(s_d)$ is defined as:
\begin{equation}\label{eq:projected}
\kappa(s_d,\,u) \triangleq \sum_{\check{s}\in S_c\cup S_n}\frac{\texttt{D}(l,\, \chi(q,\, \check{q}))}{|\chi(q,\check{q})|}\cdot p_E(x,u,\check{x})\cdot p_L(\check{x},\check{l}),
\end{equation}
where~$\check{s}\triangleq \langle \check{x},\check{l},\check{q}\rangle$ and  function~$\texttt{D}: 2^{AP}\times 2^{2^{AP}}\rightarrow \mathbb{N}$ returns the distance between an element~$l\in 2^{AP}$ and a set~$\chi \subseteq 2^{AP}$, was firstly introduced in~\cite{guo2015multi} and restated below.  \hfill $\blacksquare$
\end{definition}

\begin{algorithm}[t]
\caption{Complete Policy Synthesis} \label{alg:complete}
\DontPrintSemicolon
\KwIn{$\mathcal{P}$ by Definition~\ref{def:product}, $\gamma$, $\beta$}
\KwOut{the complete policy~$\boldsymbol{\pi}^\star$, $\boldsymbol{\mu}^\star$}
\eIf{~$\Xi_{acc}\neq \emptyset$}{
{1. Construct~$\mathcal{Z}_{\texttt{pre}}$, and~$\mathcal{Z}_{\texttt{suf}}$ for each~$(S'_c,\, U'_c)\in \Xi_{acc}$.}\\
{2. Derive~$\boldsymbol{\pi}^\star$ via solving~\eqref{eq:comb-amec}, and~\eqref{eq:policy-bad}.}\\
}{
{1. Construct~$\mathcal{Z}_{\texttt{prex}}$, and~$\mathcal{Z}_{\texttt{sufx}}$ for each~$S'_c \in \Omega_{acc}$.}\\
{2. Derive~$\boldsymbol{\pi}^\star$ via solving~\eqref{eq:comb-ascc}, and~\eqref{eq:policy-bad}.}\\
}
3. Construct~$\boldsymbol{\mu}^\star$ from~$\boldsymbol{\pi}^\star$ by~\eqref{eq:transform}
\end{algorithm}

Simply speaking,~$\kappa(s_d,\,u)$  evaluates how much the product automaton~$\mathcal{P}$ is violated on the average if 
the bad state~$s_d\in S_d$ is projected into the set of good states~$S_c\cup S_n$ using action~$u\in U(s_d)$. 
{Function $\texttt{D}(\ell, \, \chi)=0$ if $\ell \in \chi$ and $\texttt{D}(\ell, \, \chi)=\min_{\ell'\in \chi}\; |\{a\in AP\,|\,a\in \ell, a\notin \ell'\}|$, otherwise. Namely, it returns the minimal difference between~$\ell$ and any element in~$\chi$.}
Given~$\kappa(\cdot)$, the policy at~$s_d\in S_d$ is given by 
\begin{equation}\label{eq:policy-bad}
\boldsymbol{\pi}^\star(s_d, u) = \begin{cases}
1 & \text{for}~u = \text{argmin}_{u\in U(s_d)} \kappa(s_d,\,u); \\
0 & \text{other}\;u\in U(s_d),
\end{cases}
\end{equation}
which chooses the single action that minimizes~\eqref{eq:projected}. 
Combing~\eqref{eq:comb-amec},~\eqref{eq:comb-ascc} and~\eqref{eq:policy-bad} provides the complete policy for~$\mathcal{P}$. The above discussions are summarized in Algorithm~\ref{alg:complete}.

\subsubsection{Mapping~${\pi}^\star$ to~${\mu}^\star$}\label{sec:map-policy}
Lastly, 
we need to map the optimal {stationary} policy~$\boldsymbol{\pi}^\star$ of~$\mathcal{P}$ above to the optimal {finite-memory} policy~$\boldsymbol{\mu}^\star$ of~$\mathcal{M}$.
Starting from stage~$t=0$, the initial state~$s_0=\langle x_0,l_0,q_0\rangle\in S_n$ and the optimal action to take is given by the distribution~$\boldsymbol{\pi}^\star(s_0)$. Assume that~$u\in U(s_0)$ is taken. Then at stage~$t=1$, the robot observes its resulting state~$x_1$  and the label~$l_1$. Thus the subsequent state in~$\mathcal{P}$ is~$s_1=\langle x_1,l_1,q_1\rangle$, where~$q_1=\delta(q_0,l_0)$ is unique as~$\mathcal{A}_{\varphi}$ is deterministic. The optimal action to take now is given by the distribution~$\boldsymbol{\pi}^\star(s_1)$. 
This process repeats itself indefinitely. Denote by~$s_t\in S$ the \emph{reachable} state at stage~$t\geq 0$ which is always unique given the robot's past sequence of states~$X_{t}=x_0x_1\cdots x_t$ and labels~$L_t=l_0l_1\cdots l_t$.
Thus the optimal policy~$\boldsymbol{\mu}^\star$ at stage~$t\geq 0$ given~$X_t$ and~$L_t$ is 
\begin{equation}\label{eq:transform}
\boldsymbol{\mu}^\star(X_t, L_t) = \boldsymbol{\pi}^\star(s_t),
\end{equation}
i.e., the control policy at the reachable state~$s_t$ in~$\mathcal{P}$ is the best control policy in~$\mathcal{M}$ at stage~$t$,~$\forall t\geq 0$. 
Last but not least, if the system reaches a bad state at stage~$t-1$, i.e.,~$s_{t-1}\in S_d$, according to policy~\eqref{eq:policy-bad} the robot will take action~$u^\star$ and more importantly the next reachable state is \emph{set to be}~$s_t\triangleq \langle x_t,l_t,q'_t\rangle \in (S_c \cup S_n)$, where~$x_t$, $l_t$ are the observed robot location and label at stage~$t$ and~$q'_t\triangleq \text{argmin}_{\check{q}\in Post(q_{t-1})} \texttt{D}(l_{t-1},\chi(q_{t-1},\check{q}))$.

\begin{theorem}\label{thm:whole}
{Algorithm~\ref{alg:complete} solves Problem~\ref{prob:main} if~AECs of~$\mathcal{P}$ exist and $\beta=0$. }
Otherwise, if no~AECs of~$\mathcal{P}$ exist, then Problem~\ref{prob:main} has no solution. {In this case, Algorithm~\ref{alg:complete} provides a relaxed policy that minimizes the \emph{relaxed} suffix cost~$\text{C}_{\texttt{sufx}}(S_c', d)$ defined in~\eqref{eq:LP3}}.
Moreover, given any finite run~$S_T = s_0s_1\cdots s_T$ of~$\mathcal{P}$ under the optimal policy~$\boldsymbol{\pi}^\star$,  the probability that~$S_T$ does not intersect with the set of bad states~$S_d$ for all time~$t\in [0,\, T]$ is bounded as
$$Pr(s_t\notin S_d, \forall t\in[0, T]) \geq (1-\gamma_{\texttt{prex}})\cdot (1-\gamma_{\texttt{sufx}}(d))^{N_s},$$
where~$N_s\geq 0$ is the number of accepting cyclic paths contained in~$S_T$ that depends on $T$. 
\end{theorem}
\begin{proof}
To show the \emph{first} part of this theorem, similar to Lemma~\ref{lem:risk}, the constraints of \eqref{eq:LP1_constraint1}--\eqref{eq:LP1_constraint3} ensures that the total probability of reaching the union of all AMECs is lower-bounded by~$1-\gamma$.
{Moreover, the first part of Lemma~\ref{lem:rabin-suffix} shows that any infinite run~$\tau_{\mathcal{P}}$ of~$\mathcal{P}$ would satisfy~$\varphi$ once~it enters any AMEC~$(S'_c,U'_c)\in \Xi_{acc}$, by following the plan suffix.}
The fact that~$\boldsymbol{\pi}^\star$ also minimizes the mean total cost in~\eqref{eq:objective} when~$\beta=0$ in~\eqref{eq:comb-amec} can be shown as follows: as discussed in~\cite{chatterjee2011two, chatterjee2011energy, brazdil2015multigain}, the mean payoff objective depends on how the system suffix behaves within the AMECs. 
The second part of Lemma~\ref{lem:rabin-suffix} guarantees that the derived plan suffix~$\boldsymbol{\pi}_{\texttt{suf}}^\star$ minimizes the mean total cost of staying within any of the AMECs, while satisfying the accepting condition.

To show the \emph{second} part of the theorem, no solution to Problem~\ref{prob:main} exists regardless of the choice of~$\gamma$, {as the probability of satisfying the task is zero. Instead, when $\beta=0$, the optimal policy~$\boldsymbol{\pi}^\star$ obtained by Algorithm~\ref{alg:complete} minimizes the \emph{relaxed} suffix cost~$\text{C}_{\texttt{sufx}}(S_c', d)$.} {At the same time, due to the constraints in \eqref{eq:LP1} that are also present in \eqref{eq:comb-amec}, the plan prefix~$\boldsymbol{\pi}^\star_{\texttt{prex}}$ ensures that all runs stay within~$S_n$ with at least probability~$(1-\gamma_{\texttt{prex}})$ before entering any ASCC~$S_c'\in \Omega_{acc}$, while the relaxed plan suffix~$\boldsymbol{\pi}_{\texttt{sufx}}^\star$ ensures that the runs stay within~$S_c'$ with at least probability~$(1-\gamma_{\texttt{sufx}}(d))$ for one execution of any accepting cyclic path. }
Consequently, if the finite run contains~$N_s$ accepting cyclic paths, the probability of avoiding~$S_d$, is lower bounded by~$(1-\gamma_{\texttt{prex}})\cdot (1-\gamma_{\texttt{sufx}}(d))^{N_s}$. 
Even though this probability approaches zero as~$N_s$ approaches infinity, this result still ensures that the frequency of visiting bad states over finite intervals is minimized.
\end{proof}

\begin{algorithm}[t]
\caption{Policy Execution} \label{alg:execution}
\DontPrintSemicolon
\KwIn{$\mathcal{M}$,~$\varphi$, observed state~$x_t$ and label~$l_t$ at stage~$t\geq 0$}
\KwOut{$\boldsymbol{\mu}^\star$ and~$u_t$ at stage~$t\geq 0$}
1. \textbf{Offline}: Construct~$\mathcal{P}$ and synthesize~$\boldsymbol{\pi}^\star$ by Alg.~\ref{alg:complete}. \;
2. {At~$t=0$}: set $s_0=\langle x_0,l_0,q_0\rangle$ and apply $u_0\sim \boldsymbol{\pi}^\star(s_0)$.\;
3. \While{$t=1,2,\cdots$} 
{observe~$x_t$ and~$l_t$.\; 
\eIf{$s_{t-1}\notin S_d$}
{Set~$s_t=\langle x_t,l_t,q_t\rangle$, where~$q_t=\delta(q_{t-1},l_{t-1})$.}{Set~$s_t=\langle x_t,l_t,q'_t\rangle\in (S_n \cup S_c)$.}
Apply action~$u_t\sim \boldsymbol{\pi}^\star(s_t)$.}
\end{algorithm}

\subsubsection{Policy Execution}\label{sec:policy-execution}
Clearly, {the optimal policy~$\boldsymbol{\mu}^\star$ from~\eqref{eq:transform} requires only a finite memory to save the current reachable state~$s_t$ and the optimal policy~$\boldsymbol{\pi}^\star$.} It is synthesized off-line once via Algorithm~\ref{alg:complete} and its online execution involves observing the current state~$x_t$ and label~$l_t$, updating the reachable state~$s_t$, and applying the action according to~$\boldsymbol{\pi}^\star(s_t)$.
Details are given in Algorithm~\ref{alg:execution}.


\section{Simulation Results}\label{sec:example}
In this section, we present  simulation results to validate the  scheme.
All algorithms are implemented in Python 2.7 and available online~\cite{git_mdp_tg}. 
All simulations are carried out on a laptop (3.06GHz Duo CPU and 8GB of RAM).

\subsection{Model Description}\label{sec:model-description}
We consider a partitioned~$10m\times 10m$ workspace as shown in Figure~\ref{fig:traj-supply}, where each cell is a $2m\times 2m$ area. {The properties of interest are~$\{\texttt{Obs},\texttt{b1}, \texttt{b2},\texttt{b3},\texttt{Spl}\}$}. The properties satisfied at each cell are probabilistic: three cells at the corners satisfy~$\texttt{b1}$, $\texttt{b2}$ and $\texttt{b3}$, respectively with probability one. Four cells at~$(1m,5m),(5m,3m),(9m,5m),(5m,9m)$ satisfy~$\texttt{Spl}$ with probabilities ranging from~$0.2$ to~$0.8$, modeling the likelihood that a supply appears at that particular cell. One cell at~$(5m,1m)$ satisfies~$\texttt{Obs}$ with probability~$0.7$. Other obstacles will be described later upon different task scenarios. 

The robot motion follows the unicycle model, i.e.,~$\dot{x} = v\cos(\theta)$, $\dot{y} = v\sin(\theta)$, $\dot{\theta} = \omega$, where~$p(t)=(x(t),\,y(t))\in \mathbb{R}^2$, $\theta(t)\in (-\mathbf{pi},\,\mathbf{pi}]$ are the robot's position and orientation at time~$t\geq 0$. The control input is~$u(t)=(v(t),\,\omega(t))$ and contains the linear and angular velocities. 
Due to actuation noise and drifting, the robot's motion is subject to uncertainty.
The action primitives and the associated uncertainties are shown in Figure~\ref{fig:action} and described below: 
action~``$\textsf{FR}$'' means driving forward for~$2m$ by setting~$v(t)= v_0$ and~$\omega(t)= 0$, $\forall t=[0,\,2/v_0]$. This action has probability~$0.8$ of reaching~$2m$ forward and probability~$0.1$ of drifting to the left or right by~$2m$, respectively;  
{action~``$\textsf{BK}$'' can be defined analogously to ``$\textsf{FR}$''}; 
action~``$\textsf{TR}$'' means turning right by an angle of~$\mathbf{pi}/2$ by setting~$v(t)=0$ and~$\omega(t)=-\omega_0$, $\forall t=[0,\,\mathbf{pi}/(2\omega_0)]$.
This action has  probability~$0.9$ of turning to the right by~$\mathbf{pi}/2$, probability~$0.05$ of turning less than~$\mathbf{pi}/4$ due to undershoot and probability~$0.05$ of turning more than~$3\mathbf{pi}/4$ due to overshoot; {action~``$\textsf{TL}$'' can be defined analogously to ``$\textsf{TR}$''};
lastly, action~``$\textsf{ST}$'' means staying still by setting~$v(t)=\omega(t)=0$,~$\forall t=[0,T_0]$ where~$T_0$ is the chosen waiting time. It has probability~$1.0$ of staying where it is.
The cost of each action is given by~$[2,4,3,3,1]$, respectively, where the cost of~``$\textsf{ST}$'' is set to~$1$ as it consumes time to wait at one cell.

\begin{table}[t]
\begin{center}
\scalebox{1.05}{
    \begin{tabular}{| c| c | c | c | c|}
    \hline
   $\gamma$ & \textbf{Total Cost}  &  \textbf{Failure} & \textbf{Success} & \textbf{Unfinished}  \\ \hline \hline
    0 & 132.2 & 0 & 910 & 90\\ \hline
    0.1 & 118.1 & 99 & 872 & 29\\ \hline
    0.2 & 110.5 & 219 & 770 & 11\\ \hline
    0.3 & 104.6 & 308 & 692 & 0\\ \hline
    0.4 & 98.3 & 417 & 583 & 0\\ 
    \hline
    \end{tabular}}
\caption{Statistics of~$1000$ Monte Carlo simulations of~$500$ time steps, under different~$\gamma$ for task~\eqref{eq:reach-task}.}
\label{table:reach-statistics}
\end{center}
\end{table}

With the above model, we can abstract the robot state by the cell coordinate in which it belongs, namely, $(x_c, y_c)\in \{1,3,\cdots,9\}^2$ and its four possible orientations~($N, E, S, W$). The transition relation and probability can be built following the description above. The resulting probabilistically-labeled MDP has~$100$ states and~$816$ edges.

\begin{table*}[t]
\begin{center}
\begin{tabular}{ccccccccccccc}
\toprule
\multicolumn{2}{c}{$\mathcal{M}$} & \multicolumn{2}{c}{$\mathcal{P}$} & 
\multicolumn{2}{c}{AMECs $\Xi_{acc}$} &  \multicolumn{3}{c}{$\boldsymbol{\pi}^{\star}$ via~\eqref{eq:comb-amec}}\\
\cmidrule(lr){1-2}\cmidrule(lr){3-4}\cmidrule(lr){5-6}\cmidrule(lr){7-9}
Size & {Time [\si{\second}]} & 
Size & {Time [\si{\second}]} &
Size & {Time [\si{\second}]} & 
Size of~\eqref{eq:LP1} & Size of~\eqref{eq:LP2} & {Time to solve~\eqref{eq:comb-amec} [\si{\second}]} \\
\midrule
(100,\;816)& 0.13 & (4.2e3,\;4.1e4) & 16.3 & 1.2e3 & 4.15 & (443,\;2.0e3,\;8.3e3) & (1.2e3,\;4.9e3,\;2.1e4) & 0.21\\
\addlinespace[0.3em]
(324,\;2.8e3)& 1.69 & (1.1e4,\;1.0e5) & 41.2 & 3.6e3 & 29.4 & (1.3e3,\;6.3e3,\;2.2e4) & (3.6e3,\;1.7e4,\;5.9e4) & 0.72\\
\addlinespace[0.3em]
(900,\;8.4e3)& 24.2 & (2.9e4,\;2.8e5) & 106.8 & 1.0e4 & 337.1 & (3.6e3,\;1.7e4,\;6.0e4) &  (9.9e3,\;4.8e4,\;1.6e5) & 16.74\\
\addlinespace[0.3em]
(1.4e3,\;1.3e4)& 88.7 & (4.7e4,\;4.5e5) & 391.7 & 1.6e4 & 1.1e3 & (5.8e3,\;2.8e4,\;9.7e4) & (1.5e4,\;7.7e4,\;2.6e5) & 20.81\\
\addlinespace[0.3em]
(2.5e3,\;2.4e4)& 326.9 & (8.1e4,\;7.8e5) & 290.1 & 2.7e4 & 4.8e3 & (1.0e4,\;4.9e4,\;1.6e5) & (2.7e4,\;1.3e5,\;4.5e5) & 15.74\\
\addlinespace[0.3em]
(3.3e3,\;3.2e4)& 558.3 & (1.0e5,\;1.1e6) & 380.1 & 3.7e4 & 9.4e3 & (1.3e4,\;6.6e4,\;2.2e5) & (3.7e4,\;1.8e5,\;6.1e5) & 32.04\\
\bottomrule
\end{tabular}
\caption{Size and computation time of various models~$\mathcal{M}$ as described in Section~\ref{sec:surv-rex} under task~\eqref{eq:deliver-task}. The notation~$\texttt{a}\text{e}\texttt{b}\triangleq \texttt{a}\times 10^{\texttt{b}}$ for~$\texttt{a},\texttt{b}>0$. The size of~$\mathcal{M}$, $\mathcal{A}_{\varphi}$ and~$\mathcal{P}$ includes the number of states and transitions. The size of LP problems~\eqref{eq:comb-amec} which contains~\eqref{eq:LP1} and~\eqref{eq:LP2} includes the number of rows, columns and variables in the linear equations, as indicated by the ``\texttt{Gurobi}'' solver~\cite{gurobi}. 
}
\label{table:sim-complexity}
\end{center}
\end{table*}

In the sequel, we consider \emph{three} different task formulas in the order of increasing complexity. 
We used~``\texttt{Gurobi}''~\cite{gurobi} to solve the Linear Programs in~\eqref{eq:comb-amec} and~\eqref{eq:comb-ascc}. {When comparing the performance in the plan suffix, we also use the total cost in~\eqref{eq:cycle-total-cost} as an indicator, especially when the difference in the mean total cost  in~\eqref{eq:mean-cycle} is too small to measure.}

\begin{figure}[t]
     \centering
     \includegraphics[width=0.4\textwidth]{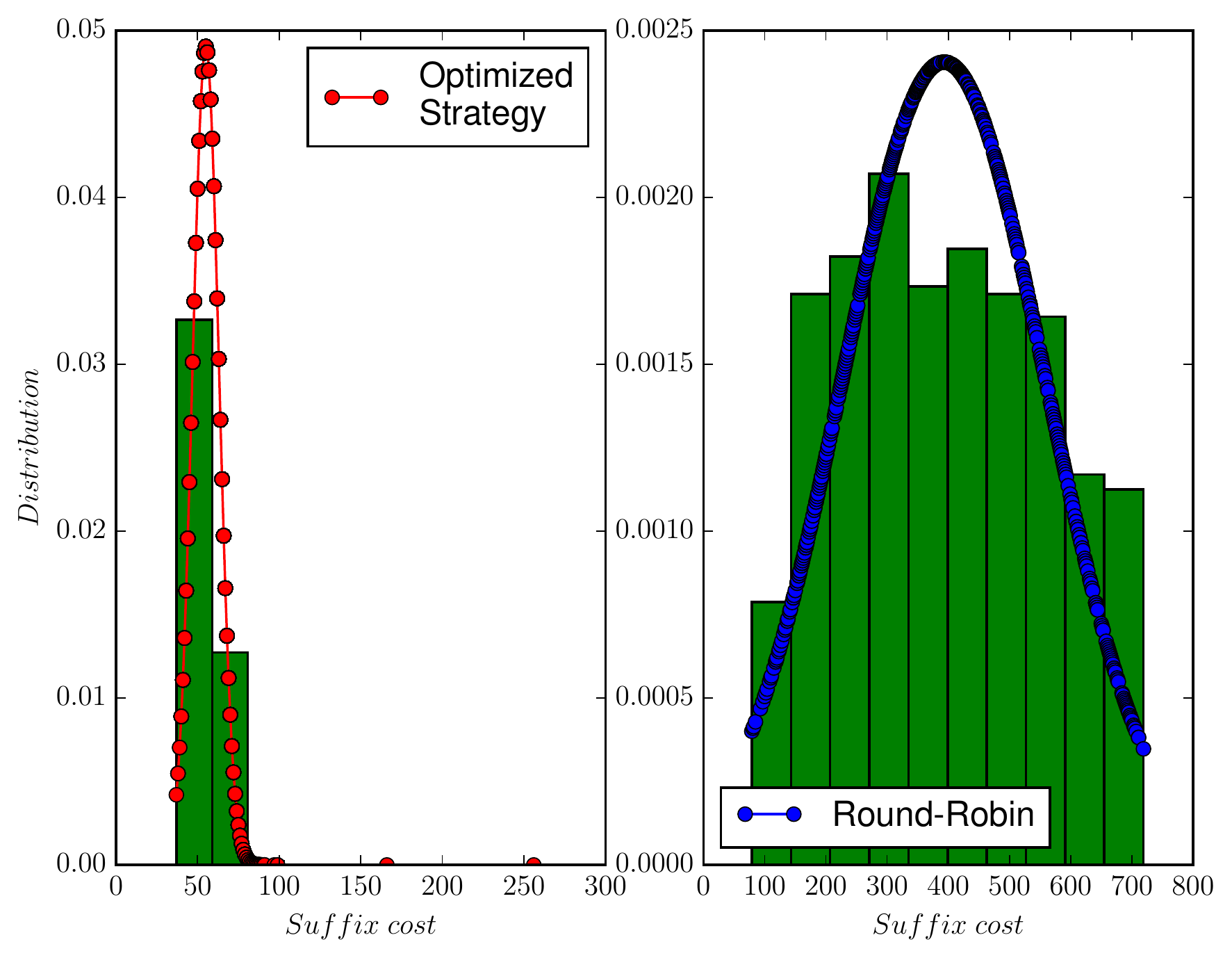}
     \caption{The normalized distribution of the total cost of accepting cyclic paths from $1000$ Monte Carlo simulations under the optimal plan suffix (left) and the Round-Robin policy (right), for task~\eqref{eq:surveillance-task}.}
     \label{fig:compareSurv}
   \end{figure}

\subsection{Ordered Reachability}\label{sec:order-reach}
In this case, we show the trade-off between reducing the expected total cost and decreasing the risk factor in the plan prefix synthesis using~\eqref{eq:LP1}.
In particular, the robot needs to reach~$\texttt{b1}$,~$\texttt{b2}$,~$\texttt{b3}$ (in this order) from the initial cell while avoiding obstacles for all time. Afterwards it should stay at~$\texttt{b3}$. The LTL formula for this task is
\begin{equation}\label{eq:reach-task}
\varphi_1 = (\Diamond (\texttt{b1} \wedge \Diamond (\texttt{b2} \wedge \Diamond \texttt{b3}))) \wedge (\square \neg \texttt{Obs}) \wedge (\Diamond \square \texttt{b3}).
\end{equation}
The associated DRA derived using~\cite{klein2007ltl2dstar} has~$7$ states,~$24$ transitions and~$1$ accepting pair. 
An additional obstacle is added which has probability~$0.7$ of appearing in the cell~$(5m,9m)$.

It took~$10.9s$ to construct the product automaton which has~$840$ states,~$7280$ transitions.
Since one AMEC exists, we synthesize the optimal policy using Algorithm~\ref{alg:complete} via solving~\eqref{eq:comb-amec} under $\beta=0.5$ and different risk factors $\gamma$ chosen from~$\{0,0.1,\cdots,0.4\}$, which took on average~$0.1s$. 
Then, we perform~$1000$ Monte Carlo simulations of $500$ time steps each, where we  evaluate the total cost in~\eqref{eq:problem_lp1} and whether the task is satisfied.
As shown in Table~\ref{table:reach-statistics}, the total cost increases when the allowed risk factor~$\gamma$ is decreased. 
The percentage of simulated runs that collide with an obstacle is approximately~$(1-\gamma)$, which verifies the risk constraint in Lemma~\ref{lem:risk}.

\subsection{Surveillance}\label{sec:surveil}
In this case, we compare the efficiency of the optimal plan suffix from Algorithm~\ref{alg:complete} and the Round-Robin policy. Particularly, the robot should visit~$\texttt{b1}$,~$\texttt{b2}$ and~$\texttt{b3}$ infinitely often for surveillance and avoid all obstacles:
\begin{equation}\label{eq:surveillance-task}
\varphi_2 = (\square \Diamond \texttt{b1}) \wedge (\square \Diamond \texttt{b2}) \wedge (\square \Diamond \texttt{b3}) \wedge (\square \neg \texttt{Obs}). 
\end{equation}
The associated DRA has~$8$ states,~$30$ transitions, and~$1$ accepting pair. It took~$5.8s$ to construct the product~$\mathcal{P}$ which has~$700$ states,~$5712$ transitions and~$1$ accepting pair.
{Since one AMEC exists in the product, we synthesize the optimal policy using Algorithm~\ref{alg:complete} via solving~\eqref{eq:comb-amec} under~$\gamma=0$ and $\beta=0.1$, which took~$0.2s$.}
We conducted~$1000$ Monte Carlo simulations and Figure~\ref{fig:compareSurv} shows that the total cost from~\eqref{eq:cycle-total-cost} of accepting cyclic paths in the plan suffix under the optimal policy is much lower than the Round-Robin policy ($50$ versus $400$).
Moreover, Figure~\ref{fig:supply-compare} shows that the average number of times each base station is visited by the robot under the optimal policy is much higher than under the Round-Robin policy.

\begin{figure}[t]
\begin{minipage}[t]{0.495\linewidth}
\centering
   \includegraphics[width =0.97\textwidth]{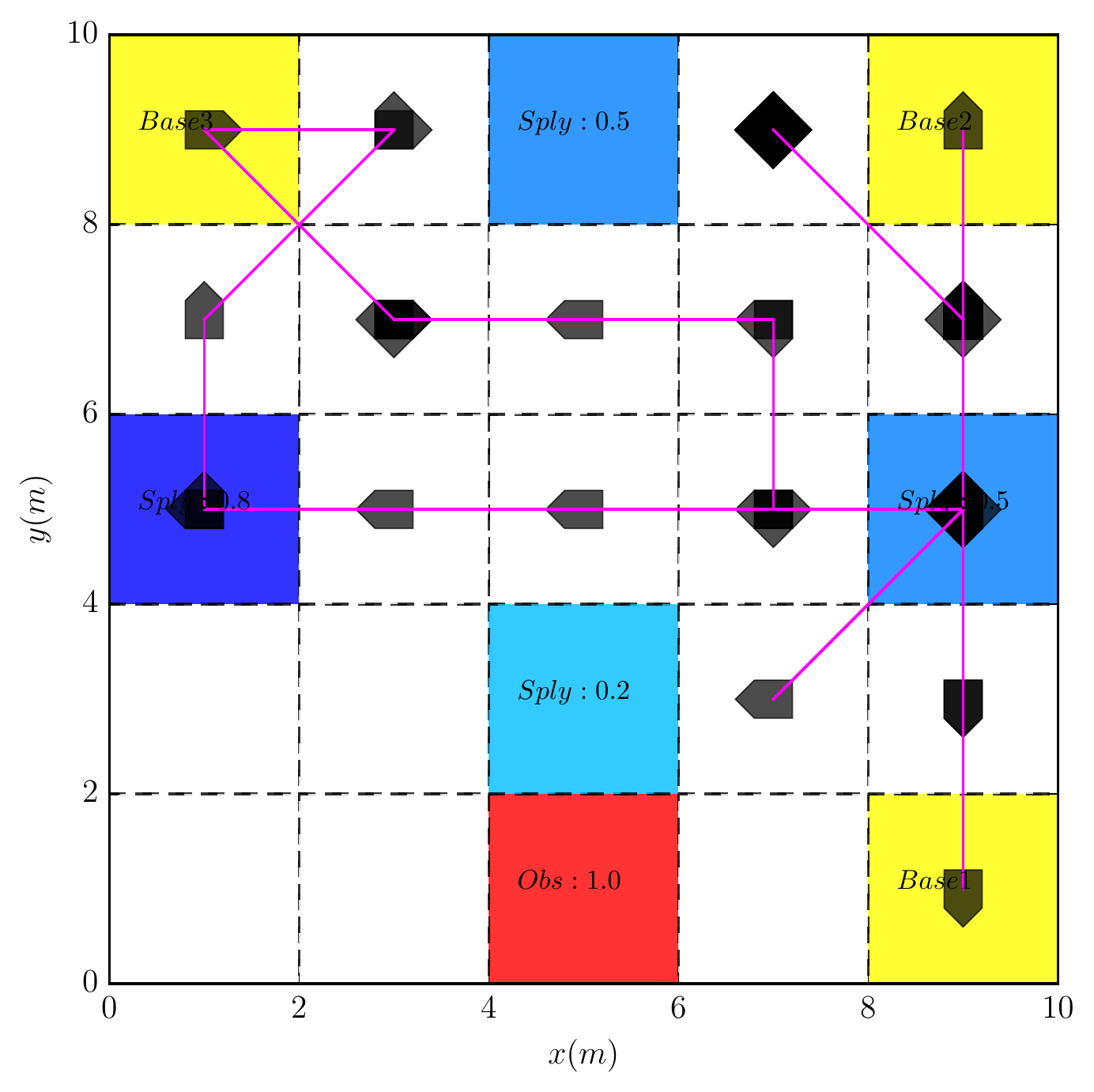}
  \end{minipage}
\begin{minipage}[t]{0.495\linewidth}
\centering
    \includegraphics[width =0.97\textwidth]{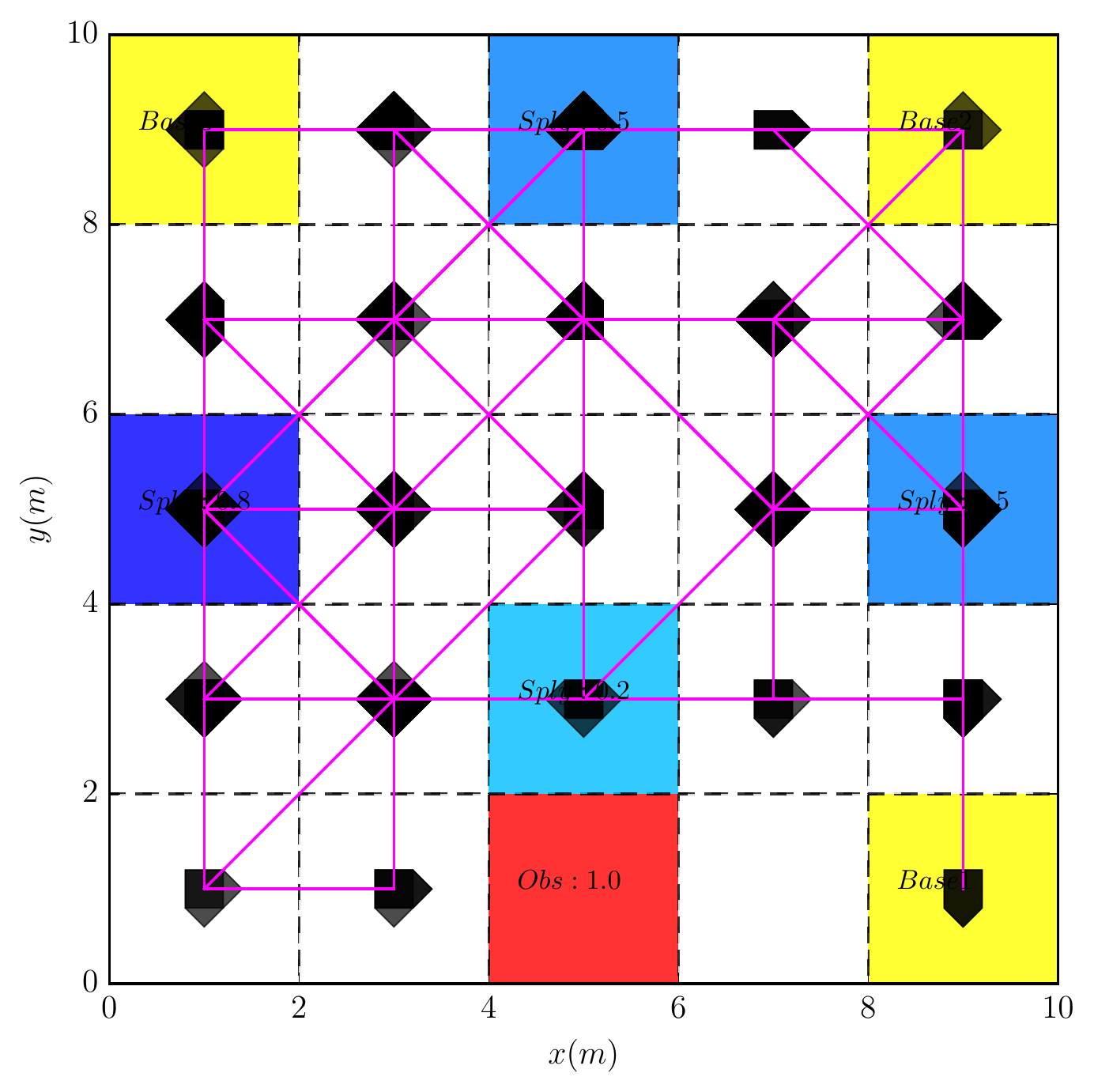}
  \end{minipage}
  \caption{Simulated trajectory suffix under the optimized plan suffix is applied (left) and the Round-Robin policy (right).}
\label{fig:traj-supply}
\end{figure}

\subsection{Ordered Supply-delivery}\label{sec:ordered-supply}
In this case, we demonstrate the reactiveness of the derived optimal policy. 
The robot needs to collect supplies from the cells that are marked by~$\texttt{Spl}$, where supplies appear probabilistically. 
Then it needs to transport these supplies to each base station.
Furthermore, the robot should \emph{not} visit two base stations consecutively without collecting a supply first. 
It should always avoid obstacles. The LTL task formula  is
\begin{equation}\label{eq:deliver-task}
\varphi = \varphi_{\texttt{all\_base}} \wedge \varphi_{\texttt{order}} \wedge (\square \neg \texttt{Obs}),
\end{equation}
where~$\varphi_{\texttt{all\_base}}=(\square\Diamond\texttt{b1})\wedge (\square\Diamond\texttt{b2}) \wedge (\square\Diamond\texttt{b3})$ means that all base stations should be visited infinitely often and~$\varphi_{\texttt{order}}=\square(\varphi_{\texttt{one}}\rightarrow \bigcirc((\neg \varphi_{\texttt{one}}) \, \mathsf{U} \, \texttt{Spl}))$, with~$\varphi_{\texttt{one}}=(\texttt{b1} \vee \texttt{b2} \vee \texttt{b3})$ means that when one base station is visited, then no base can be visited until a supply has been collected. 
The associated DRA is derived using~\cite{klein2007ltl2dstar, git_mdp_tg} in $0.05s$, which has~$32$ states,~$298$ transitions and~$1$ accepting pair.

It took around~$16s$ to construct the product automaton that has~$4224$ states,~$41344$ transitions and 1 accepting pair.
{Since two AMECs exist in the product, we synthesize the optimal policy using Algorithm~\ref{alg:complete} via solving~\eqref{eq:comb-amec} under~$\gamma=0$ and $\beta=0.1$, which took around~$0.2s$ given the complexity of task~\eqref{eq:deliver-task}. }
Notice that the optimal plan sometimes requires the robot to wait at a cell marked by~$\texttt{Spl}$ by taking action~``$\textsf{ST}$'', since the expected cost of traveling to another cell with supply  might be higher than waiting there for the supply to appear.
Figure~\ref{fig:traj-supply} compares the simulated trajectories under the optimal policy and the Round-robin policy.
Based on~$1000$ Monte Carlo simulations, the total cost of accepting cyclic paths is much lower under the optimal policy than the Round-Robin policy ($70$ versus $550$).
Furthermore, Figure~\ref{fig:supply-compare} shows the average number of supplies received at each base under these two policies. It can be seen that much more supplies are received at each base station under the optimal policy. Simulation videos of both cases can be found in~\cite{mdp_video}.  
Lastly, to show how the choice of~$\beta$ in~\eqref{eq:comb-amec} affects the optimal prefix and suffix cost, we repeat the above procedure for different $\beta$ and the results are summarized in Table~\ref{table:vary-beta}. 
In the table, the prefix cost equals to $\textup{\textbf{C}}_{\texttt{pre}}(S_c)$, the mean suffix cost equals to $\sum_{(S_c',U_c')\in \Xi_{acc}}\textup{\textbf{C}}_{\texttt{suf}}(S_c',U_c')$ from~\eqref{eq:comb-amec}. The total suffix cost is computed based on~\eqref{eq:cycle-total-cost} in order to magnify the changes in the suffix cost. 
It can be noticed that for small non-zero values of $\beta$, less $0.2$, the optimal prefix cost is reduced \emph{dramatically} (from $180.7$ to $62.4$), without increasing much the optimal suffix cost (from $66.1$ to $67.1$).

\begin{table}[t]
\begin{center}
\scalebox{1.1}{
    \begin{tabular}{| c| c | c| c| c|}
    \hline
   \multirow{2}{*}{$\beta$} & \multirow{2}{*}{\textbf{Prefix Cost}}  &  \multicolumn{2}{c|}{\textbf{Suffix Cost}} & \multirow{2}{*}{
\begin{tabular}{@{}c@{}}\textbf{Balanced Cost} \\ \textbf{by}~\eqref{eq:comb-amec}\end{tabular}}\\ 
 \cline{3-4}\\[-1em]
& & Total & Mean & \\ \hline 
    0 & 180.7 & 66.1 & 2.524 & 66.1 \\ \hline
    0.2 & 62.4  & 67.1 & 2.533 & 65.2\\ \hline
    0.4 & 50.5 & 72.9 & 2.551 & 64.1\\ \hline
    0.6 & 49.8 & 73.5 & 2.552 & 59.3\\ \hline
    0.8 & 49.5 & 74.3 & 2.554 & 54.4\\ \hline
    1.0 & 49.5 & 246.7 & 2.817 & 49.5\\ 
    \hline
    \end{tabular}}
\caption{{The optimal prefix cost, suffix cost and the balanced cost as defined in~\eqref{eq:comb-amec} of task~\eqref{eq:deliver-task} under different~$\beta$ with $\gamma=0$.}}
\label{table:vary-beta}
\end{center}
\end{table}

\begin{figure}[t]
\begin{minipage}[t]{0.495\linewidth}
\centering
   \includegraphics[width =0.97\textwidth]{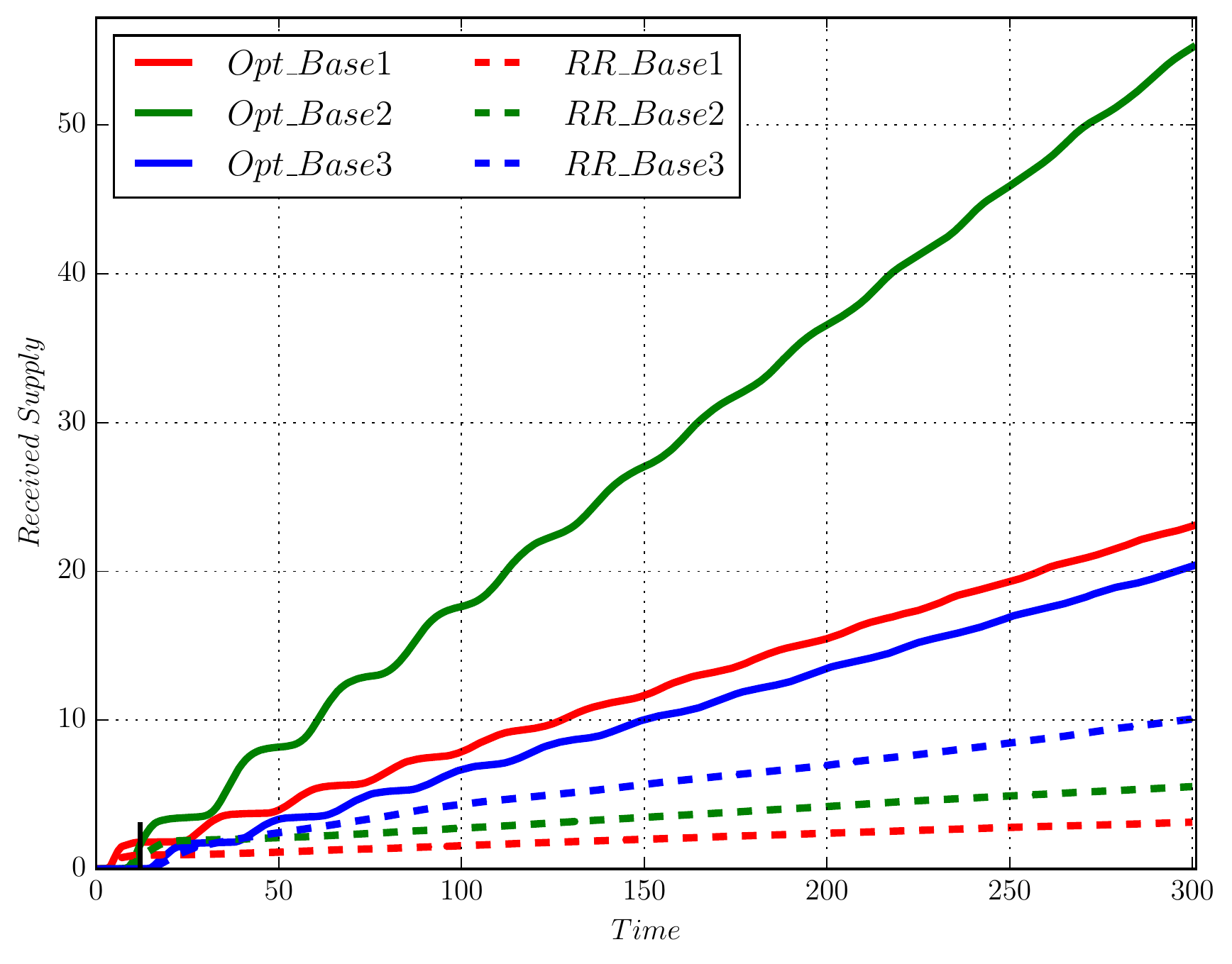}
  \end{minipage}
\begin{minipage}[t]{0.495\linewidth}
\centering
    \includegraphics[width =0.97\textwidth]{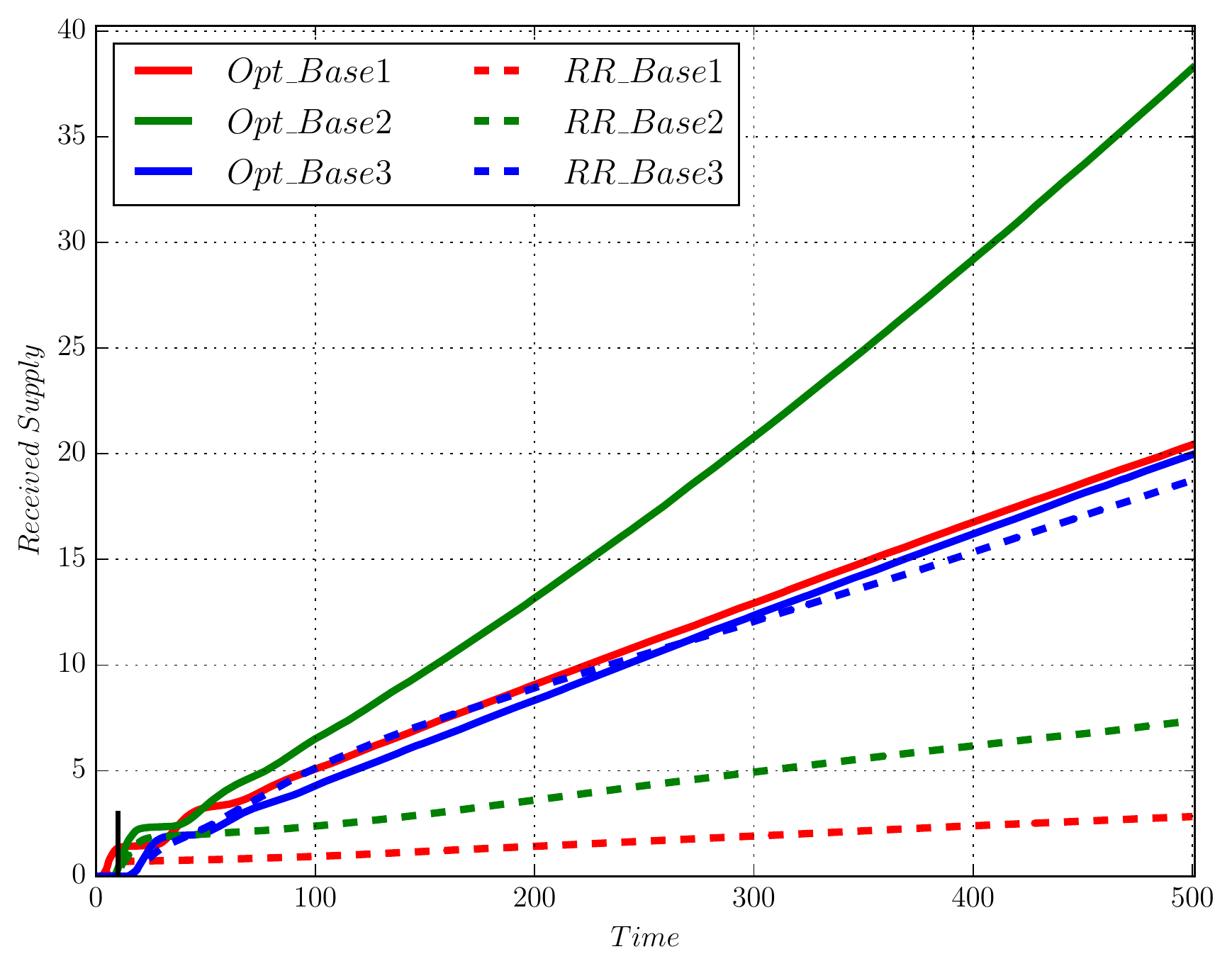}
  \end{minipage}
  \caption{Left: the average number of times each base is visited, for task~\eqref{eq:surveillance-task}.  Right: the average number of supplies received at each base for task~\eqref{eq:deliver-task}. The optimal policy  is shown in solid lines while the Round-Robin policy in dashed lines.}
\label{fig:supply-compare}
\end{figure}

In order to demonstrate scalability and computational complexity of the proposed algorithm, we repeat the policy  synthesis under the same task~\eqref{eq:deliver-task} but for  workspaces of various sizes. Particularly, we increase the number of cells from~$5^2$ to~$9^2$, $15^2$, $19^2$, $25^2$, $29^2$. The size of resulting~$\mathcal{M}$, $\mathcal{P}$, $\Xi_{acc}$ and the time taken to compute them are shown in Table~\ref{table:sim-complexity}, where we also list the complexity of the LP~\eqref{eq:comb-amec}, which consists of~\eqref{eq:LP1} and~\eqref{eq:LP2}, and the time taken to solve~\eqref{eq:comb-amec}. 
It can be seen from Table~\ref{table:sim-complexity} that solving~\eqref{eq:comb-amec} requires a small fraction of total time, compared to the construction of~$\mathcal{M}$,~$\mathcal{P}$ and~$\Xi_{acc}$. 

\subsection{Surveillance with Clustered Obstacles}\label{sec:surv-rex}
In this case, we demonstrate how the relaxed plan prefix and suffix can be  synthesized under scenarios where no AECs can be found.
In particular, we consider the surveillance task in~\eqref{eq:surveillance-task} but more obstacles are placed in the workspace as shown in Figure~\ref{fig:relaxed_1}. The center cell~$(5m,\,5m)$ has probability~$0.9$ of being occupied by an obstacle and the four cells above and on the left have probability~$0.01$ of being occupied by an obstacle. Thus,~$\texttt{b1}$ is surrounded by possible obstacles around it, even though the probability is very low.

The resulting product automaton has~$1184$ states,~$13888$ transitions, and $1$ accepting pair.
{It can be verified that \emph{no AECs exist} in~$\mathcal{P}$ and thus the second case of Algorithm~\ref{alg:complete} is activated, where the optimal solution is derived by solving~\eqref{eq:comb-ascc}.}
We synthesize the relaxed optimal policy  under different~$\gamma_{\texttt{prex}}$  and~$d$, as shown in Table~\ref{table:rex-statistics}.
{It took in average~$37s$ to synthesize the complete policy for $\beta=0.1$ and any chosen $\gamma_{\texttt{prex}}$ and~$d$ in this case.}
{Recall that $d$ is a large positive penalty for entering the set of bad states in~\eqref{eq:LP3}}.
In particular, we first choose~$\gamma_{\texttt{prex}}=0.1$ and~$d=300$. Two simulated trajectories under the derived policy are shown in Figure~\ref{fig:relaxed_1}.
Furthermore, we perform $1000$ Monte Carlo simulation under the~$\gamma_{\texttt{prex}}$ and~$d$ listed in Table~\ref{table:rex-statistics}, where we compare the number of times that the robot fails the task by colliding with obstacles (the failure),  the number of times that the robot successfully reaches the set of ASCC~$S_c$ (the prefix success), and the number of times that the robot successfully executes one accepting cyclic path associated with~$S_c'$ and $I_c'$ of one ASCC (the suffix success).
{It can be seen that~$(1-(1-\gamma_{\texttt{prex}})(1-\gamma_{\texttt{sufx}}))$,~$(1-\gamma_{\texttt{prex}})$ and~$(1-\gamma_{\texttt{sufx}})$ matches very well the probability of failure, the prefix success, and the suffix success, respectively, as discussed in Theorem~\ref{thm:whole}. 
Also, it can be seen that the system can recover from the bad states and continue executing the task if the recovery policy proposed in~\eqref{eq:policy-bad} is activated.}
It can also be seen that increasing~$\gamma_{\texttt{prex}}$ leads to a lower prefix success rate and decreasing~$d$ leads to a lower suffix success rate.

\begin{figure}[t]
\begin{minipage}[t]{0.495\linewidth}
\centering
   \includegraphics[width =1.02\textwidth]{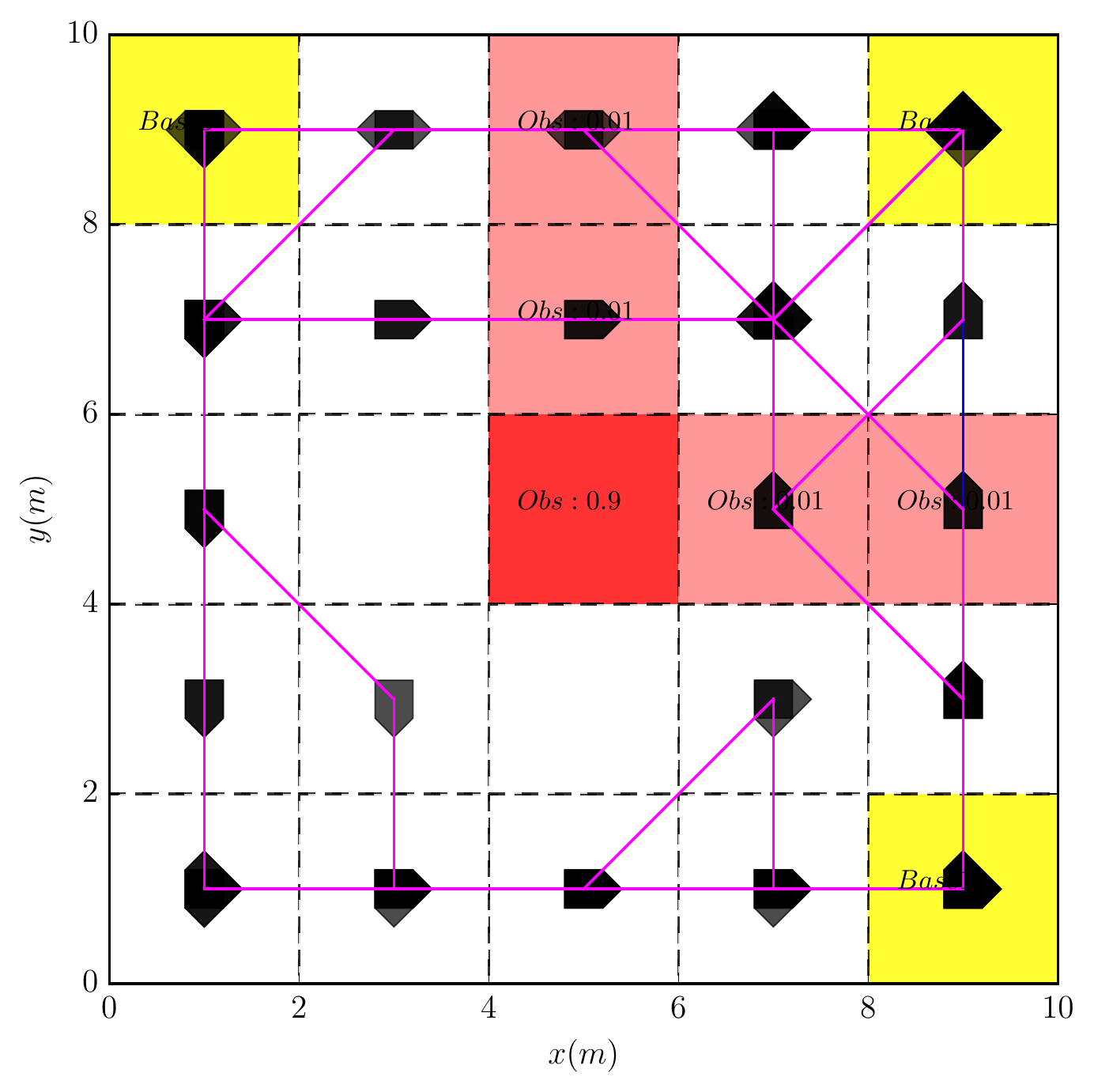}
  \end{minipage}
\begin{minipage}[t]{0.495\linewidth}
\centering
    \includegraphics[width =1.02\textwidth]{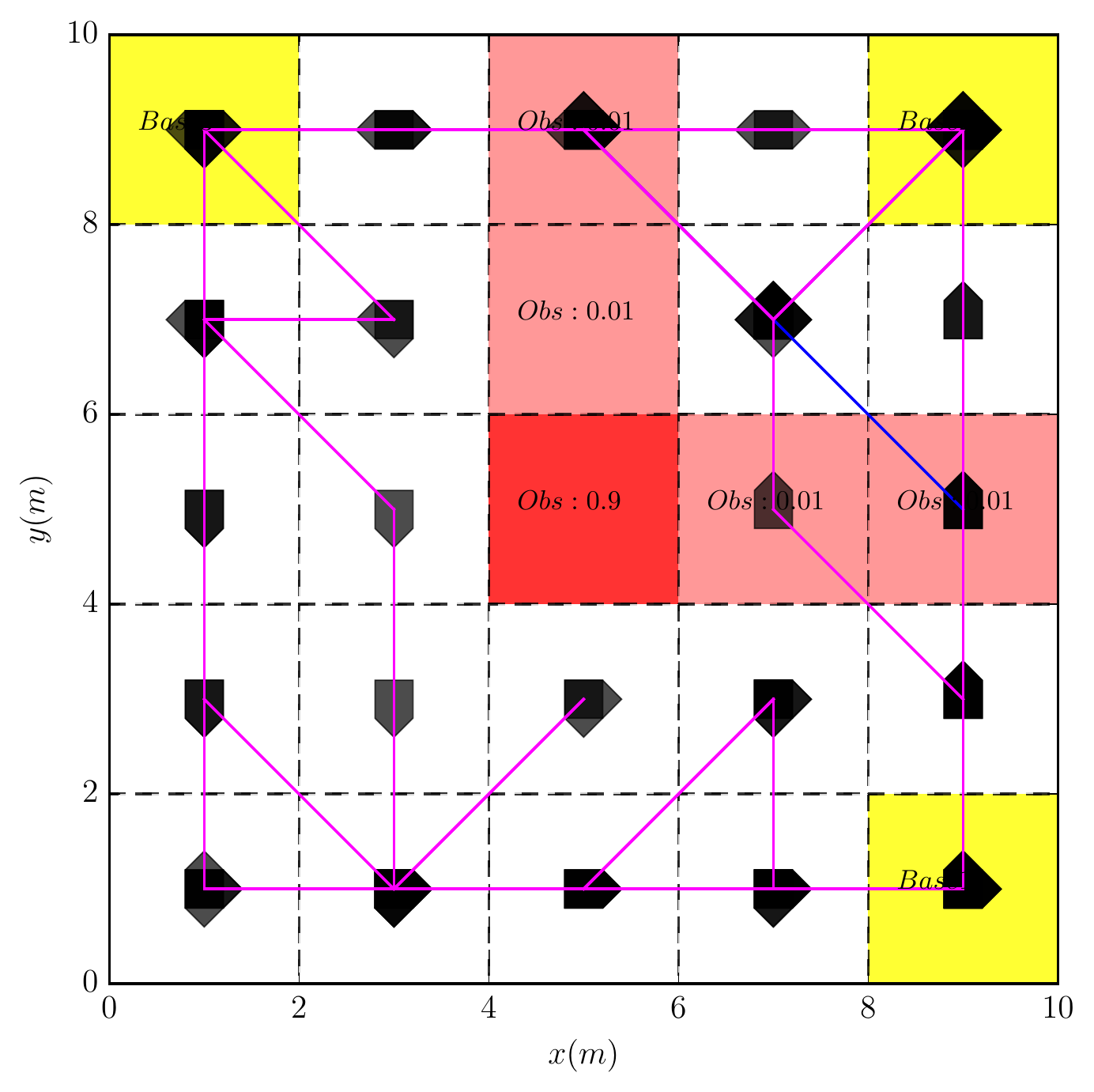}
  \end{minipage}
  \caption{Two simulated trajectories of~$200$ time steps for the surveillance task~\eqref{eq:surveillance-task}, under the relaxed optimal policy. 
}
\label{fig:relaxed_1}
\end{figure}

To  demonstrate  scalability and computational complexity of the proposed algorithm when AMECs do not exist, we repeat the policy synthesis under the same task~\eqref{eq:surveillance-task} but for different workspaces of various sizes, as in Section~\ref{sec:ordered-supply}. 
We set~$\gamma=0.3$, $d=300$ and $\beta=0.1$.
The size of resulting~$\mathcal{M}$, $\mathcal{P}$, $\Omega_{acc}$ and the time taken to compute them are shown in Table~\ref{table:sim-complexity-2}, where we also list the complexity of the \eqref{eq:comb-ascc}, which consists of~\eqref{eq:LP1} and~\eqref{eq:LP3}, and the time taken to solve~\eqref{eq:comb-ascc}.
It can be seen above that solving~\eqref{eq:comb-ascc} now requires a larger fraction of total time, compared to the construction of~$\mathcal{M}$,~$\mathcal{P}$ and~$\Omega_{acc}$. However, it requires much less time to compute the set of ASCCs~$\Omega_{acc}$ than the set of AMECs~$\Xi_{acc}$. For instance, in the case of~$29^2$ cells in the workspace, it took around~$23.1$ seconds to construct~$\mathcal{P}$ (which has approximately~$2.8\times 10^4$ states and~$2.9\times 10^5$ transitions) and~$19.6$ seconds to construct its ASCCs (compared with $160$ minutes in Table~\ref{table:sim-complexity}). Once~\eqref{eq:comb-ascc} is constructed, it took around~$2.5$ minutes to solve it.

\subsection{Comparison with PRISM}\label{sec:prism}

In this section we compare the proposed algorithm to the widely-used model-checking tool PRISM~\cite{kwiatkowska2011prism}. 
The following results were obtained using PRISM 4.3.1, where Linear Programming is chosen as the solution method.  First, since PRISM does not take the probabilistically-labeled MDP in~\eqref{eq:mdp} as inputs, we translate the product automaton in~\eqref{eq:prod} into PRISM language and verify its Rabin accepting condition directly. Implementation details can be found in~\cite{git_mdp_tg}.
For tasks~\eqref{eq:reach-task},~\eqref{eq:surveillance-task} and~\eqref{eq:deliver-task}, PRISM verifies that the probability of satisfying each of them is~$1.0$, within time~$0.46s$,~$0.38s$ and~$6.4s$, respectively. The difference in computation time is likely due to the difference in the LP solvers. 
Second, in order to test different values of~$\gamma$, we use the ``multi-objective property'' to find the minimal cumulative reward while ensuring the risk of violating the task is bounded by~$\gamma$. Note that the associated model has to be the modified product model~$\mathcal{Z}_{\texttt{pre}}$ defined in Section~\ref{sec:plan-prefix} as PRISM does not currently support multi-objective property with the ``$\textbf{F}$ target'' operator (i.e., $\Diamond S_c$). The computation time is approximately the same as in the previous cases. 
Last, the current PRISM version does not support the mean-payoff optimization in the AMECs, nor does it generate the relaxed control policy for the case where no AMECs exist in the product automaton. In fact, PRISM will simply return that the maximal probability of satisfying the task is $0$. 
The \texttt{MultiGain} tool recently proposed in~\cite{brazdil2015multigain} can handle multiple mean-payoff constraints but does not allow the tuning of the satisfaction probability~$(1-\gamma)$.

\begin{table}[t]
\begin{center}
\scalebox{1.05}{
    \begin{tabular}{| c| c|c | c | c | c|}
    \hline
   $\gamma_{\texttt{prex}}$ & $d$  & $\gamma_{\texttt{sufx}}$& \textbf{Failure} & \textbf{Pre. Success}  &  \textbf{Suf. Success} \\ \hline \hline

    0.1 & $300$& 0.05  & 106 & 894 & 852   \\ \hline
    0.2 & $300$& 0.05  & 169 & 831 & 785   \\ \hline
    0.3 & $300$& 0.05  & 318 & 682 & 650   \\ \hline
    0.4 & $300$& 0.05  & 409 & 591 & 549   \\ \hline
    0.1 & $280$& 0.85  & 888 & 901 & 117   \\ \hline    
    0.1 & $270$& 0.98  & 997 & 903 & 4   \\ \hline    
    \end{tabular}}
\caption{Statistics of~$1000$ Monte Carlo simulations under different~$\gamma_{\texttt{prex}}$ and~$d$, for task~\eqref{eq:surveillance-task} in Section~\ref{sec:surv-rex}. 
}
\label{table:rex-statistics}
\end{center}
\end{table}

\begin{table*}[t]
\begin{center}
\begin{tabular}{ccccccccccccc}
\toprule
\multicolumn{2}{c}{$\mathcal{M}$} & \multicolumn{2}{c}{$\mathcal{P}$} & 
\multicolumn{2}{c}{ASCCs $\Omega_{acc}$} &  \multicolumn{3}{c}{$\boldsymbol{\pi}^{\star}$ via~\eqref{eq:comb-ascc}}\\
\cmidrule(lr){1-2}\cmidrule(lr){3-4}\cmidrule(lr){5-6}\cmidrule(lr){7-9}
Size & {Time [\si{\second}]} & 
Size & {Time [\si{\second}]} &
Size & {Time [\si{\second}]} & 
Size of~\eqref{eq:LP1} & Size of~\eqref{eq:LP3} & {Time to solve~\eqref{eq:comb-ascc} [\si{\second}]} \\
\midrule
(100,\;816)& 0.13 & (1.0e3,\;1.1e4) & 0.9 & 3.1e2 & 0.66 & (202,\;920,\;3.4e3) & (301,\;1.4e3,\;4.9e3) & 0.45\\
\addlinespace[0.3em]
(324,\;2.8e3)& 1.57 & (2.9e3,\;3.1e4) & 3.39 & 9.8e2 & 1.84 & (6.5e2,\;3.1e3,\;1.1e4) & (9.7e2,\;4.7e3,\;1.6e4) & 2.41\\
\addlinespace[0.3em]
(900,\;8.4e3)& 23.9 & (7.7e3,\;7.9e4) & 7.04 & 2.7e3 & 5.09 & (1.8e3,\;8.7e3,\;3.0e4) &  (2.7e3,\;1.3e4,\;4.5e4) & 9.89\\
\addlinespace[0.3em]
(1.4e3,\;1.3e4)& 92.2 & (1.2e4,\;1.2e5) & 9.78 & 4.3e3 & 8.41 & (2.9e3,\;1.4e4,\;4.9e4) & (4.3e3,\;2.1e4,\;7.2e4) & 22.94\\
\addlinespace[0.3em]
(2.5e3,\;2.4e4)& 322.1 & (2.1e4,\;2.1e5) & 20.1 & 7.5e3 & 17.1 & (5.1e3,\;2.5e4,\;8.5e4) & (7.5e3,\;3.7e4,\;1.3e5) & 83.33\\
\addlinespace[0.3em]
(3.3e3,\;3.2e4)& 625.2 & (2.8e4,\;2.9e5) & 23.1 & 1.0e4 & 19.6 & (6.7e3,\;3.3e4,\;1.1e5) & (1.0e4,\;4.9e4,\;1.7e5) & 145.8\\
\bottomrule
\end{tabular}
\caption{Size and computation time of various models~$\mathcal{M}$ under task~\eqref{eq:surveillance-task} where \emph{no} AECs exist in~$\mathcal{P}$. The notations are defined similarly as in Table~\ref{table:sim-complexity}. In this case, the combined LP in~\eqref{eq:comb-ascc} contains~\eqref{eq:LP1} and~\eqref{eq:LP3} instead. 
}
\label{table:sim-complexity-2}
\end{center}
\end{table*}


\section{Experimental Study}\label{sec:experiment}

In this section, we present an experimental study. We use a differential-driven ``iRobot'' whose position we track in real-time via an Optitrack motion capture system.  
The communication among the planning module, the robot actuation module, and the Optitrack is handled by the Robot Operating System (ROS). The software implementation for this experiment is available in~\cite{git_mdp_tg_ros}. The experiment videos are online~\cite{mdp_exp_video}.

\subsection{Model Description}\label{sec:model-exp}
Consider the $2.5m \times 1.5m$ experiment workspace as shown in Figure~\ref{fig:exp-0}, with three base stations located at the corners and one obstacle region. It consists of~$5\times 3$ square cells of dimension ~$0.5m\times 0.5m$ each. The robot's motion within the workspace is abstracted similarly as in Section~\ref{sec:model-description}. The resulting MDP has~$60$ states and $456$ edges.



\begin{figure}[t]
\begin{minipage}[t]{0.495\linewidth}
\centering
   \includegraphics[width =0.93\textwidth]{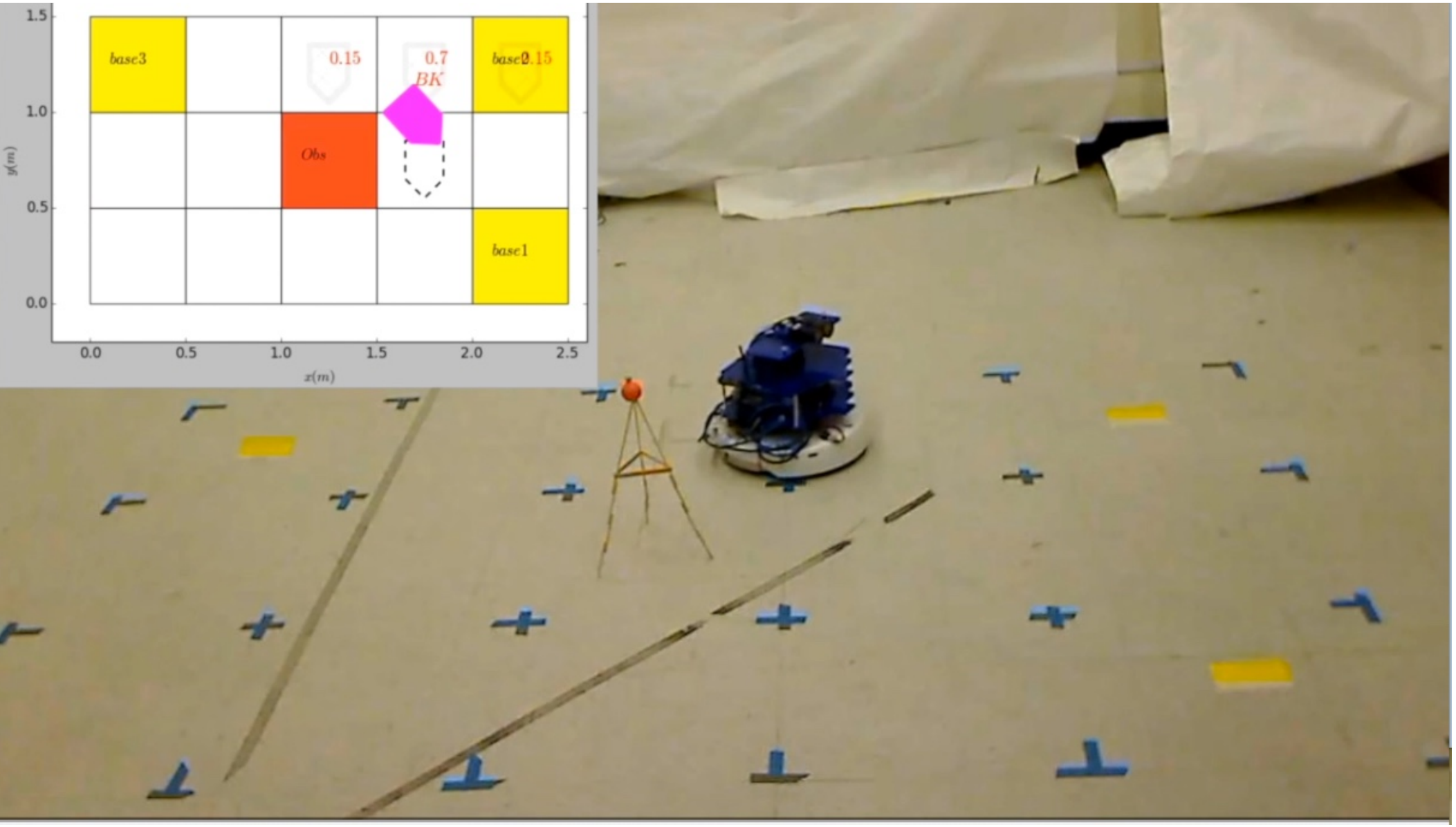}
  \end{minipage}
\begin{minipage}[t]{0.495\linewidth}
\centering
    \includegraphics[width =1.0\textwidth]{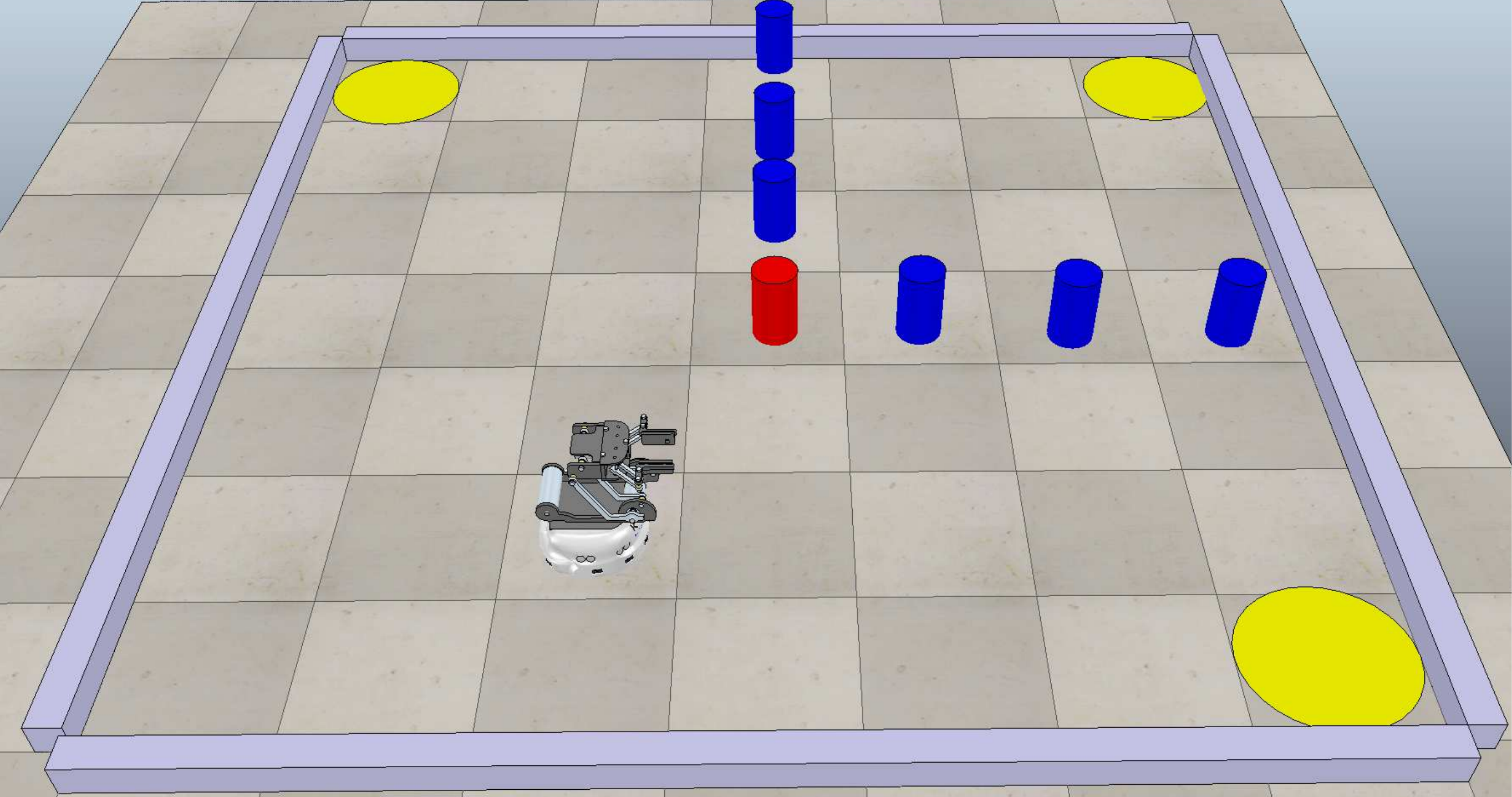}
  \end{minipage}
  \caption{The experiment workspace (left) with the monitoring panel. Three bases are marked by yellow tapes, while the tripod represents the obstacle. The monitoring panel displays the real-time  position, the control policy, the motion uncertainty and the robot status (being in prefix (green) or in suffix (magenta)). Customizable virtual experiment platform (right) in V-REP for task~\eqref{eq:surveillance-task} where no AMECs can be found in the product, see~\cite{git_mdp_tg} and~\cite{mdp_exp_video}.
}
\label{fig:exp-0}
\end{figure}


\subsection{Experimental Results}\label{sec:oneline-exp}

We consider two different tasks: first the sequential visiting task~\eqref{eq:reach-task} and then the surveillance task~\eqref{eq:surveillance-task}.

\begin{figure}[t]
\begin{minipage}[t]{0.495\linewidth}
\centering
   \includegraphics[width =1.02\textwidth]{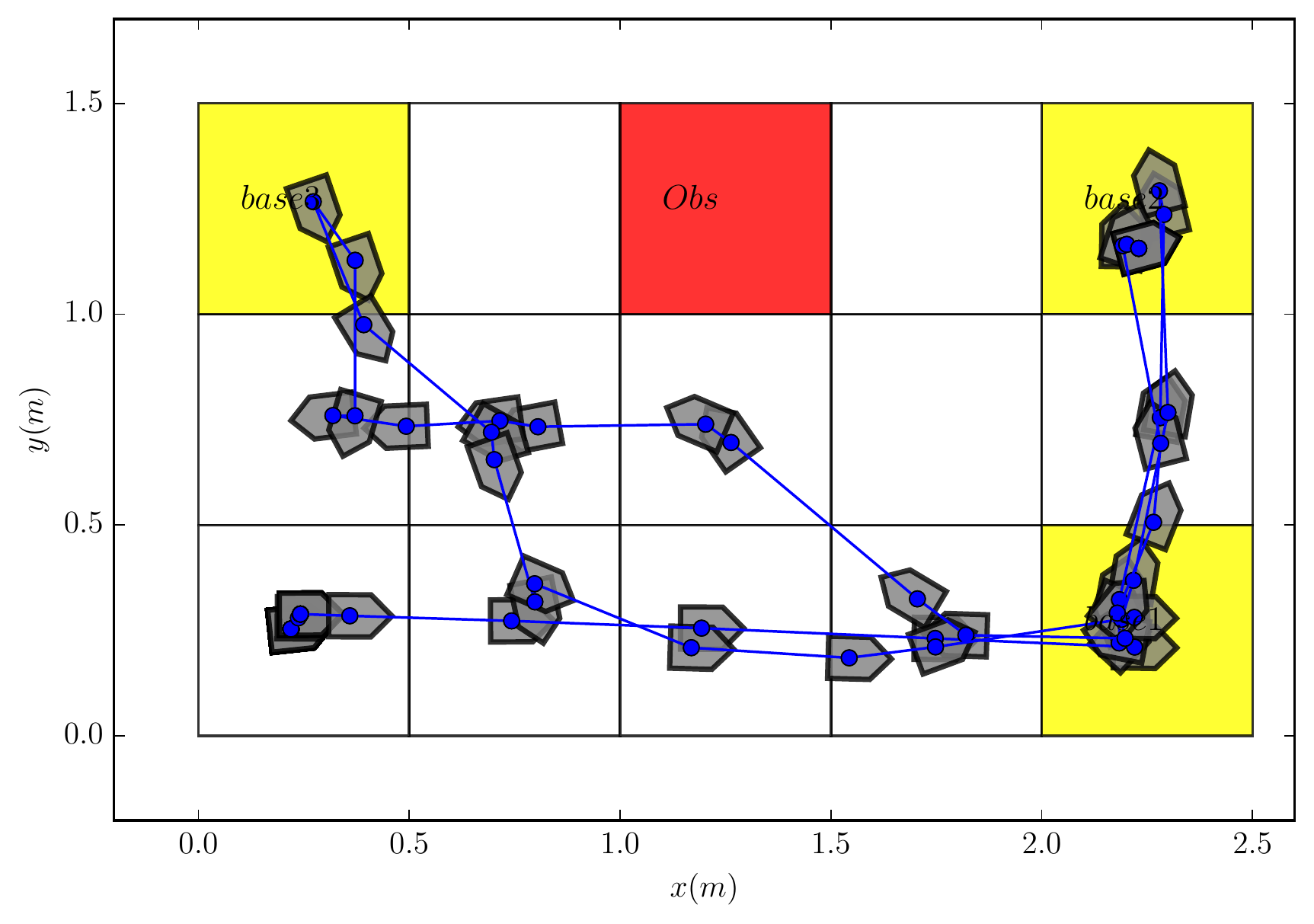}
  \end{minipage}
\begin{minipage}[t]{0.495\linewidth}
\centering
    \includegraphics[width =1.02\textwidth]{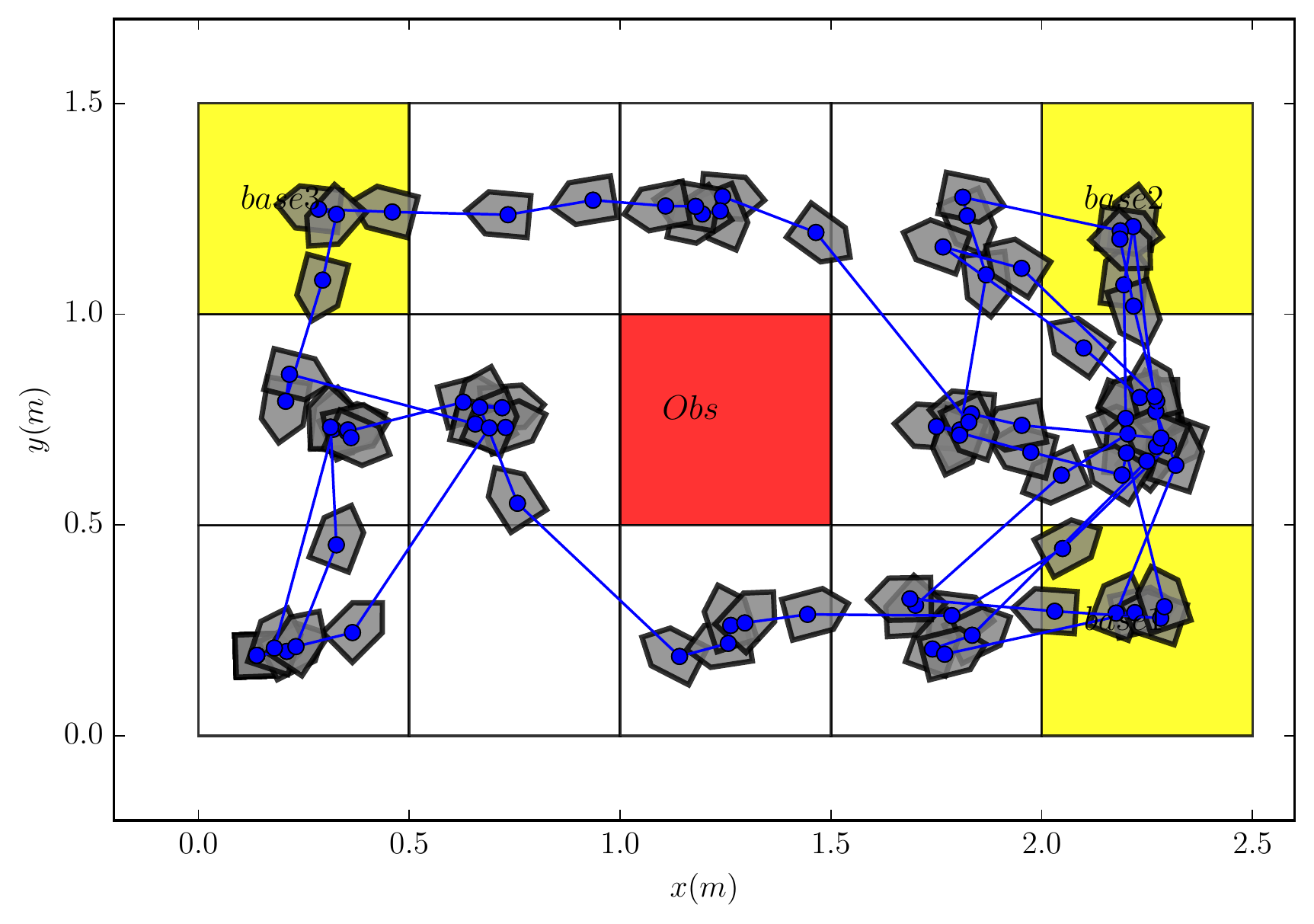}
  \end{minipage}
  \caption{The robot trajectory to satisfy task~\eqref{eq:reach-task} (left) and~\eqref{eq:surveillance-task} (right) when~$\gamma=0$, sampled at every~$15s$.
}
\label{fig:exp}
\end{figure}


\subsubsection{Sequential Visiting Task}\label{sec:task-1-exp}
The LTL task formula is given in~\eqref{eq:reach-task} and the associated DRA is constructed in Section~\ref{sec:order-reach}. The obstacle has probability~$0.1$ of appearing in the cell~$(1.25m, 1.25m)$. The resulting product automaton in this case has~$532$ states and~$4228$ edges and~$1$ accepting pair. For~$\gamma=0$ and $\beta=0.1$, it took~$3.16s$ to synthesize the complete policy  using Algorithm~\ref{alg:complete}, resulting in an average prefix cost~$47.72$ and  suffix cost~$1.0$. 
Then the robot was controlled in real-time using Algorithm~\ref{alg:execution}. The robot state was retrieved using the motion capture system and the observed label was generated randomly.
The complete video is online~\cite{mdp_exp_video} and the resulting trajectory shown in Figure~\ref{fig:exp}. Notice that the robot avoids completely collision with the obstacle.




\subsubsection{Surveillance Task}\label{sec:task-2-exp}
The LTL task formula is given in~\eqref{eq:surveillance-task} and the associated DRA is constructed in Section~\ref{sec:order-reach}. The obstacle has probability~$0.1$ of appearing in the cell~$(1.25m, 0.75cm)$. The resulting product automaton in this case has~$608$ states, $4992$ edges, and $1$ accepting pair. 

In the \emph{first} experiment, we choose~$\gamma=0$ and $\beta=0.1$ so that there is no risk allowed in the plan prefix. It took~$5.2s$ to synthesize the complete plan offline using Algorithm~\ref{alg:complete}. 
The real-time execution of the system followed Algorithm~\ref{alg:execution}. The resulting trajectory is shown in Figure~\ref{fig:exp}. In the \emph{second} experiment, we selected~$\gamma=0.1$ and $\beta=0.1$ to allow  risk in the plan prefix. It took~$4.9s$ to synthesize the complete policy. 
Compared to the case where~$\gamma=0$, the optimal policy instructs the robot to move forward, straight to the base station at $(2.25m, 0.25m)$, even though there is a risk of colliding with the obstacle at $(1.25m, 0.75m)$ due to the uncertainty in its forward action.  Both experiment videos are online~\cite{mdp_exp_video}. 

Lastly, to demonstrate the proposed scheme for much larger workspaces and more complex tasks, particularly when no AMECs can be found in the product automaton, we create a virtual experiment platform based on V-REP~\cite{rohmer2013v}, which is available in~\cite{git_mdp_tg}. A snapshot is shown in Figure~\ref{fig:exp-0}. The user can easily change the configuration of the workspace and the robot task specification. Once the control policy is synthesized via Algorithm~\ref{alg:complete} and saved, the user can perform any number of test runs in this environment.  Demonstration videos are online~\cite{mdp_exp_video} where we replicate the surveillance task with clustered obstacles from Section~\ref{sec:surv-rex}. It can be seen that the relaxed control policy can ensure high probability of avoiding bad states over long time intervals.


\section{Conclusion and Future Work}\label{sec:conc}
In this paper,  we propose a plan synthesis algorithm for probabilistic motion planning, subject to high-level LTL task formulas and risk constraints. Uncertainties in both the robot motion and the workspace properties are considered. We obtain optimal policies that optimize  the total cost both in the prefix and suffix of the system trajectory.
We also address the case where no AECs exist in the product automaton in which case the probability of satisfying the task is zero. 
The proposed solution provides provable guarantees on the probabilistic satisfiability  and the mean total-cost optimality, and is verified via both numerical simulations and experimental studies.
Future work involves extensions to multi-robot systems.

\bibliography{IEEEabrv,meng.bib}

\end{document}